\documentclass{article}

     \usepackage[nonatbib,preprint]{neurips_2023}
\usepackage[utf8]{inputenc} 
\usepackage[T1]{fontenc}    
\usepackage{hyperref}       
\usepackage{url}            
\usepackage{booktabs}       
\usepackage{amsfonts}       
\usepackage{nicefrac}       
\usepackage{microtype}      
\usepackage{tikz}

\usepackage{amsthm} 
\theoremstyle{plain}
\newtheorem{theorem}{Theorem}[section]

\theoremstyle{definition}
\newtheorem{lemma}[theorem]{Lemma}

\newtheorem{assumption}[theorem]{Assumption}
\theoremstyle{remark}

\usepackage{rotating}
\usepackage{hhline} 
\usepackage{algorithmic} 
\usepackage{algorithm} 
\usepackage{epsfig}
\usepackage{verbatim}
\usepackage[printonlyused]{acronym} 
\usepackage{xcolor}         
\usepackage{amsfonts}    
\usepackage{physics}
\usepackage{enumitem} 
\usepackage{subcaption}     
\usepackage{url}            
\usepackage{booktabs}          
\usepackage{nicefrac}       
\usepackage{microtype}      
\usepackage{mathrsfs}
\usepackage{times}
\usepackage{multicol} 
\usepackage{multirow}                
\usepackage{hhline}
\usepackage{hyperref}

\newcommand{\matrx}[1]{\ensuremath\boldsymbol{\rm #1}}
\newcommand{\vect}[1]{\ensuremath\boldsymbol{\rm #1}}
\newcommand{\set}[1]{\ensuremath{\mathcal #1}}
\newcommand{\graph}[1]{\ensuremath\boldsymbol{\mathcal #1}}

\newcommand{\twonorm}[1]{\ensuremath{\|#1\|_2}}

\newcommand{\fnorm}[1]{\ensuremath{\|#1\|_F}}

\newcommand{\inproduct}[1]{\ensuremath{\left \langle #1 \right \rangle}}

\newcommand*{\Scale}[2][4]{\scalebox{#1}{$#2$}}%
\newcommand\scalemath[2]{\scalebox{#1}{\mbox{\ensuremath{\displaystyle #2}}}}

\title{Mobilizing Personalized Federated Learning in Infrastructure-Less and Heterogeneous Environments via Random Walk Stochastic ADMM}

\author{
    Ziba Parsons \\
    CIS Department \\
    University of Michigan \\
    Dearborn, MI\\
    \texttt{zibapars@umich.edu}
\And
    Fei Dou \\
    School of Computing \\
    University of Georgia \\
    Athens, GA \\
    \texttt{fei.dou@uga.edu}
\And
    Houyi Du \\
    CIS Department \\
    University of Michigan \\
    Dearborn, MI\\
    \texttt{houyidu@umich.edu} 
\And
    Zheng Song \\
    CIS Department \\
    University of Michigan \\
    Dearborn, MI\\
    \texttt{zhesong@umich.edu} 
\And
    Jin Lu \\
    School of Computing \\
    University of Georgia \\
    Athens, GA \\
    \texttt{jin.lu@uga.edu} 
}

\begin{document}

\maketitle
\vspace{-10pt}
\begin{abstract}\label{sec:abstract}
\vspace{-6pt}
This paper explores the challenges of implementing Federated Learning (FL) in practical scenarios featuring isolated nodes with data heterogeneity, which can only be connected to the server through wireless links in an infrastructure-less environment. To overcome these challenges, we propose a novel mobilizing personalized FL approach, which aims to facilitate mobility and resilience. Specifically, we develop a novel optimization algorithm called Random Walk Stochastic Alternating Direction Method of Multipliers (RWSADMM). RWSADMM capitalizes on the server's random movement toward clients and formulates local proximity among their adjacent clients based on hard inequality constraints rather than requiring consensus updates or introducing bias via regularization methods. To mitigate the computational burden on the clients, an efficient stochastic solver of the approximated optimization problem is designed in RWSADMM, which provably converges to the stationary point almost surely in expectation. Our theoretical and empirical results demonstrate the provable fast convergence and substantial accuracy improvements achieved by RWSADMM compared to baseline methods, along with its benefits of reduced communication costs and enhanced scalability.

\end{abstract}
\vspace{-15pt}
\section{Introduction}
\label{sec:introduction}
\vspace{-5pt}
Federated Learning (FL) \cite{mcmahan2017communication,li2020federated,lim2020federated,banabilah2022federated} is a distributed machine learning paradigm that enables clients to learn a shared model without sharing their private data.
Unlike traditional machine learning approaches that rely on central servers for model training, FL allows clients to collaborate and train the model in a distributed manner, overcoming privacy issues related to passing data to a central server. Despite its advancements, real-world applications in environments with insufficient network support continue to face challenges. a) Maintaining consistent and reliable connections between the central server and clients becomes exceedingly challenging in environments lacking network infrastructures, e.g., natural disasters or military warzones. While intermittent connectivity may be available through satellite networks, the instability and limited capacity of such networks prevent the transmission of large data volumes, making it difficult to collect model updates from soldiers or first responders. b) The non-IID (non-independent and identically distributed) nature of clients' data, characterized by heterogeneity across the network, can hinder the generalization of the global model for each client.
Addressing these challenges is crucial for practical FL in such environments. In this paper, we propose RWSADMM, a novel FL scheme that uses Random Walk (RW) algorithm to enable server mobility among client clients.
These dynamic approach benefits scenarios with limited internet connectivity, where clients form clusters using local short-range transmission devices.

For instance, in the context of a warzone, we consider a scenario where soldiers are equipped with integrated visual augmentation systems (IVAS) \cite{handelman2022adaptive}. To facilitate the collection of model updates from the soldiers, a tactical vehicle equipped with a powerful computer 
navigates the warzone \cite{gilmore2015warfighter}, communicating with soldiers' locations through a satellite network.
Upon reaching a soldier, the tactical vehicle employs short-range communication technologies such as WiFi direct, Zigbee \cite{sadikin2020zigbee}, or Bluetooth to establish connections with nearby IVAS devices. Through these connections, the tactical vehicle collects model updates from the soldiers and distributes new models as necessary.
A graph-based representation is employed to determine the order of interactions with soldiers, where soldiers are depicted as nodes, and the edge between a soldier and its neighbor indicates that the neighbor is within the communication range of the vehicle that reaches the soldier.
This graph assists the tactical vehicle in making informed decisions about the order in which it engages with the soldiers.

Other applicable examples of such constrained network situations include ad hoc wireless learning \cite{mowla2019federated, li2021fedvanet, ochiai2022wireless}, wildlife tracking \cite{liu2015internet, toldov2018multi, arshad2020my}, Internet of Underwater Things (IoUT) \cite{kao2017comprehensive, victor2022federated}, natural disaster management \cite{ahmed2020active, imteaj2021fedresilience}, warzones \cite{bairwa2020mla}, or creating a digital democracy \cite{helbing2015society,gastil2017embracing} which 
assists in overcoming restrictions imposed by totalitarian regimes that prohibit internet access to civilians.



Specifically, to address challenge a), we propose an algorithmic framework called RWSADMM (Fig. \ref{fig:schema}), short for Random Walk Alternating Directional method of Multipliers, which considers a dynamic reachability graph among distributed clients using a movable vehicle as the central server. 
Clients are represented as nodes in the graph, with edges denoting neighborhood connections. Personal devices,  referred to as local clients, establish dynamic connectivity with the server when the server is nearby.
The server connects with a selected client and its neighbors while moving between locations using a non-homogeneous RW algorithm for probabilistic navigation. In each computation round, the vehicle broadcasts and gathers local model updates from residing clients, who rely on short-range communication to interact with the moving server, when it's within the communication range. The received updates are aggregated and used to update the global model iteratively.

To tackle the second challenge (b) arising from the heterogeneity of data distribution among clients, RWSADMM incorporates model personalization through local proximity among adjacent clients using hard inequality constraints, as opposed to unconstrained optimization with regularization techniques 
that may induce model bias. By formulating the problem with these constraints, RWSADMM reduces the computational complexity for clients, effectively mitigating the limitations of local computational power. This is achieved by designing the solver to the stochastic approximation of the minimization subproblem within the typical ADMM algorithm.

\begin{figure}[htp!] 
\setlength{\belowcaptionskip}{-5pt}
    \centering
    \begin{subfigure}{0.3\textwidth}
        \includegraphics[width=\textwidth]{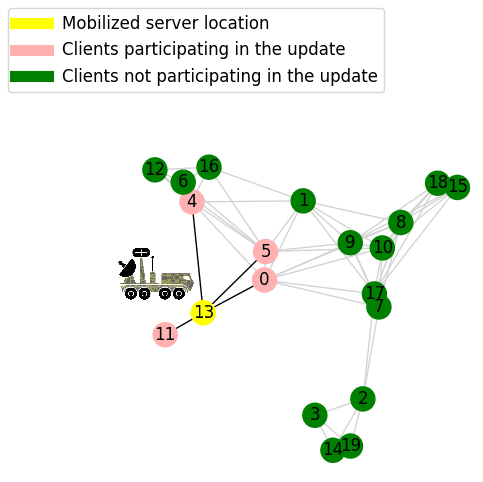}
        \caption{step $k$}
        \label{fig:schema-t=k}
    \end{subfigure}
  \begin{tikzpicture}
    \draw[dashed] (0,0) -- (0,4.5);
  \end{tikzpicture}
    \begin{subfigure}{0.3\textwidth}
        \includegraphics[width=\textwidth]{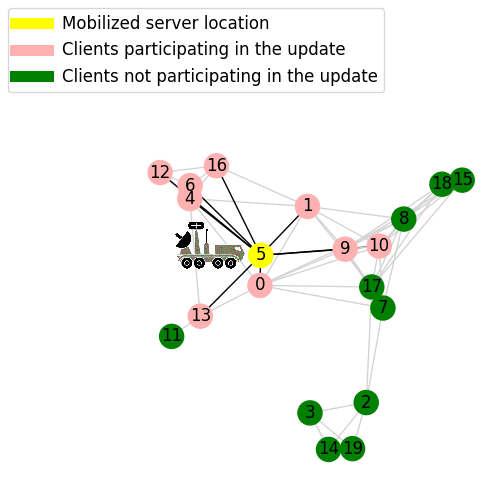}
        \caption{step $k+1$}
        \label{fig:schema-t=k+1}
    \end{subfigure}
  \begin{tikzpicture}
    \draw[dashed] (0,0) -- (0,4.5);
  \end{tikzpicture}
    \begin{subfigure}{0.3\textwidth}
        \includegraphics[width=\textwidth]{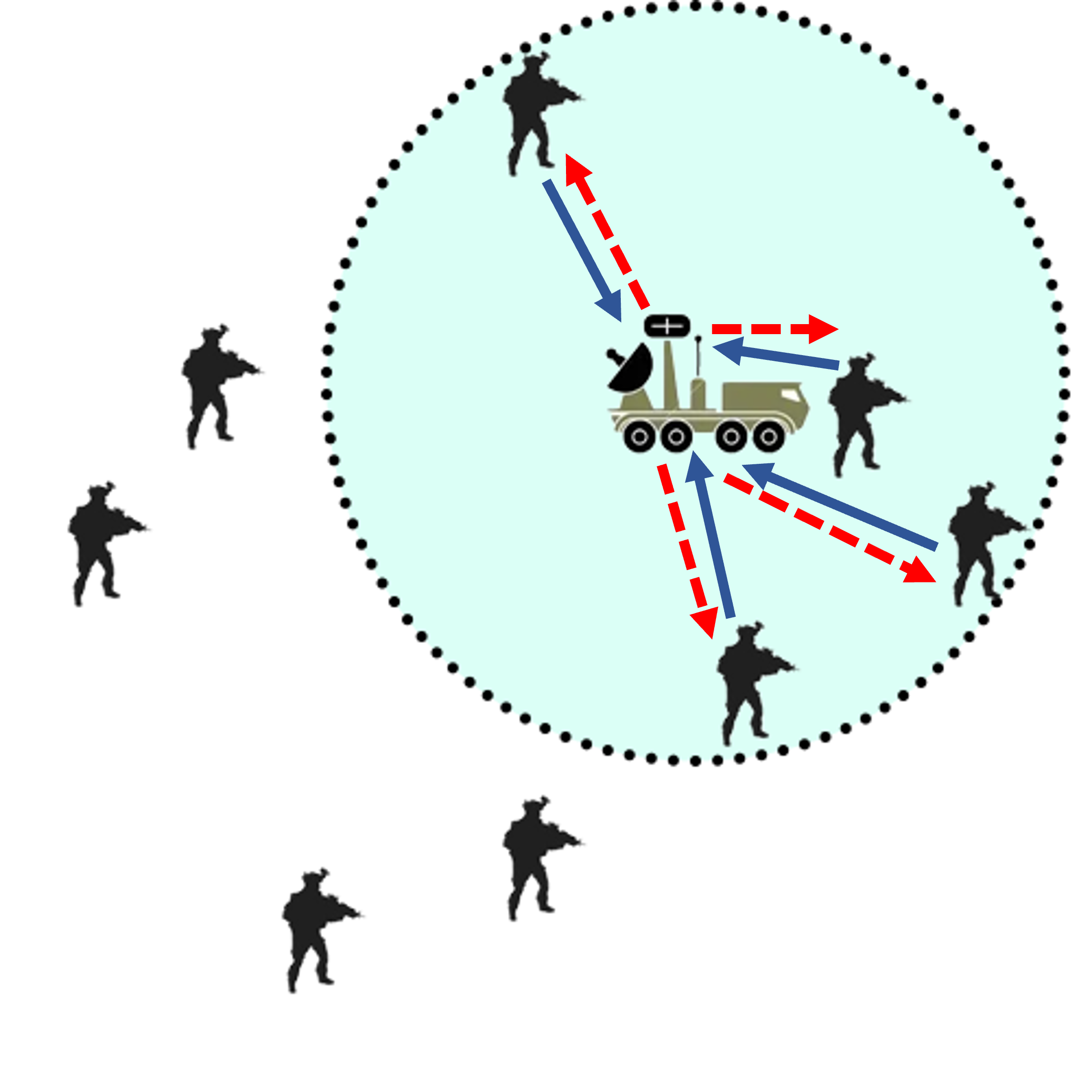}
        \caption{Mobilized FL between clients and the server}
        \label{fig:schema-t=k+1}
    \end{subfigure}    
    \caption{
    This illustration showcases the training process using the RWSADMM algorithm. A vehicle, serving as the mobile server, navigates between different clients using a random walk strategy. (a) In step $k$, the server moves to client $13$, covering the clients in $\mathcal{N}(13)$ for FL model training and completing the aggregation step. (b) In step $k+1$, client 5 is selected, and the vehicle moves to client 5. The training and aggregation steps occur within the zone encompassing $\mathcal{N}(5)$. 
    }
    \label{fig:schema}
\end{figure}
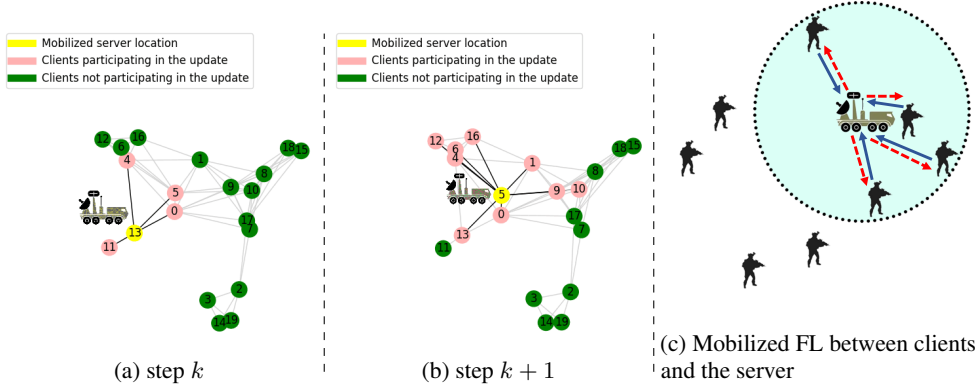
\vspace{4pt}
Our research makes three main contributions. Firstly, our proposed RWSADMM algorithm is \textbf{the first attempt to enable mobilizing FL with efficient communication and computation in an infrastructure-less setting}. The RWSADMM framework involves a server dynamically moving between different regions of clients and receiving updates from one or a few clients residing in the selected zone, which 
reduces communication costs and enhances system flexibility. 
Secondly, to address the issue of data heterogeneity among clients, RWSADMM \textbf{formulates local proximity among adjacent clients based on hard inequality constraints, 
avoiding the introduction of model bias via consensus updates}.
This approach provides an alternative realization of personalization, which is crucial when dealing with highly heterogeneous data distributions. Thirdly, to mitigate the computational burden on the clients, an efficient stochastic solver of the subproblem is designed in RWSADMM, which \textbf{provably converges to the stationary point almost surely in expectation} under mild conditions independent of the data distributions. Our theoretical and empirical results demonstrate the superiority of our proposed algorithm over state-of-the-art personalized Federated Learning (FL) algorithms, providing empirical evidence for the effectiveness of our approach.

\vspace{-6pt}
\section{Related Work}
\label{sec:related work}
\vspace{-5pt}
This paper is relevant to two distinct research areas, which are reviewed in two separate sections: FL frameworks that tackle data heterogeneity and ADMM-based FL frameworks.

\textbf{FL with data heterogeneity  }FL was initially introduced as FedAvg, a client-server-based framework that didn't allow clients to personalize the global model to their local data \cite{mcmahan2017communication}. This led to poor convergence due to local data heterogeneity, negatively impacting the global FL model's performance on individual clients. Recent works proposed a two-stage approach to personalize the global model. In the first stage, the FL global model is trained similarly to FedAvg. However, the second stage is included to personalize the global model for each local client through additional training on their local data. \cite{jiang2019improving} demonstrated that FedAvg is equivalent to Reptile, a new meta-learning algorithm, when each client collects the same amount of local data. To learn a global model that performs well for most participating clients, \cite{fallah2020personalized} proposed an improved version, Per-FedAvg. This new variant aims to learn a good initial global model that can adapt quickly to local heterogeneous data. An extension of Per-FedAvg, called pFedMe \cite{t2020personalized}, introduced an $\ell_2$-norm regularization term to balance the agreement between local and global models and the empirical loss. \cite{li2021ditto} proposed Ditto, a multi-task learning-based FL framework that provides personalization while promoting fairness and robustness to byzantine attacks. Ditto uses a regularization term to encourage personalized local models to be close to the optimal global model. \cite{deng2020adaptive} proposed to interpolate local and global models to train local models while also contributing to the global model. However, scaling these approaches can be challenging due to high communication costs, reliance on strong assumptions about network connectivity, or the requirement to compute second-order gradients. Additionally, there is a potential for enhancing the algorithms' overall performance. 

\textbf{ADMM-based FL   } The Alternating Direction Method of Multipliers (ADMM) is a widely recognized algorithm that effectively tackles optimization problems across multiple domains. In recent studies, ADMM has been successfully employed in distributed learning, as demonstrated by several works \cite{issaid2020communication, elgabli2020fgadmm, huang2019dp, graf2019distributed, zheng2018stackelberg, zhang2018improving, song2016fast, boyd2011distributed}.
In Federated Learning (FL) context, researchers have proposed various methods to address specific challenges. For handling falsified data in Byzantine settings, \cite{li2017robust} introduced a robust ADMM-based approach. To mitigate local computational burdens in FL, \cite{yue2021inexact} developed an inexact ADMM-based algorithm suitable for edge learning configurations. FL itself enables local training without the need to share personal data between clients and the server.
Despite the advantages of FL, there is still a concern regarding clients' privacy. Analyzing the parameter differences in the trained models uploaded by each client can compromise their privacy. To tackle this issue, \cite{ryu2021differentially} proposed an inexact ADMM-based federated learning algorithm that incorporates differential privacy (DP) techniques \cite{wei2020federated}. By leveraging DP, the algorithm enhances privacy protection during the FL process.
These ADMM-based frameworks also have high communication costs, ranging from $O(n)$ to $O(n^2)$ per iteration, depending on the network's density with $n$ clients. \cite{hong2017prox} introduced a Proximal Primal-Dual Algorithm (Prox-PDA) to enable network nodes to compute the set of first-order stationary solutions collectively. 
Moreover, these algorithms do not account for data heterogeneity in their framework designs, leading to performance deterioration in such scenarios. 
The most similar algorithm to RWSADMM is called Walkman \cite{mao2020walkman}. Walkman is an ADMM-based framework utilizing the random walk technique for distributed optimization. In Walkman, the communication and computation costs are reduced by activating only one agent at each step. 
Compared to other ADMM-based approaches, including Walkman, RWSADMM has several distinctive features. RWSADMM leverages stochastic approximation to reduce computation costs per iteration and enforces hard inequality constraints instead of consensus to manage heterogeneous data, resulting in increased robustness. Additionally, RWSADMM considers the dynamic graph, allowing it to adapt to changing network conditions and potentially improve communication efficiency. RWSADMM also incorporates a hard constraint parameter $\epsilon$ to promote local proximity among clients instead of using a regularization term as Walkman does to promote client consensus. This approach better balances personalization and global optimization. Finally, while Walkman is fully distributed without server involvement, RWSADMM is a server-based approach in which the server aggregates information from a small group of clients in each computation round. 
\vspace{-5pt}
\section{Random Walk Stochastic ADMM (RWSADMM)}
\label{sec:methods}
\vspace{-5pt}
Before delving into the specifics of the proposed algorithm, we present the key notation used throughout this research.
$\vect x\in \mathbb{R}^d$ represents a vector with length $d$ and $\vect e$ is defined as a vector with entries equal to 1 and $\matrx X\in \mathbb{R}^{l \times d}$ depicts a matrix with $l$ rows and $d$ columns. 
$[\vect x]_i$ represents the $i$th element of vector $\vect x$ and $[\matrx X]_{ij}$ is the $(i,j)$th element of matrix $\matrx X$. 
$[\matrx X]_i$, $[\matrx X]^j$ represent the $i$th row and $j$th column of matrix $\matrx X$, respectively. 
$(\nabla f(\vect{x}))_j$ is used to denote the $j$th entry of the gradient of $f(\vect{x})$. 
The inner product of $A$ and $B$ is shown as $\langle A, B\rangle$. 
$\mathbb{E}_t[.]$ indicates the expectation given the past $\xi_1,\ldots,\xi_{t-1} $. $\odot $ represents the Hadamard product/element-wise product and $\otimes $ represents the Kronecker product between two matrices. 
Finally, Norm p of vector $\vect{x}$ is denoted as $||\vect{x}||_p^p = \sum_{i=1}^{d}|x_i|^p$, $\vect x\in \mathbb{R}^d$ and Frobenius norm of matrix $\matrx{X}$ is written as $\fnorm{\matrx{X}} = \sqrt{\sum_{i=1}^n\sum_{j=1}^m\abs{x_{ij}}^2}$.

Let us first define our Mobilizing FL problem. Mobilizing FL, which involves a mobile server, can be formulated as an optimization problem on a connected graph $\graph{G}=(\set{V},\set{E})$. The graph comprises a set of $n$ clients, represented as $\set{V} = {v_1, v_2, \ldots, v_n}$, and a set of $m$ edges, denoted as $\set{E}$. The objective is to minimize the average loss function across all clients while adhering to inequality constraints that ensure local proximity among the clients' respective local models. The optimization problem can be formulated mathematically as follows:
\vspace{-3pt}
\begin{equation}
    \label{optimization problem}
\begin{aligned}
  &\scalemath{0.9}{ ~~~~~~~~~~~~~~\underset{\vect{x}_{1:n}\in\mathbb{R}^p}{\min} \; \frac{1}{n}\sum^n_{i=1}f_i(\vect{x}_i)} 
  ~~~~s.t.~ &  \scalemath{0.9}{\abs{\vect x_i - \vect x_j} \leq \vect \epsilon_i, \forall i\in \{1,\ldots,n\}, \forall j\in \mathcal{N}(i)/v_i.}
\end{aligned}
\end{equation}
where $f_i(\vect x_i)$ represents the local loss function with the model parameter as $\vect x_i$ for the client $i$, the vertex set $\mathcal{N}(i)$ contains client $i$ and its neighboring clients, and $\vect \epsilon_i$ is the non-consensus relaxation between local neighboring clients to replace model consensus requirement in typical FL. 
In our proposed FL method, we model the server's movement as a dynamic Markov Chain, introducing a dynamic element to the traditional ADMM-based approach. This work is the first to consider a dynamic mobile server within the ADMM-based FL framework.
In RWSADMM, client-server communication occurs only when the server is close to a client. The sequence of client indices that are updated, denoted as ${i_k}$, evolves based on a non-homogeneous Markov Chain with a state space of ${1,\ldots,n}$ \cite{ram2009incremental}. 
To describe the transition dynamics of the Markov Chain, we employ the non-homogeneous Markovian transition matrix $\matrx{P}(k)$, which represents the probabilities of transitioning between clients at time $k$. Specifically, the conditional probability of selecting client $j$ as the next client, given that client $i$ is the current client, is defined as:
\begin{equation}
    \label{the p_k eq}
    [\matrx{P}(k)]_{i,j} = Pr\;\{i_{k+1}=j | i_k=i\} \in [0,1]
\end{equation}
Additionally, it is assumed that the server determines the probability of all possible locations for its next destination based on the transition matrix $\matrx{P}(k)$ at time $k$. This provides a probabilistic approach to server navigation, allowing it to move around the network more effectively. 
To guarantee convergence, RWSADMM depends on the frequency of revisiting each agent. This quality is described by the \textit{mixing time} of the algorithm. An assumption for the mixing time is as follows:
\begin{assumption}
    \label{irreducible aperiodic markov chain}
The random walk $(i_k)_{k\geq0}, v_{i_k}\in \set{V}$ forms an irreducible and aperiodic (ergodic) Markov Chain with transition probability matrix of $\matrx{P}(k)\in\mathbb{R}^{n\times{n}}$ defined in Eq. \eqref{the p_k eq} and stationary distribution $\vect{\pi}$ satisfying  $\lim_{k\rightarrow\infty}\vect{\pi}^T\matrx{P}(k) =\vect{\pi}^T$.
The mixing time (for a given $\delta>0$) is defined as the smallest integer $\tau(\delta)$ such that $\forall i\in V$,
\vspace{-3pt}
\begin{equation} \label{mixing time ineq1}
\scalemath{0.9}{
    \norm{[\matrx{P}(k)^{\tau(\delta)}]_{i} - \vect{\pi^T}} \leq \delta\vect{\pi}_* 
    } 
\end{equation}
where $\vect{\pi}_* := \min_{i\in \set{V}}\vect{\pi}_i$. This inequality states the fact that regardless of the current state $i$ and time $k$, the probability of visiting each state $j$ after $\tau(\delta)$ steps is $(\delta\vect{\pi}_*)$-close to $\vect{\pi}_j$, that is, $\forall i,j\in \set{V}$,
\begin{equation}\label{mixing time ineq2}
\scalemath{0.9}{   
\norm{[\matrx{P}(k)^{\tau(\delta)}]_{ij} - \vect{\pi}_j} \leq \delta\vect{\pi}_* 
}
\end{equation}
\end{assumption}
Eq. \eqref{mixing time ineq2} is used to prove the sufficient descent of a Lyapunov function $L_\beta$ in Section 3.1.
Let's also define
\begin{equation}
    \label{p_max eq}
    \scalemath{0.9}{
    \matrx{P}_{max}=\lim_{k=+\infty}\{\matrx P | [\matrx P]_{ij}=\max_k[\matrx P(k)]_{ij} \},
    }
\end{equation}
from which one can further obtain $\norm{\matrx{P}(k)} \leq \norm{\matrx{P}_{max}}$ for all $k$.
Namely, the matrix $P_{max}$ is computed as the element-wise maximum matrix among all the matrices $P(k)$, for $k=0,\ldots,\infty$.
Therefore, the mixing time requirement in Eq. \eqref{mixing time ineq1} is guaranteed to hold for
\begin{equation}\label{mixing time ineq3}
\begin{aligned}
  \scalemath{0.9}{  
  \tau(\delta)  = \big\lceil \frac{1}{1-\sigma(\matrx{P})} \ln \frac{\sqrt{2}}{\delta\pi_*}  \big\rceil \overset{(a)}{\leq} \big\lceil \frac{1}{1-\sigma(\matrx{P}_{max})} \ln \frac{\sqrt{2}}{\delta\pi_*}  \big\rceil
  }
\end{aligned}
\end{equation}
where $ \sigma(\matrx P) := \sup \{\norm{f^T\matrx{P}}/\norm{f}:f^T\matrx{1}=0, f\in\mathbb{R}^n \}  $.  Using Eq. \eqref{p_max eq}, we have $\forall \matrx{P},\; \sigma(\matrx P) \leq \sigma(\matrx P)_{max}$ and the inequality $(a)$ can be inferred.
\vspace{-5pt}
\subsection{Algorithm}
\label{problem formulation}
\vspace{-5pt}
In this section, we derive RWSADMM by integrating random walk and stochastic inexact approximation techniques into ADMM. 
Considering $\matrx{X} := row(\vect{x}_1, \vect{x}_2,\ldots,\vect{x}_n) \in \mathbb{R}^{p \times n},\; F(\matrx{X}) := \sum^n_{i=1} f_i(\vect{x}_i) $, 
where the operation $row(.)$ refers to row-wise stacking of vectors $\vect{x}_i$'s. The mobilizing FL problem \eqref{optimization problem} can be expressed as:
\begin{equation}
    \label{problem formulation eq2}
\begin{aligned}
&\scalemath{0.9}{
~~~~~~~~~~\underset{\vect{y}_{1:n},\matrx{X}}{\min} \;\; \frac{1}{n}F(\matrx{X})} 
   &\scalemath{0.9}{ 
   ~\text{s.t.} \;\; \abs{\vect{1}\otimes \vect{y}_i - \matrx{X}_{\mathcal{N}(i)}} \leq \vect{1}\otimes \vect \epsilon_i/2, \forall i=1,\ldots,n }
\end{aligned}
\end{equation}
where $\vect{1} = [1 \;1 \ldots 1] \in \mathbb{R}^{n_i}$, $n_i$ denotes the volume of the vertex set $\set{N(i)}$. The constraint implies that $\abs{\vect{x}_i - \vect{x}_j} \leq \vect \epsilon_i$, $\forall i = 1 \ldots n$ and $\forall j \in \mathcal{N}(i)$ through the triangle inequality. $\vect y_{i}$ stored on the server is necessarily introduced as a local proximity of $\mathcal{N}(i)$.
We can obtain the augmented Lagrangian for problem \eqref{problem formulation eq2} 
\begin{align}
    \scalemath{0.85}{  L_\beta(\vect{y}_{1:n},\matrx{X},  \matrx{Z}_{1:n})= }&\scalemath{0.85}{ \frac{1}{n}\big[ F(\matrx{X}) + \sum_{i=1}^n\inproduct{\matrx{Z}_i, \abs{\vect{1}\otimes \vect{y}_i - \matrx{X}_{\mathcal{N}(i)}} - \vect \varepsilon_i} } 
    \scalemath{0.85}{ + \frac{\beta}{2}\sum_{i=1}^n\fnorm{\abs{\vect{1}\otimes \vect{y}_i - \matrx{X}_{\mathcal{N}(i)}} - \vect \varepsilon_i}^2   \label{problem formulation eq3} \big] } 
\end{align}
where 
$\vect \varepsilon_i=\vect \epsilon_i/2$ and $\matrx{Z}_i \in \mathbb{R}^{n_ip}$ are the dual variable and $\beta > 0$ is the barrier parameter. 
The RWSADMM algorithm minimizes the augmented Lagrangian $L_\beta(\vect{y}_{1:n},\matrx{X}, \matrx{Z}_{1:n})$ in an iterative manner. At each iteration $k$, only a subset of clients covered by the mobilized server, the clients in $\mathcal{N}(i_k)$, participate in the federated update. The following updates are performed:
\begin{equation}
\nonumber
\begin{aligned}
\vect x_{i_{k}}= 
&\scalemath{0.9}{
~~\arg\min_{\vect x_{i_{k}}} L_\beta (\vect y'_{i_{k}}, \matrx{x}_{i_k},  {\matrx{z}'_{i_{k}}}), } 
\scalemath{0.9}{
~~~~\vect y_{i_{k}}= 
\arg\min_{\vect y_{i_{k}}} L_\beta (\vect y_{i_{k}}, \matrx{X}_{\mathcal{N}(i_k)},  {\matrx{Z}'_{\mathcal{N}(i_{k})}}),}
\end{aligned}
\end{equation}
where $\vect y'_{i_{k}}$, $\vect x_{i_{k}}$, and ${\matrx{z}'_{i_{k}}}$ denote the groups of variables of the local parameters stored by client $i_k$ at the $(k-1)th$ update. 
After solving these subproblems, we update the multiplier $\matrx z_{i_{k}}$ as follows:
\begin{equation}
\nonumber
\begin{aligned}
\matrx z_{i_{k}}=
&\scalemath{0.9}{
\vect{z}'_{i_{k}} + \beta (\abs{ \vect{y}_{i_k} - \matrx{x}_{i_k}} - \vect \varepsilon_i), }
\end{aligned}
\end{equation}
Next, we derive the solver of each subproblem. The three steps are noted as Updating $\vect x_{i_{k}}$, Updating $\vect y_{i_{k}}$, and Updating $\matrx z_{i_{k}}$.

Updating $\vect x_{i_{k}}$:
\vspace{-25pt}
\begin{align}
\scalemath{0.9}{\min_{\vect x_{i_{k}}} \bigg[ f_{i_k}(\vect x_{i_{k}}) +  \inproduct{\matrx{z}'_{i_k}, \abs{\vect{y}'_{i_k} - \matrx{x}_{i_k}} - \vect \varepsilon_{i_k}} 
  + \frac{\beta}{2}\twonorm{ \abs{\vect{y}'_{i_k} - \matrx{x}_{i_k}} - \vect \varepsilon_{i_k} }^2     \bigg] }
  \hspace{-20pt} 
  \label{eq:x_subproblem_initial}
\end{align}
The Problem (\ref{eq:x_subproblem_initial}) can be solved iteratively, consuming significant computational resources for the local clients. Furthermore, the computational complexity increases as the local dataset grows, as is often true in real-world applications. By utilizing the stochasticity and first-order subgradient expansion, we arrive at a more computationally efficient approximation of the original problem in Eq. (\ref{eq:x_subproblem_initial2}).
\begin{align}
\scalemath{0.9}{\min_{\vect x_{i_{k}}} \bigg[ g_{i_k}(\vect{x}'_{i_k},\vect{\xi}_{i_k}) (\vect x_{i_{k}}-\vect x'_{i_{k}}) +  \inproduct{\matrx{z}'_{i_k}, \abs{ \vect{y}'_{i_k} - \matrx{x}_{i_k}} - \vect \varepsilon_{i_k}}  
   +  \frac{\beta}{2}\twonorm{ \abs{\vect{y}'_{i_k} - \matrx{x}_{i_k}} - \vect \varepsilon_{i_k} }^2     \bigg] }
   \label{eq:x_subproblem_initial2}
\end{align}

In Eq. \eqref{eq:x_subproblem_initial2}, $\vect{\xi}_{i_k}$  denotes one or a few samples randomly selected by client $i_k$ from its feature set and their ground truth labels in pairs at the $k$-th iteration. The function $g_{i_k}(\vect{x}'_{i_k},\vect{\xi}_{i_k}) $ is defined as the stochastic gradient of $f_{i_k}(\vect x'_{i_{k}})$ at $ \vect{x}'_{i_k}$. The stochastic approximation can tremendously reduce memory consumption and save computational costs in each iteration.  By setting the subgradient of the objective function in Eq. (\ref{eq:x_subproblem_initial2}) to zero, we can derive the closed-form solution in Eq. (\ref{eq:closeform_x}).
\begin{align}
\scalemath{0.8}{\vect{x}_{i_k} =} &\scalemath{0.8}{\vect{y}'_{i_k} +\frac{1}{\beta}\vect{z}'_{i_k}\odot sgn(\vect t')-\frac{1}{\beta}sgn(\vect t')\odot  \big( \vect \varepsilon_i +g_{i_k}(\vect{x}'_{i_k},\vect{\xi}_{i_k})  \big)} 
= \scalemath{0.8}{\vect{y}'_{i_k} +\frac{1}{\beta}sgn(\vect t')\odot (\vect{z}'_{i_k}-\vect \varepsilon_i-g_{i_k}(\vect{x}'_{i_k},\vect{\xi}_{i_k}) )}
\label{eq:closeform_x}
\end{align}
where the signum function $sgn(\cdot)$ extracts the signs of a vector and $\vect t'_{i_k}=\vect{y}'_{i_k} - \matrx{x}'_{i_k}$.

Updating $\vect y_{i_{k}}$:
We solve the following problem
\begin{equation}
\begin{aligned}
\min_{\vect y_{i_{k}}} 
& \scalemath{0.9}{
\inproduct{\matrx{Z}_{\mathcal{N}(i_k)}, \abs{\vect{1}\otimes \vect{y}_{i_k} - \matrx{X}_{\mathcal{N}({i_k})}} -\vect{1}\otimes  \vect \varepsilon_{i_k}} 
+\frac{\beta}{2}\fnorm{ \abs{\vect{1}\otimes \vect{y}_{i_k} - \matrx{X}_{\mathcal{N}({i_k})}} - \vect{1}\otimes \vect \varepsilon_{i_k}}^2 }
\end{aligned}
\label{eq:y_subproblem}
\end{equation}

one can readily derive a closed-form solution for the problem \eqref{eq:y_subproblem} as:
\begin{equation}
\label{eq:y_solution1}
\scalemath{0.9}{\vect y_{i_{k}} = \frac{1}{n_{i_k}}\sum _{j\in \mathcal{N}_{i_k}} \big[\vect x_{i_k} -(\frac{\vect z_{i_k} }{\beta} + \vect \varepsilon_{i_k} ) \odot sgn(\vect t_{i_k})  \big] }
\end{equation}
where $\vect t_{i_k}=\vect{y}'_{i_k} - \matrx{x}_{i_k}$ is similar to that of Eq. \eqref{eq:closeform_x} except the updated $\matrx{x}$. Specifically, via mathematical induction, we can attain the new updated form of $\vect y_{i_{k}} $ below, which can also reduce the communication cost from $O(n)$ to $O(1)$:
\vspace{-5pt}
\begin{equation}
\label{eq:y_solution2}
\begin{aligned}
\vect y_{i_{k}} = 
&\scalemath{0.9}{
\vect y'_{i_{k}}+\frac{1}{n_{i_k}}\big[\vect x_{i_k} -(\frac{\vect z_{i_k} }{\beta} + \vect \varepsilon_{i_k} ) \odot sgn(\vect t_{i_k})  \big] 
- \big[\vect x'_{i_k} -(\frac{\vect z'_{i_k}}{\beta} + \vect \varepsilon_{i_k} ) \odot sgn(\vect t'_{i_k})  \big] }
\end{aligned}
\end{equation}

Updating $\vect z_{i_{k}}$:
The Lagrangian multiplier $\vect{z}_{i_k}$ can be updated strictly following the standard ADMM scheme below:
\vspace{-5pt}
\begin{equation}\label{z update eq}
 \scalemath{0.9}{           
 \vect{z}_{i_k} = \vect{z}'_{i_k} + \kappa\beta\big[\vect{x}_{i_k} -\vect{y}'_{i_k} -\vect \varepsilon_{i_k} \big] }
\end{equation}

The $\kappa$ coefficient used in Eq. \eqref{z update eq} is decayed in each process step to achieve better convergence.

Please refer to Appendix \ref{appendix: algorithm} for the entire RWSADMM algorithm framework.
\vspace{-6pt}
\section{Theoretical Analysis} \label{sec: convergence}
\vspace{-5pt}
In this section, we present the theoretical convergence guarantee of RWSADMM. To ensure its convergence, certain common assumptions are made regarding the properties of the loss functions. The assumptions are as follows:
\begin{assumption}
    \label{f coercive and bounded below}
The objective function $f(\vect{x})$ is bounded from below 
and coercive over $\mathbb{R}^p$, that is, for any sequence $\{\vect{x}^k\}_{k\geq0}\subset\mathbb{R}^p$, 
\begin{equation}
 \scalemath{0.9}{   
 \text{if} \;\norm{\vect{x}^k}\xrightarrow[]{k\rightarrow\infty}\infty \Rightarrow \frac{1}{n}\sum^n_{i=1}f_i(\vect{x})\rightarrow\infty }
\end{equation}
\end{assumption}
\begin{assumption}
    \label{f is lsmooth and l-lipschitz}
The objective function $f_i(\vect{x})$'s are L-smooth, that is, $f_i$ are differentiable, and its gradients are L-Lipschitz, that is, $\forall \vect{u}, \vect{v} \in \mathbb{R}^p$ \cite{mao2020walkman},
\begin{equation}\label{l-lipschitz}
 \scalemath{0.9}{   
 \norm{\nabla f_i(\vect{u}) -  \nabla f_i(\vect{v})} \leq   L\norm{\vect{u}-\vect{v}},\; \; \forall i= 1,\ldots,n }
\end{equation}
\end{assumption}
Remark: In consequence it also holds that $\forall \vect{u},\vect{v} \in \mathbb{R}^p$
\begin{align} \label{l-lipschitz-2}
 \scalemath{0.9}{   
 f_i(\vect{u}) -  f_i(\vect{v}) } 
 & \scalemath{0.9}{
 \leq \nabla f_i(\vect{v})^T(\vect{u} - \vect{v}) + \frac{L}{2}\norm{\vect{u}-\vect{v}}^2,  
\; \; \forall i = 1,\ldots,n. }  
\end{align}
\begin{assumption}
    \label{f is m-lipschitz}
The objective function $f$ is M-Lipschitz, that is, $\forall \vect{u},\vect{v}\in\mathbb{R}^p$ \cite{li2019convergence},
\begin{equation}
\scalemath{0.9}{    
|f(\vect{u}) - f(\vect{v})| \leq M\norm{\vect{u}-\vect{v}} }
\end{equation}
\end{assumption}
\begin{assumption}
    \label{expectation of estimation}
The first-order stochastic gradient is sampled, which returns a noisy but unbiased estimate of the gradient of $f$ at any point $\vect{x}\in\mathbb{R}^p$, that is, $\forall \vect{x}\in\mathbb{R}^p$,
\begin{equation}\label{Eg=gF}
  \scalemath{0.9}{  
  \mathbb{E}_\xi[g(\vect{x},\xi)]=\nabla f(\vect{x}) }
\end{equation}
\end{assumption}
Remark: Substituting Eq. \eqref{Eg=gF} into Eq. \eqref{l-lipschitz}, one can obtain that for $i=1,\ldots,n$, we have
\begin{equation}\label{l-lipschitz for expectation of estimation - 1}
\begin{aligned}
    &   \scalemath{0.9}{
    \norm{\mathbb{E}_\xi [g(\vect{u},\xi)] - \mathbb{E}_\xi[g(\vect{v},\xi)]} \leq L\norm{\vect{u}-\vect{v}} }
\end{aligned}
\end{equation}
Substituting Eq. \eqref{Eg=gF} into Eq. \eqref{l-lipschitz-2}, for $i=1,\ldots,n$, we can obtain
\begin{align}
    &  \scalemath{0.9}{ 
    f_i(\vect{u}) -  f_i(\vect{v}) \leq \mathbb{E}_\xi[g(\vect{v},\xi)]^T(\vect{u} - \vect{v}) + \frac{L}{2}\norm{\vect{u}-\vect{v}}^2,} 
    \label{l-lipschitz for expectation of estimation - 2}
    \end{align}
\begin{assumption}
        \label{noise variance}
The noise variance of the stochastic gradient is bounded as:
\begin{equation}
     \scalemath{0.9}{ \mathbb{E}_\xi(\norm{\nabla f(\vect{x}) - g(\vect{x},\xi)}^2)\leq \exp(1), \; \text{for~all~} \vect{x}.}
\end{equation}
\end{assumption}
This condition bounds the expectation of $\norm{\nabla f(\vect{x}_t)-g(\vect{x}_t,\xi_t)}^2$. Using Jensen's inequality, this condition implies a bounded variance \cite{li2019convergence}.

We revisit the related crucial properties of the Markov Chain. The first time that the Markov Chain $(i_k)_{k\geq0}$ hits agent $i$ is denoted as $T_i:=\min\{k:i_k=i\}$, and maximum value of $T$ over all clients is defined as $\scalemath{0.9}{T := \max\{T_1,\ldots,T_n\}}$.
For $\scalemath{0.9}{k>T}$, let $\scalemath{0.9}{\tau(k,i)}$ denote the iteration of the last visit to agent $i$ before $k$, mathematically we have
\begin{equation} \label{tau def}
    \tau(k,i) = \max\{k^\prime: i_{k^\prime}=i,k^\prime < k \}.
\end{equation}
To prove the convergence of our proposed algorithm, two Lyapunov functions defined for RWSADMM are required to be investigated:
\begin{equation}\label{lyapEquations}
    \scalemath{0.9}{ 
    L_\beta^k := L_\beta(\vect{y}^k,\matrx{X}^k;\matrx{Z}^k),~~
    M_\beta^k := L_\beta^k + \frac{L^2}{n}\sum^{n}_{i=1}\norm{\vect{y}_i^{\tau{(k,i)+1}}-\vect{y}_i^{\tau{(k,i)}}}^2 }
\end{equation}
where $\scalemath{0.9}{L_\beta(\vect{y}^k,\matrx{X}^k;\matrx{Z}^k)}$ is defined in Eq. \eqref{problem formulation eq3}. The $M_\beta^k$ is utilized in the convergence analysis.
To guarantee the convergence of our algorithm, first, we refer to the asymptotic analysis of the nonhomogeneous Markov chain presented in \cite{nedic2008distributed}. Define $\scalemath{0.9}{\matrx{\Phi}(k,l)}$ with $\scalemath{0.9}{k\geq l}$ as the product of the transition probability matrices for the Markov chain from time $l$ to $k$, i.e., $\scalemath{0.9}{\matrx{\Phi}(k,l) = \matrx{P}(k)\ldots \matrx{P}(l)$ with $k\geq l}$. Then we have the following convergence result:
\begin{lemma}
    \label{transition matric lemma}
    Consider
    \vspace{-5pt}
    \begin{enumerate}
        \item $\forall s$, $ \lim_{k\rightarrow \infty}\matrx{\Phi}(k,l) = \frac{1}{n}\vect{e}\vect{e}^T$. \vspace{-5pt}
        \item The convergence of $\matrx{\Phi}$ is geometric and the rate of convergence considering $\forall k,l, \text{with } k\geq l \geq0$, is given by
        \begin{equation}
            \label{phi bound eq}
             \scalemath{0.9}{
             \big| [\matrx{\Phi}(k,l)]_{i,j} - \frac{1}{n}\big| \leq \big( 1 - \frac{\eta}{4n^2}\big)^{{\lceil\frac{k-l+1}{Q}\rceil}-2}}
        \end{equation}
    \end{enumerate}
    \vspace{-5pt}
\end{lemma}
Using Lemma \ref{transition matric lemma}, the convergence analysis of the algorithm is as follows.
\begin{lemma}
    \label{convergence lemma}
Under Assumptions \ref{f coercive and bounded below} and \ref{f is lsmooth and l-lipschitz}, if $\scalemath{0.9}{\beta>2L^2+L+2$, $(M_\beta^k)_{k\geq0}}$ is lower bounded and convergent, the iterates $(\vect{y}^k,\matrx{X}^k,\matrx{Z}^k)_{k\geq0}$ generated by RWSADMM is bounded.
\end{lemma}
The proof sketch and the detailed convergence proof are presented in Appendix \ref{appendix: convergence}. Using Lemma \ref{convergence lemma} and \ref{convergence lemma2}, we can present the convergence of RWSADMM in Theorem \ref{convergence theorem}. 
\begin{theorem}
    \label{convergence theorem}
Let Assumption \ref{noise variance} hold. 
For $\beta>2L^2+L+2$, it holds that any limit point $(\vect{y}^*, \matrx{X}^*, \matrx{Z}^*)$ of the sequence $(\vect{y}^k, \matrx{X}^k, \matrx{Z}^k)$ generated by RWSADMM satisfies $\vect{y}^* = \vect{x_i}^*, \; i = 1,\ldots,n $ where $\vect y^*$ is a stationary point of Eq. \eqref{problem formulation eq2}, with probability $1$, that is,
\vspace{-5pt}
\begin{equation}
    \label{convergence theorem eq1}
\begin{aligned}
 \scalemath{0.9}{Pr\big(0 \in \frac{1}{n}\sum^n_{i=1}\grad f_i(\vect{y^*}) \big) = 1}
\end{aligned} 
\end{equation} \vspace{-5pt}
If the objective function of Eq. \eqref{problem formulation eq2} is convex, then $\vect y^*$ is a minimizer. 
\end{theorem}

Next, Theorem \ref{convergence rate theorem} further presents that the algorithm converges sublinearly. This is comparable to the convergence rate of other FL methods \cite{karimireddy2020scaffold,yu2019parallel,li2021ditto,deng2020adaptive}, but the existing methods didn’t consider the dynamic graph and infrastructure-less environment. The detailed proof is offered in Appendix \ref{appendix: convergence rate}. 


\begin{theorem}
    \label{convergence rate theorem}
Under Assumptions \eqref{irreducible aperiodic markov chain}, \eqref{f coercive and bounded below}, and \eqref{f is lsmooth and l-lipschitz}, with given $\beta$ in Lemma \ref{convergence lemma}, and local variables initiated as $\grad f_i(\vect{x}_i^0) = \beta \vect{x}_i^{0} = \vect{z}_i^{0}, \forall{i} \in \{1,\ldots, n\}$, there exists a sequence $\{g^k\}_{k\geq0}$ with $\{g^k\}\in\partial{L^{k+1}_\beta}$ satisfying 
\vspace{-2pt}
\begin{equation}
    \label{convergence rate eq}
    \scalemath{0.9}{\underset{k\leq K}{\textrm{min}} \mathbb{E}\norm{g^k}^2 \leq \frac{C}{K} (L^0_\beta - \underset{-}{f}), \quad \forall K > \tau(\delta) + 2}
\end{equation}
where $C$ is a constant depending on $\beta$, $L$, and $\gamma$, $n$, and $\tau(\delta)$. 
\end{theorem}
\textbf{Communication Complexity    }
Using Theorem \ref{convergence rate theorem}, the communication complexity of RWSADMM for nonconvex nonsmooth problems is as follows. To achieve ergodic gradient deviation $\scalemath{0.9}{E_t := \underset{k\leq K}{min}   \mathbb{E}\norm{g^k}^2 \leq \omega$ for any $K > \tau(\delta) +2}$, it is sufficient to have
\vspace{-5pt}
\begin{equation}
    \label{communication complexity eq1}
    \begin{aligned}
        \scalemath{0.9}{\frac{C}{K} (L_\beta^0 - \underset{-}{f}) \leq \omega} \overset{(a)}{\longrightarrow}  \scalemath{0.9}{K \sim O \big( \frac{1}{\omega} . \frac{\tau(\delta)^2+1}{(1-\delta)n\pi_*}\big)}
    \end{aligned}
\end{equation}
\vspace{-5pt}
(a) is achieved by taking $L_\beta^0$ and $\underset{-}{f}$ as constants and independent of $n$ and the network structure.
Using the $\tau(\delta)$ definition from \eqref{mixing time ineq3}, by setting $\delta=1/2$ and assuming the reversible Markov chain with $P(k)^T = P(k)$, the communication complexity is
\begin{equation}
    \label{communication complexity}
    \begin{aligned}
        \scalemath{0.9}{O\big( \frac{1}{\omega} . \frac{ln^2 n}{(1-\lambda_2(\matrx P(k)))^2} \big)}
    \end{aligned}
\end{equation}
\textbf{Communication Comparison    }
Among the baseline frameworks, Per-FedAvg \cite{fallah2020personalized} and APFL \cite{deng2020adaptive} have addressed the communication complexity of their respective frameworks. By assuming that Assumption \ref{irreducible aperiodic markov chain} holds and utilizing Eq. \eqref{communication complexity}, we can determine the communication complexity of RWSADMM as $O(\omega^{-1})$ for $K$ iterations. In comparison, Per-FedAvg exhibits a higher communication complexity of $O(\omega^{-3/2})$.
In the case of APFL, all clients are assumed to be used in each computation round to ensure convergence in nonconvex settings. The communication complexity of APFL is determined as $O(n^{3/4}\omega^{-3/4})$, where $n$ represents the total number of clients. Consequently, when $n$ is large, APFL exhibits a significantly higher communication rate than RWSADMM. Overall, the communication complexity analysis suggests that RWSADMM offers superior scalability and communication efficiency compared to existing methods.
\vspace{-6pt}
\section{Experimental Results}\label{sec:experiments} 
\vspace{-5pt}
\textbf{Setup   } We evaluate the performance of RWSADMM using heterogeneous data distributions.  All the experiments are conducted on a workstation with Threadripper Pro 5955WX, 64GB DDR4 RAM, and NVidia 4090 GPU. All frameworks are performed on standard FL benchmark datasets (MNIST \cite{lecun1998mnist}, Synthetic \cite{t2020personalized}, and CIFAR10 \cite{krizhevsky2009learning}) with 10-class labels and convex and non-convex models. Multinomial logistic regression (MLR), multilayer perceptron network (MLP), and convolutional neural network (CNN) models are utilized for strongly convex and two non-convex settings, respectively.
We create a moderately dynamic connected graph of randomly placed nodes where each node has at least $5$ neighboring nodes at $k$-th update. We set the probability transition matrix $\matrx P(k)$ as $[\matrx P(k)]ij=1/deg(i_k)$ and set up the experiments for $N=20$ clients with a regeneration frequency of 10 steps for the dynamic graph. The data is split among clients using a pathological non-IID setting. The data on each client contains a portion of labels (two out of ten labels), and the allocated data size for each client is variable. For the Synthetic data, we use the same data generative procedures of \cite{t2020personalized} with 60 features and 100 clients. 
All local datasets are split randomly with $75\%$ and $25\%$ for training and testing, respectively. The models' details, the rationale behind graph construction, and hyperparameter tuning for $\beta$, $\kappa$, and selected $\varepsilon$ value are further described in Appendix \ref{appendix: experiments-rest}. 

\textbf{Performance Comparison  }
The performance of RWSADMM is compared with FedAvg \cite{mcmahan2017communication} as a benchmark and several state-of-the-art personalized FL algorithms such as Per-FedAvg \cite{fallah2020personalized}, pFedMe \cite{t2020personalized}, APFL \cite{deng2020adaptive}, and Ditto \cite{li2021ditto}. The test accuracy and training loss for the MNIST dataset is depicted in Fig. \ref{fig: mnist-performance}. 
\begin{figure}[htp!]
    \centering
    \begin{subfigure}{0.25\textwidth}
        \includegraphics[trim=0pt 5pt 5pt 0pt, clip, width=\textwidth]{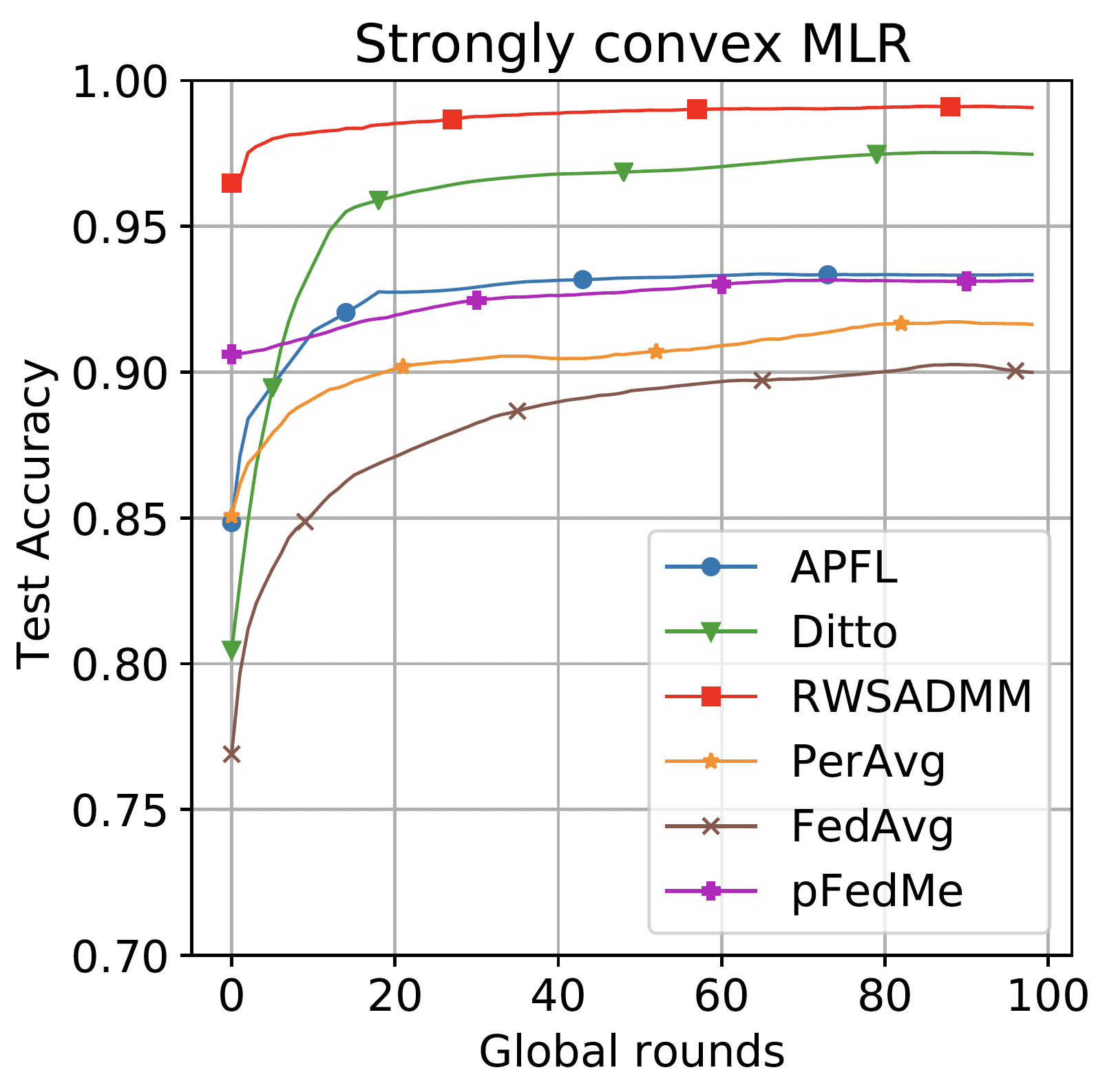}
        \caption{MLR acc}
        \label{fig:mlr-acc-mnist-all}
    \end{subfigure}
    \begin{subfigure}{0.25\textwidth}
        \includegraphics[trim=0pt 5pt 5pt 0pt, clip, width=\textwidth]{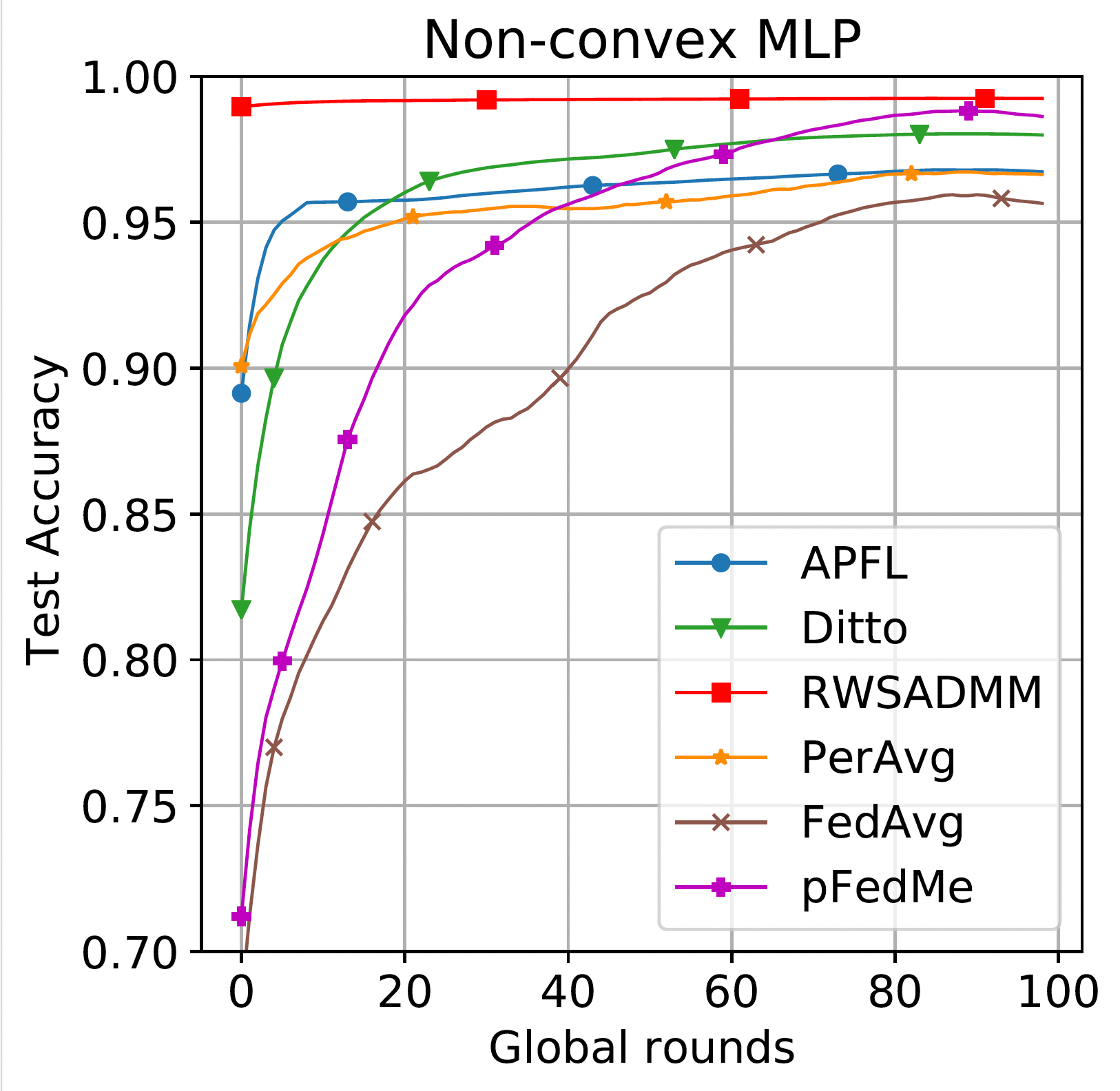}
        \caption{MLP acc}
        \label{fig:mlp-acc-mnist-all}
    \end{subfigure}
    \begin{subfigure}{0.25\textwidth}
        \includegraphics[trim=5pt 5pt 5pt 5pt, clip, width=\textwidth]{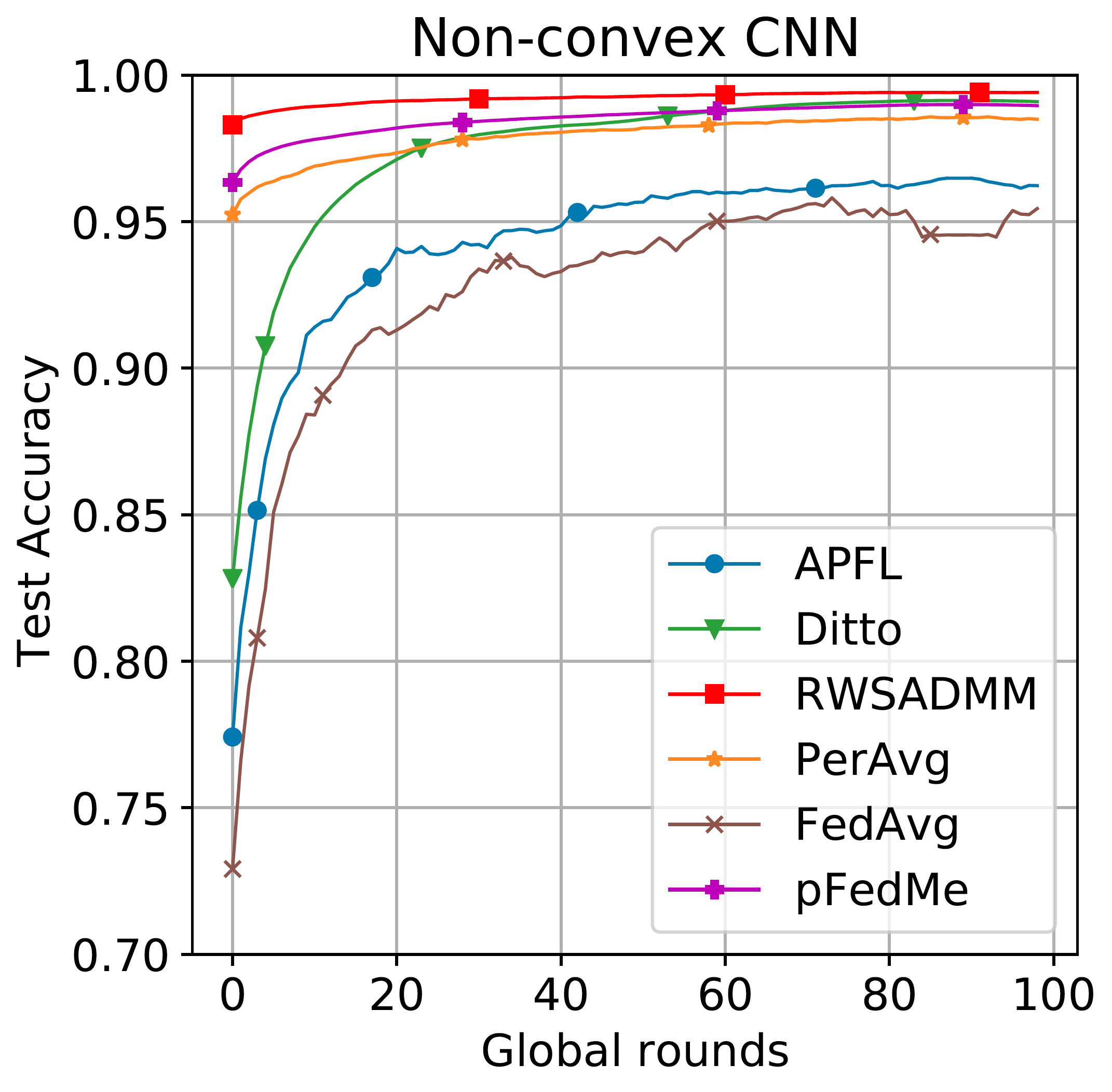}
        \caption{CNN acc}
        \label{fig:cnn-acc-mnist-all}
    \end{subfigure}
    \vskip\baselineskip
    \vspace{-8pt}
    \begin{subfigure}{0.25\textwidth}
        \includegraphics[trim=0pt 5pt 5pt 0pt, clip, width=\textwidth]{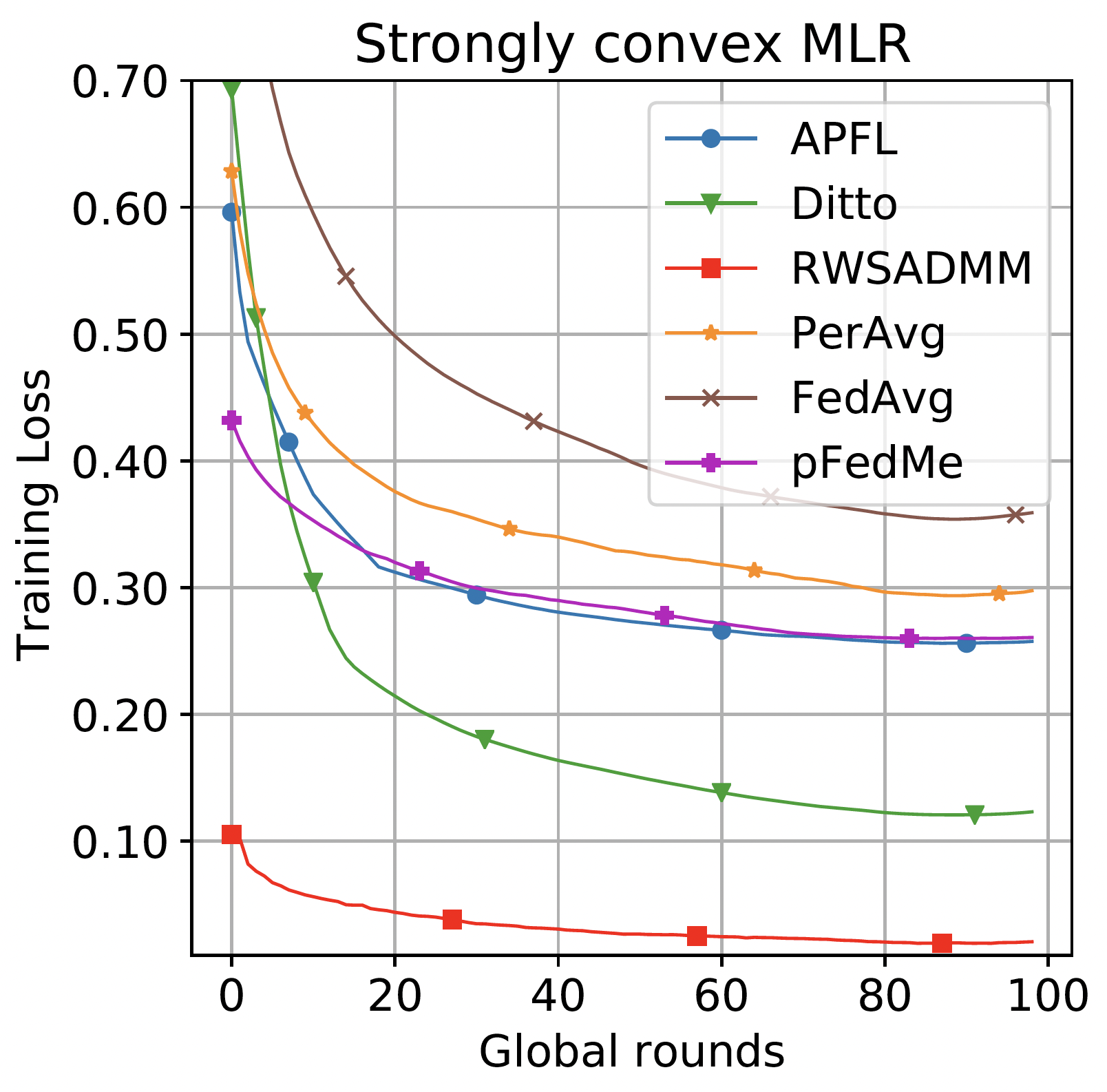}
        \caption{MLR loss}
        \label{fig:mlr-loss-mnist-all}
    \end{subfigure}
    \begin{subfigure}{0.25\textwidth}
        \includegraphics[trim=0pt 5pt 5pt 0pt, clip,width=\textwidth]{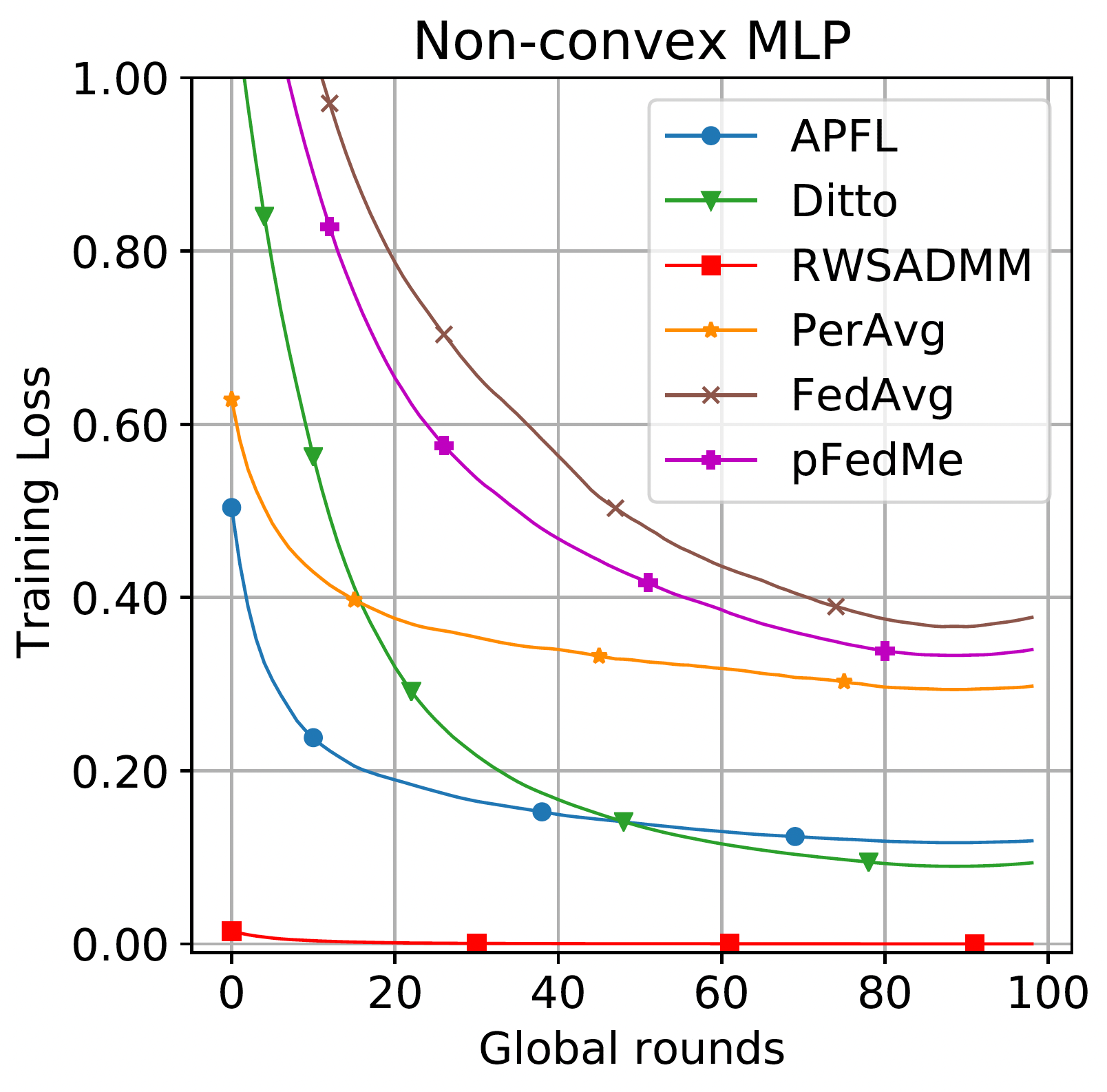}
        \caption{MLP loss}
        \label{fig:mlp-loss-mnist-all}
    \end{subfigure}
    \begin{subfigure}{0.25\textwidth}
        \includegraphics[trim=0pt 5pt 5pt 0pt, clip, width=\textwidth]{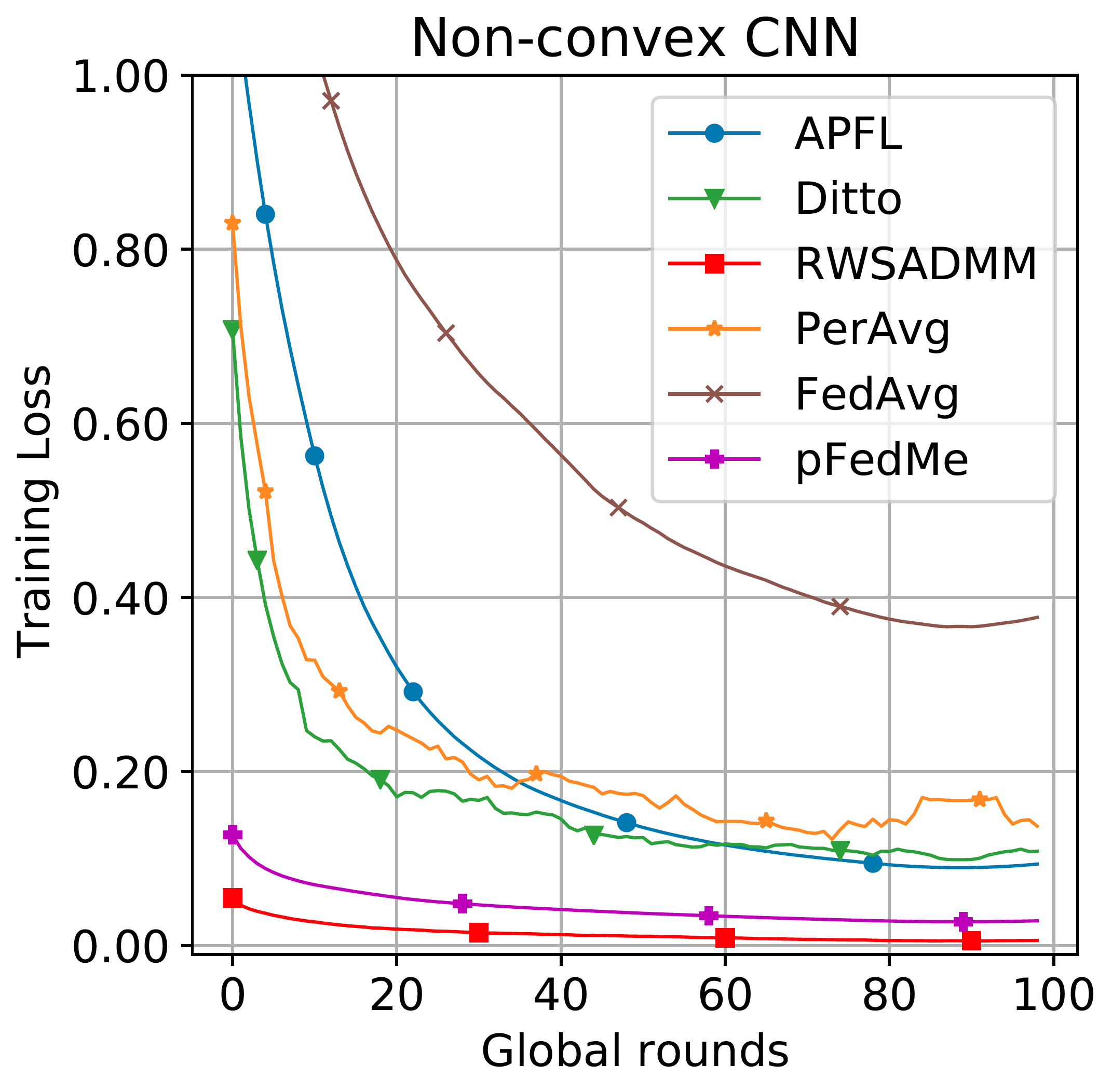}
        \caption{CNN loss}
        \label{fig:cnn-loss-mnist-all}
    \end{subfigure}
    \caption{Performance comparison (test accuracy and training loss) of RWSADMM, pFedMe, Per-Avg, FedAvg, APFL, and Ditto for MNIST dataset for the MLR (\ref{fig:mlr-acc-mnist-all}, \ref{fig:mlr-loss-mnist-all}), MLP (\ref{fig:mlp-acc-mnist-all}, \ref{fig:mlp-loss-mnist-all}), and CNN (\ref{fig:cnn-acc-mnist-all}, \ref{fig:cnn-loss-mnist-all}) models. The first 100 iterations are plotted to show the convergence progress better. }
   
    \label{fig: mnist-performance}
\end{figure}
(Synthetic and CIFAR10 figures are presented in Appendix D). Test accuracy and time cost for all the
datasets are reported in Table \ref{tab: accuracies-times}. 

The test accuracy progress curves of RWSADMM for all the models (\ref{fig:mlr-acc-mnist-all}-\ref{fig:cnn-acc-mnist-all}) have a significantly faster convergence. 
For the non-convex models (\ref{fig:mlp-acc-mnist-all}), RWSADMM reaches convergence after 200 iterations, while the rest of the algorithms, except Ditto, work toward convergence until 600 iterations. The performance curves are shown for 100 iterations for consistency. When tested on MNIST with MLP, RWSADMM demonstrated comparable performance against pFedMe. In the test on Synthetic data with MLR models, RWSADMM exhibited a significant advantage over the other methods, with an improved margin of $14.95\%$. Regarding computational efficiency, RWSADMM is slower than FedAvg and Per-FedAvg, but faster than pFedMe. Furthermore, RWSADMM converges in fewer iterations (200 iterations) than pFedMe (600 iterations). RWASDMM is also run for more extensive networks with 50 and 100 nodes as a separate set of experiments. The performance comparison results and diagrams are also in Appendix \ref{appendix: experiments-rest}.


\begin{table}[ht!] 
    \centering
    \scalebox{0.65}{%
    \begin{tabular}{|l|ll|ll|ll|ll|ll|ll|}
    \toprule
    \multirow{3}{*}{Frameworks}  & \multicolumn{6}{c|}{\textbf{MNIST}} & \multicolumn{4}{c|}{\textbf{Synthetic}}  \\ \cline{2-11}
     &  \multicolumn{2}{c|}{MLR}  & \multicolumn{2}{c|}{MLP} & \multicolumn{2}{c|}{CNN}  & 
     \multicolumn{2}{c|}{MLR}  & \multicolumn{2}{c|}{MLP}  \\ \cline{2-11}
      & acc(\%) & t(s)  & acc(\%) & t(s) & acc(\%) & t(s)
      & acc(\%) & t(s) & acc(\%) & t(s)  \\
     \toprule
    FedAvg  & $93.96\pm 0.02$  & 384 & $98.79\pm 0.03$  & 464 & $97.83\pm 0.15$ & 7965  & $77.62\pm 0.11$ & 592 & $83.64\pm 0.22$ & 680\\
    PerAvg  & $94.37\pm 0.04$ & 472 & $98.90\pm 0.02$ & 608 & $98.97\pm 0.08$ & 7296 & $81.49\pm 0.09$ & 800 & $85.01\pm 0.10$ & 808\\
    pFedMe  & $95.62\pm 0.04$  & 1344 &  \textbf{99.46$\pm$ 0.01}  & 2096 & $99.05\pm 0.06$ & 16623  & $83.20\pm 0.06$ & 760 & $86.36\pm 0.15$ & 4240\\
    Ditto  & $97.37 \pm 0.02$ & 828 & $97.79 \pm 0.03 $ & 1268 & $99.20 \pm 0.11$ & 9820 &  $86.24 \pm 0.03$ & 216 & $85.26 \pm 0.10$ &  236 \\
    APFL  &  $92.64 \pm 0.03 $ & 913 & $97.74 \pm 0.02$ & 1598 & $98.58 \pm 0.03$ & 17800 & $83.40 \pm 0.04 $ & 284 & $82.52 \pm 0.15 $ & 332  \\
    $\mathtt{RWSADMM}\text{ (our method)}$  & \textbf{98.63 $ \pm $ 0.01} & 500 & \textbf{99.29 $ \pm $ 0.02}  & 884 & \textbf{99.52 $ \pm $ 0.04} & 11570  & \textbf{96.44 $ \pm $ 0.12} & 1420 & \textbf{97.17 $ \pm $ 0.18} & 2076 \\
    \bottomrule
    \end{tabular}
    }    
    \scalebox{0.65}{%
    \begin{tabular}{|l|ll|ll|ll|}
    \toprule
    \multirow{3}{*}{Frameworks} & \multicolumn{6}{c|}{\textbf{CIFAR10}}  \\ \cline{2-7}
     &  \multicolumn{2}{c|}{MLR}  & \multicolumn{2}{c|}{MLP} & \multicolumn{2}{c|}{CNN}  \\ \cline{2-7}
     & acc(\%) & t(s)  & acc(\%) & t(s) & acc(\%) & t(s)  \\
     \toprule
    FedAvg   & $40.84\pm0.01$ & 480 & $41.02\pm0.05$ & 205 & $38.65\pm0.05$ & 235 \\
    PerAvg  & $47.43\pm0.09$ & 576 & $60.25\pm0.07$ & 760 & $83.52\pm0.01$ & 2401 \\
    pFedMe  & $67.53\pm0.34$ & 1544 & $78.12\pm0.38$ & 1020 &  $83.56\pm0.05$ & 10440 \\
    Ditto  & $75.2 \pm 0.01$ & 675 & $81.37 \pm 0.13 $ & 778 & $83.86 \pm 0.02$ & 6566  \\
    APFL  & $75.17\pm0.32$ & 150 & $78.00\pm0.18$ & 165 & $66.23\pm0.03$ & 2106   \\
    $\mathtt{RWSADMM}\text{ (our method)}$  & \textbf{80.72 $\pm $ 0.11} & 392 & \textbf{84.99 $ \pm $ 0.20} & 760 & \textbf{87.08 $ \pm $ 0.03} & 11276  \\
    \bottomrule
    \end{tabular}
    }

    \caption{Performance comparisons of FedAvg, Per-FedAvg, pFedMe, Ditto, APFL, and RWSADMM frameworks on MNIST, Synthetic, and CIFAR10 datasets. Three models are utilized for each dataset, and each model's converged accuracy (\%) and time consumption (seconds) are reported. Each configuration is executed for ten iterations, and variance is calculated to compute the degree of confidence for test accuracy rates. }
    \label{tab: accuracies-times}
\end{table}


\section{Conclusion and Future Work} \label{sec:conclusions}
\vspace{-5pt}
This study proposes a novel approach called RWSADMM, designed for systems with isolated nodes connected via wireless links to the mobile server without relying on pre-existing communication infrastructure. The algorithm enables the server to move randomly toward a local client, establishing local proximity among adjacent clients based on hard inequality constraints, addressing the challenge of data heterogeneity. Theoretical and experimental results demonstrate that RWSADMM is fast-converging and communication-efficient, surpassing current state-of-the-art FL frameworks. This study primarily focuses on the methodological framework for RWSADMM. Future research directions should explore essential techniques such as incorporating differential privacy techniques and examining scalability in more extensive network and dataset scenarios. Further investigation is needed to assess the implementation in physical networks and evaluate the effect of communication delays in the real world.



\newpage
\bibliographystyle{unsrt}  

\newpage
\appendix

\title{Mobilizing Personalized Federated Learning in Infrastructure-Less and Heterogeneous Environments via Random Walk Stochastic ADMM}

\maketitle 
\section{RWSADMM: Algorithm} 
\label{appendix: algorithm}

The RWSADMM scheme is as presented in Algorithm \ref{alg:RWSADMM}. Client $i_k$ is selected by Random Walk via $P(k)$, and $\vect y'_{i_{k}}$ is the token from the previous update.
Note that we only use one client in each derivation iteration. Still, it is straightforward to generalize the algorithm to have multiple active clients in $\mathcal{S}(i_k)\subset\mathcal{N}(i_k)$ simultaneously to stabilize the computation better as follows:
\begin{equation}
\label{eq:y_solution3}
\begin{aligned}
\scalemath{0.86}{\vect y_{i_{k}} = \vect y'_{i_{k}}
+\frac{1}{n_{i_k}S_{i_k}}\sum _{i_k\in \mathcal{S}(i_k)}\big[\vect x_{i_k} -(\frac{\vect z_{i_k} }{\beta} + \vect \varepsilon_{i_k} ) \odot sgn(\vect t_{i_k})  \big] 
-\sum_{{i_k}\in \mathcal{S}(i_k)} \big[\vect x'_{i_k} -(\frac{\vect z'_{i_k}}{\beta} + \vect \varepsilon_{i_k} ) \odot sgn(\vect t'_{i_k})  \big] } 
\end{aligned}
\end{equation}
where $S_{i_k}$ represents the volume of $\mathcal{S}(i_k)$. 
\vspace{8pt}

\begin{algorithm}
\caption{RWSADMM}\label{alg:RWSADMM}
\begin{algorithmic}[1]
\STATE \textbf{Initialization:} \\
Initialize Markov transition matrices $\scalemath{0.85}{\{\matrx P(0),\matrx P(1),\ldots,\}}$. \\
Initialize $\{\vect{x}_i^0\}_{i=1}^n = 0$, $\{\vect{z}_i^0\}_{i=1}^n = 0$, and 
\vspace{-10pt}
\begin{equation}\label{y1 computation eq}
    \vect{y}^1 = \frac{1}{n}\sum^n_{i=1}\big( \vect{x}_i^0 - \frac{\vect{z}_i^0}{\beta}\big) = 0 
\end{equation}
\vspace{-10pt}
\STATE \textbf{RWSADMM({$\beta,\vect{y}_1$}):} 
\REPEAT
\FOR {$k \in {0, 1, 2, ...}$}
\STATE Client $i_k$ receives $\vect{y}'_{i_{k}}$ and
updates $\matrx{X}$, $\matrx{Z}$, and $\matrx{y}$ using equations \eqref{eq:closeform_x}, \eqref{z update eq}, and \eqref{eq:y_solution2};
\ENDFOR \\
$\kappa = 0.99\times\kappa$
\UNTIL the termination condition is TRUE. \\
\textbf{RETURN $\vect{X}^*,\vect{y}^*$}
\end{algorithmic}
\end{algorithm}


\section{RWSADMM: Convergence Analysis}
\label{appendix: convergence}

Our proof of convergence for the proposed stochastic ADMM-based federated learning algorithm is non-trivial and non-straightforward. It introduces novel techniques to address the challenge of integrating the mobilized server and stochasticity into our federated ADMM framework. Specifically, we carefully consider the movement of the server, which is realized by using a dynamic Markov matrix. This is a significant novelty and challenge in the proof, as it is the first method introduced in federated learning that considers this type of server movement. By introducing assumptions on the dynamic Markov matrix and accounting for the dynamic behavior of the server, we are able to guarantee the convergence of our algorithm under certain conditions. While there are a few existing works, such as \cite{hu2023stochastic} that address convergence rates and properties of LADMM algorithms, they are not directly applicable to our random walk mobilization and stochastic update setting and are not directly adaptable to the unique requirements of federated learning. As such, our work represents a significant contribution to the field and provides valuable insights into the convergence properties of stochastic ADMM-based federated learning algorithms. We believe that our novel approach and careful consideration of the movement of the server will inspire further research and development in this area, leading to more effective and efficient federated learning algorithms. The proof sketch of the Convergence Theorem (Theorem \ref{convergence theorem}) is as follows. 

\begin{proof}
    \label{proof-sketch}
    The proof sketch is summarized as follows.
    \begin{enumerate}
        \item Under Assumption \ref{f is lsmooth and l-lipschitz}, the sequence created by the RWSADMM, i.e., $(\vect{y}^k,\matrx{X}^k,\matrx{Z}^k)_{k>T}$, satisfies
        \vspace{-5pt}
            \begin{equation}
                \label{update dual bound update primal lemma-eq0} 
                \scalemath{0.9}{
                \norm{\matrx{Z}^{k+1}-\matrx{Z}^k} =\sum \norm{\vect{z}_{i_k}^{k+1}-\vect{z}_{i_k}^k}\leq L\norm{\matrx{X}^{\tau(k,i_k)+1}-\matrx{X}^{\tau(k,i_k)}} } \notag 
            \end{equation}
            \vspace{-12pt}
        \item Recall $L_\beta^k$ defined in Eq. \eqref{lyapEquations}. For $k\geq 0$, the RWSADMM iterates satisfy
        \vspace{-5pt}
            \begin{equation} \label{lemma - descent of augmented lagrangian - eq1}
            \scalemath{0.9}{     
            L_\beta^k -L_\beta\left(\vect{y}^{k+1},\matrx{X}^k;\matrx{Z}^k\right)\geq \norm{\vect{y}^k -\vect{y}^{k+1}}^2 } \notag
            \end{equation}
            \vspace{-12pt}
        \item Under Assumption \ref{f is lsmooth and l-lipschitz}, for $\beta > L$ and $\forall k > T$,
        \vspace{-5pt}
            \begin{align} \label{lemma - lower bound for augmented lagrangian update - eq1}
                \scalemath{0.90}{ L_\beta(\vect{y}^{k+1},\matrx{X}^k;\matrx{Z}^k) - L_\beta^{k+1} \geq 
                 \frac{\beta - L}{2n}\norm{\matrx{X}^k - \matrx{X}^{k+1}}^2 - \frac{L^2}{n\beta}\norm{\matrx{X}^{\tau(k,i_k)+1}-\matrx{X}^{\tau(k,i_k)}}^2 } \notag
            \end{align}
            \vspace{-12pt}
        \item Recall $M_\beta^k$ defined in Eq. \eqref{lyapEquations}; under Assumption \ref{f is lsmooth and l-lipschitz}, for $\beta>2L^2+L+2$ and $\forall k > T$,
            \vspace{-5pt}
            \begin{align} 
                \scalemath{0.90}{
                M_\beta^k - M_\beta^{k+1}  \geq 
                \frac{\beta}{2}\norm{\vect{y}^k - \vect{y}^{k+1}}^2
                +\frac{1}{n}\norm{\matrx{X}^k - \matrx{X}^{k+1}}^2 
                +\frac{L^2}{2n}\norm{\matrx{X}^{\tau(k,i_k)+1}-\matrx{X}^{\tau(k,i_k)}}^2 }
            \end{align}
            \vspace{-12pt}
        \item For $\beta>2L^2+L+2$, RWSADMM ensures a lower bounded sequence $(M_\beta^k)_{k\geq0}$.
    \end{enumerate}
\end{proof}

The proof details are provided in the following. 
Several steps are taken to prove Lemma \ref{convergence lemma}. We introduce several Lemmas to represent these steps (Lemma \ref{update dual bound update primal}-\ref{lemma - lower bound for lyapunov function}). Lemma \ref{update dual bound update primal} shows that the update of the primal variable can bound the update on the dual variable. 

\begin{lemma}
    \label{update dual bound update primal}
Under Assumption \ref{f is lsmooth and l-lipschitz}, the sequence created by RWSADMM, $(\vect{y}^k,\matrx{X}^k,\matrx{Z}^k)_{k>T}$, satisfies,

\begin{equation}\label{update dual bound update primal lemma-eq1}
    \mathbb{E}\norm{\matrx{Z}^{k+1}-\matrx{Z}^k} = \mathbb{E}\norm{\vect{z}_{i_k}^{k+1}-\vect{z}_{i_k}^k} \leq L \cdot \mathbb{E}\norm{\matrx{X}^{\tau(k,i_k)+1}-\matrx{X}^{\tau(k,i_k)}} 
\end{equation}
\end{lemma}
\begin{proof}
    Note that client $i_k$ is activated at iteration $k$. Denote $\vect x_i^k$, $\vect z_i^k$, and $\vect x_i^k$ as the three groups of variables owned by any client $i$ ($1\leq\forall i\leq n$) at iteration $k$. Under the Assumption \ref{expectation of estimation}, the optimality condition of $\matrx{X}$ update for $i=i_k$ implies that
\begin{equation}\label{update dual bound update primal lemma-eq2}
    \mathbb{E}_\xi[sgn(\vect t')(g_i(\vect{x}_i^k,\xi)+\vect \varepsilon_{i}^k-\vect{z}_i^k) + \beta(\vect{y}^{k+1} - \vect{x}_{i_k}^{k+1})] = 0 
\end{equation}

Substituting Eq. \eqref{update dual bound update primal lemma-eq2} into Eq. (\ref{z update eq}) yields
\begin{equation}\label{update dual bound update primal lemma-eq3}
    \mathbb{E}_\xi[g_i(\vect{x}_i^k,\xi)] = \mathbb{E}_\xi[\vect{z}_{i_k}^{k+1}], \; \text{for}\; i=i_k 
\end{equation}

Hence, for $i=i_k$, we have

\begin{equation}\label{update dual bound update primal lemma-eq4}
\begin{split}
    &\mathbb{E}\norm{\vect{z}_{i}^{k+1}-\vect{z}_{i}^k} \overset{(a)}{=} \mathbb{E}\norm{\vect{z}_{i}^{k+1}-\vect{z}_{i}^{\tau(k,i)+1}} 
    \overset{(\ref{update dual bound update primal lemma-eq3})}{=}\norm{\mathbb{E}_\xi(g_i(\vect{x}_i^k,\xi)+\epsilon^{k+1}) - \mathbb{E}_\xi(g_i(\vect{x}_i^{\tau(k,i)},\xi)+\epsilon^{k})} \\
    &\overset{(b)}{=}\norm{\mathbb{E}_\xi(g_i(\vect{x}_i^{\tau(k,i)+1},\xi)) - \mathbb{E}_\xi(g_i(\vect{x}_i^{\tau(k,i)},\xi))} 
    \overset{(\ref{l-lipschitz}\&\ref{Eg=gF})}{\le}L \cdot \mathbb{E}_\xi\norm{\vect{x}_i^{\tau(k,i)+1}-\vect{x}_i^{\tau(k,i)}}
\end{split}
\end{equation}
where $\tau(k,i)$ is defined in Eq. \eqref{tau def}. The equality $(a)$ holds because $\vect{z}_{i}^k = \vect{z}_{i}^{\tau(k,i)+1}$, and equality $(b)$ holds because $\vect{x}_{i}^k = \vect{x}_{i}^{\tau(k,i)+1}$. On the other hand, when $i\neq i_k$, agent $i$ is not activated at iteration $k$, so $\scalemath{0.8}{\mathbb{E}_\xi[\norm{\vect{z}_{i_k}^{k+1}-\vect{z}_{i_k}^k}] = L\mathbb{E}_\xi[\norm{\vect x_i^{k+1}-\vect x_i^k}]=0}$. Therefore, we have the proof of Eq. \eqref{update dual bound update primal lemma-eq1}.
\end{proof}

Lemma \ref{lemma - descent of augmented lagrangian} shows that the $\vect{y}$-update in RWSADMM provides a sufficient descent of the augmented Lagrangian.

\begin{lemma}
    \label{lemma - descent of augmented lagrangian}
Recall $L_\beta^k$ defined in Eq. \eqref{lyapEquations}. For $k\ge 0$, RWSADMM iterates satisfies 
\begin{equation} \label{lemma - descent of augmented lagrangian - eq2}
    L_\beta^k -L_\beta\left(\vect{y}^{k+1},\matrx{X}^k;\matrx{Z}^k\right)\geq \norm{\vect{y}^k -\vect{y}^{k+1}}^2 
\end{equation}

\end{lemma}

\begin{proof}
    \label{proof of lemma - descent of augmented lagrangian}

We rewrite the augmented Lagrangian function $L$ in Eq. \eqref{problem formulation eq3} by adding and subtracting the term $\norm{\matrx{Z}}^2 / 2\beta$ to the RHS of the equation and rearranging the terms. One can obtain
\begin{equation} \label{proof of lemma - descent of augmented lagrangian - eq1}
    L_\beta\left(\vect{Y},\matrx{X};\matrx{Z}\right) = \frac{1}{n}\left( F(\matrx{X})+ \frac{\beta}{2} \norm{|\mathbf{1}\otimes \vect{y} - \matrx{X}|-\vect \epsilon + \frac{\matrx{Z}}{\beta}}^2
    - \frac{\norm{\matrx{Z}}^2}{2\beta} \right) 
\end{equation}

So, the augmented Lagrangian update is

\begin{equation} \label{proof of lemma - descent of augmented lagrangian - eq2}
\begin{split}
    &\scalemath{0.8}{L_\beta^k - L_\beta \left(\vect{y}^{k+1},\matrx{X}^k;\matrx{Z}^k \right)
    =\frac{\beta}{2n} \norm{\abs{\mathbf{1}\otimes \vect{y}^k - \matrx{X}^k}-\vect \epsilon^k + \frac{\matrx{Z}^k}{\beta}}^2 -\frac{\beta}{2n}\norm{\abs{\mathbf{1}\otimes \vect{y}^{k+1} - \matrx{X}^k }-\vect \epsilon^k + \frac{\matrx{Z}^k}{\beta}}^2}  \\
    &\scalemath{0.8}{\overset{(a)}{=}\frac{\beta}{2n}\sum^n_{i=1}\left(\norm{\vect{y}^k -\vect{y}^{k+1}}^2+2\left< \abs{\vect{y}^{k+1}-\vect{x}^k_i}-\vect \epsilon^k+\frac{\vect{z}^k_i}{\beta},\vect{y}^k-\vect{y}^{k+1}\right> \right)
    \geq\frac{\beta}{2}\norm{\vect{y}^k -\vect{y}^{k+1}}^2 - \left<\vect{d}^k, \vect{y}^k-\vect{y}^{k+1}\right>}
    \end{split}
\end{equation}

The equality (a) is achieved by applying the cosine identity $\norm{b+c}^2 - \norm{a+c}^2 = \norm{b-a}^2 + 2<a+c,b-a>$, and $\vect{d}^k$ is defined as

\begin{equation}
   \vect{d}^k := -\frac{\beta}{n}\sum^n_{i=1}\left(  \abs{\vect{y}^{k+1}-\vect{x}^k_i}-\vect \epsilon^k+\frac{\vect{z}^k_i}{\beta}\right) = 0 
\end{equation}

which results from $\vect{y}-$update in the RWSADMM algorithm.
\end{proof}


In Lemma \ref{lemma - lower bound for augmented lagrangian update}, the lower bound of descent in the augmented Lagrangian over the updates of $\matrx{X}$ and $\matrx{Z}$ is derived.

\begin{lemma}
    \label{lemma - lower bound for augmented lagrangian update}

Recall $L_\beta^k$ defined in Eq. \ref{lyapEquations}. Under Assumption \ref{f is lsmooth and l-lipschitz}, for $\beta > L$ and $\forall k > T$,

\begin{equation} 
\label{lemma - lower bound for augmented lagrangian update - eq2}
\begin{aligned}
    &\scalemath{0.8}{
    \mathbb{E}[L_\beta(\vect{y}^{k+1},\matrx{X}^k;\matrx{Z}^k) - L_\beta^{k+1}] } 
    &\scalemath{0.8}{
    \geq\frac{\gamma - L}{2n}\mathbb{E}\norm{\matrx{X}^k - \matrx{X}^{k+1}}^2 
    -\frac{L^2}{n\beta}\mathbb{E}\norm{\matrx{X}^{\tau(k,i_k)+1}-\matrx{X}^{\tau(k,i_k)}}^2 }
\end{aligned}
\end{equation}
\end{lemma}

\begin{proof}
    \label{proof of lemma - lower bound for augmented lagrangian update}

For the augmented Lagrangian $L_\beta$, we derive


\begin{equation}
\begin{split}
\label{proof of lemma - lower bound for augmented lagrangian update - eq1}
    &~~~~L_\beta(\vect{y}^{k+1},\matrx{X}^k;\matrx{Z}^k) - L_\beta^{k+1} \\
    &=\frac{1}{n} ( f_{i_k}(\vect{x}_{i_k}^k) +\inproduct{\vect{z}_{i_k}^k, \abs{\vect{y}^{k+1} -\vect{x}_{i_k}^k}-\vect \varepsilon_{i_k}^k} + \frac{\beta}{2}\norm{\abs{\vect{y}^{k+1} -\vect{x}_{i_k}^k}-\vect \varepsilon_{i_k}^k}^2 \\
    &-f_{i_k}(\vect{x}_{i_k}^{k+1}) -\inproduct{\vect{z}_{i_k}^{k+1}, \abs{\vect{y}^{k+1}-\vect{x}_{i_k}^{k+1}}-\vect \varepsilon_{i_k}^k} - \frac{\beta}{2}\norm{\abs{\vect{y}^{k+1}-\vect{x}_{i_k}^{k+1}}-\vect \varepsilon_{i_k}^k}^2 ) \\
    &=\frac{1}{n} ( f_{i_k}(\vect{x}_{i_k}^k) -\inproduct{\vect{z}_{i_k}^{k+1}, \abs{\vect{y}^{k+1}-\vect{x}_{i_k}^{k+1}}-\vect \varepsilon_{i_k}^k}-f_{i_k}(\vect{x}_{i_k}^{k+1}) \\
    &+\inproduct{\vect{z}_{i_k}^k, \abs{\vect{y}^{k+1} -\vect{x}_{i_k}^k}-\vect \varepsilon_{i_k}^k} -\frac{\beta}{2}\left[\norm{\abs{\vect{y}^{k+1}-\vect{x}_{i_k}^{k+1}}-\vect \varepsilon_{i_k}^k}^2 -\norm{\abs{\vect{y}^{k+1} -\vect{x}_{i_k}^k}-\vect \varepsilon_{i_k}^k}^2 \right] )  \\
    &\overset{(a)}{=}\frac{1}{n} ( f_{i_k}(\vect{x}_{i_k}^k) -\inproduct{\vect{z}_{i_k}^{k+1}, \abs{\vect{y}^{k+1}-\vect{x}_{i_k}^{k+1}}-\vect \varepsilon_{i_k}^k}-f_{i_k}(\vect{x}_{i_k}^{k+1}) +\inproduct{\vect{z}_{i_k}^k, \abs{\vect{y}^{k+1} -\vect{x}_{i_k}^k}-\vect \varepsilon_{i_k}^k} \\
    &-\frac{\beta}{2}\left[\norm{\vect{x}_{i_k}^{k+1} -\vect{x}_{i_k}^k}^2 +2\inproduct{\abs{\vect{y}^{k+1}-\vect{x}_{i_k}^{k+1}}-\vect \varepsilon_{i_k},\vect{x}_{i_k}^{k+1} -\vect{x}_{i_k}^k} \right]  ) \\
    &=\frac{1}{n} ( f_{i_k}(\vect{x}_{i_k}^k) -\inproduct{\vect{z}_{i_k}^{k+1}, \abs{\vect{y}^{k+1}-\vect{x}_{i_k}^{k+1}}-\vect \varepsilon_{i_k}^k}-f_{i_k}(\vect{x}_{i_k}^{k+1}) \\
    &+\inproduct{\vect{z}_{i_k}^k, \abs{\vect{y}^{k+1} -\vect{x}_{i_k}^k}-\vect \varepsilon_{i_k}^k} -\frac{\beta}{2}\norm{\vect{x}_{i_k}^{k+1} -\vect{x}_{i_k}^k}^2 \\
    &-\beta\inproduct{\abs{\vect{y}^{k+1}-\vect{x}_{i_k}^{k+1}}-\vect \varepsilon_{i_k},\vect{x}_{i_k}^{k+1} -\vect{x}_{i_k}^k}  ) \\
    &\overset{(b)}{=}\frac{1}{n}( f_{i_k}(\vect{x}_{i_k}^k) -f_{i_k}(\vect{x}_{i_k}^{k+1}) +\inproduct{\vect{z}_{i_k}^k,\frac{\vect{z}_{i_k}^{k+1}-\vect{z}_{i_k}^k}{\beta} +\vect{x}_{i_k}^{k+1}-\vect{x}_{i_k}^k} -\frac{\beta}{2}\norm{\vect{x}_{i_k}^{k+1} -\vect{x}_{i_k}^k}^2 \\
    &-\inproduct{\vect{z}_{i_k}^{k+1},\frac{\vect{z}_{i_k}^{k+1}-\vect{z}_{i_k}^k}{\beta}} -\beta\inproduct{\frac{\vect{z}_{i_k}^{k+1} -\vect{z}_{i_k}^k}{\beta},\vect{x}_{i_k}^{k+1} -\vect{x}_{i_k}^k} ) \\
    &=\frac{1}{n} ( f_{i_k}(\vect{x}_{i_k}^k) +\frac{\vect{z}_{i_k}^{k+1}\vect{z}_{i_k}^k}{\beta} -\frac{(\vect{z}_{i_k}^k)^2}{\beta} +\vect{z}_{i_k}^k(\vect{x}_{i_k}^{k+1}-\vect{x}_{i_k}^k) -f_{i_k}(\vect{x}_{i_k}^{k+1}) \\
    &+\frac{\beta}{2}\norm{\vect{x}_{i_k}^k-\vect{x}_{i_k}^{k+1}}^2 -\frac{(\vect{z}_{i_k}^{k+1})^2}{\beta} +\frac{\vect{z}_{i_k}^{k+1}\vect{z}_{i_k}^k}{\beta} -(\vect{z}_{i_k}^{k+1}-\vect{z}_{i_k}^k)(\vect{x}_{i_k}^k-\vect{x}_{i_k}^{k+1}) )   \\
    &= \frac{1}{n} ( f_{i_k}(\vect{x}_{i_k}^k) -f_{i_k}(\vect{x}_{i_k}^{k+1}) -\frac{1}{\beta}(\vect{z}_{i_k}^k)^2 +(\vect{z}_{i_k}^{k+1})^2 -2(\vect{z}_{i_k}^{k+1}\vect{z}_{i_k}^k) \\
    &+\frac{\beta}{2}\norm{\vect{x}_{i_k}^k -\vect{x}_{i_k}^{k+1}}^2 -\vect{z}_{i_k}^{k+1}(\vect{x}_{i_k}^k -\vect{x}_{i_k}^{k+1}) ) \\
    &\overset{(c)}{=}\frac{1}{n} ( f_{i_k}(\vect{x}_{i_k}^k) +\frac{\beta}{2}\norm{\vect{x}_{i_k}^k -\vect{x}_{i_k}^{k+1}}^2 -f_{i_k}(\vect{x}_{i_k}^{k+1}) -\frac{1}{\beta}\norm{\vect{z}_{i_k}^{k+1} -\vect{z}_{i_k}^k}^2 
    -\inproduct{\vect{z}_{i_k}^{k+1},\vect{x}_{i_k}^k -\vect{x}_{i_k}^{k+1}}  ) 
\end{split}
\end{equation}

Where equality (a) holds due to the cosine identity $\norm{b+c}^2 - \norm{a+c}^2 = \norm{b-a}^2 + 2<a+c,b-a>$, equality (b) holds because of $\vect{y}-$update in RWSADMM, and equality (c) holds due to recursion of $\vect{y}-$update in RWSADMM.
Next, we apply the stochastic property to Eq. (\ref{proof of lemma - lower bound for augmented lagrangian update - eq1}) using Eq. (\ref{update dual bound update primal lemma-eq3}),

\begin{equation}
\label{proof of lemma - lower bound for augmented lagrangian update - eq2}
\begin{aligned}
    &\mathbb{E}[L_\beta(\vect{y}^{k+1},\matrx{X}^k;\matrx{Z}^k) - L_\beta^{k+1}] \\
    &= \frac{1}{n} ( \mathbb{E}_\xi[f_{i_k}(\vect{x}_{i_k}^k) -f_{i_k}(\vect{x}_{i_k}^{k+1})] -<E_\xi(g_i(\vect{x}_i^k,\xi)+\epsilon^k),\vect{x}_{i_k}^k -\vect{x}_{i_k}^{k+1}> \\
    &+\frac{\beta}{2}E_\xi\norm{\vect{x}_{i_k}^k -\vect{x}_{i_k}^{k+1}}^2  -\frac{1}{\beta}\mathbb{E}_\xi\norm{\vect{z}_{i_k}^{k+1} -\vect{z}_{i_k}^k}^2  ) \\
    &\overset{(a)}{\geq}\frac{1}{n} \left( -\frac{L}{2}  \mathbb{E}_\xi \norm{\vect{x}_{i_k}^k -\vect{x}_{i_k}^{k+1}}^2 +\frac{\beta}{2}\mathbb{E}_\xi \norm{\vect{x}_{i_k}^k -\vect{x}_{i_k}^{k+1}}^2  -\frac{1}{\beta}\mathbb{E}_\xi \norm{\vect{z}_{i_k}^{k+1} -\vect{z}_{i_k}^k}^2  \right) \\
    &\overset{(b)}{\geq}\frac{1}{n} ( -\frac{L}{2}\mathbb{E}_\xi \norm{\vect{x}_{i_k}^k -\vect{x}_{i_k}^{k+1}}^2
    +  \frac{\beta}{2}\mathbb{E}_\xi\norm{\vect{x}_{i_k}^k -\vect{x}_{i_k}^{k+1}}^2 
    -\frac{L^2}{\beta}\mathbb{E}_\xi \norm{\vect{x}_{i_k}^{\tau(k,i_k)+1} -\vect{x}_{i_k}^{\tau(k,i_k)}}^2 ) \\
    &=\frac{\beta-L}{2n}\mathbb{E}_\xi \norm{\vect{x}_{i_k}^k -\vect{x}_{i_k}^{k+1}}^2
    -\frac{L^2}{n\beta} \mathbb{E}_\xi\norm{\vect{x}_{i_k}^{\tau(k,i_k)+1} -\vect{x}_{i_k}^{\tau(k,i_k)}}^2
\end{aligned}
\end{equation}

Where inequality (a) is achieved using Assumption \ref{f is lsmooth and l-lipschitz} (Eq. \eqref{l-lipschitz-2}), and inequality (b) is due to Lemma \ref{update dual bound update primal} (Eq. \eqref{update dual bound update primal lemma-eq1}).
\end{proof}


In Lemma \ref{lemma - sufficient descent in Lyapunov functions}, the sufficient descent in Lyapunov functions is established.

\begin{lemma}
    \label{lemma - sufficient descent in Lyapunov functions}

Recall $M_\beta^k$ defined in Eq. \eqref{lyapEquations}; under Assumption \ref{f is lsmooth and l-lipschitz}, for $\beta>2L^2+L+2$ and $\forall k > T$,
\begin{equation}
    \label{lemma - sufficient descent in Lyapunov functions - eq1}
    \begin{split}
    M_\beta^k - M_\beta^{k+1} \geq \frac{\beta}{2}\norm{\vect{y}^k - \vect{y}^{k+1}}^2
    +\frac{1}{n}\norm{\matrx{X}^k - \matrx{X}^{k+1}}^2 
    +\frac{L^2}{2n}\norm{\matrx{X}^{\tau(k,i_k)+1}-\matrx{X}^{\tau(k,i_k)}}^2 
    \end{split}
\end{equation}
\end{lemma}

\begin{proof}
    \label{proof of lemma - sufficient descent in Lyapunov functions}

Using Eq. \eqref{lyapEquations}, we can attain
\begin{equation}
    \label{proof of lemma - sufficient descent in Lyapunov functions - eq1}
\begin{aligned}
&~~~~M_\beta^k - M_\beta^{k+1}\\
&= L_\beta^k - L_\beta^{k+1} +\frac{L^2}{n}\big( \norm{\vect{x}_{i}^{\tau(k,i_k)+1}-\vect{x}_{i}^{\tau(k,i_k)}}^2 
-\norm{\vect{x}_{i}^{\tau(k+1,i_k)+1}-\vect{x}_{i}^{\tau(k+1,i_k)}}^2 \big) \\
&\overset{(a)}{=} L_\beta^k - L_\beta^{k+1} +\frac{L^2}{n}\big( \norm{\vect{x}_{i_k}^{\tau(k,i_k)+1}-\vect{x}_{i_k}^{\tau(k,i_k)}}^2 - \norm{\vect{x}_{i_k}^{k+1}-\vect{x}_{i_k}^k}^2 \big) \\
&= L_\beta^k - L_\beta^{k+1} +\frac{L^2}{n}\big( \norm{\matrx{X}^{\tau(k,i_k)+1}-\matrx{X}^{\tau(k,i_k)}}^2 - \norm{\matrx{X}^{k+1} - \matrx{X}^k}^2 \big) 
\end{aligned}
\end{equation}

where inequality (a) is due to the following property,

\begin{equation}
    \label{proof of lemma - sufficient descent in Lyapunov functions - eq2}
    x_i^{\tau(k+1,i)+1} -x_i^{\tau(k+1,i)} =
    \begin{cases}
        x_i^{k+1} -x_i^k, \;\; i = i_k \\
        x_i^{\tau(k,i)+1} -x_i^{\tau(k,i)}, \; otherwise
    \end{cases} 
\end{equation}

Substituting Eq. \eqref{lemma - descent of augmented lagrangian - eq1}, one can obtain
\begin{equation}
    \label{proof of lemma - sufficient descent in Lyapunov functions - eq3}
\begin{aligned}
&\scalemath{0.9}{
M_\beta^k - M_\beta^{k+1} \geq L_\beta\left(\vect{y}^{k+1},\matrx{X}^k;\matrx{Z}^k\right) -L_\beta^{k+1} +\norm{\vect{y}^k -\vect{y}^{k+1}}^2  
+\frac{L^2}{n}\big( \norm{\matrx{X}^{\tau(k,i_k)+1}-\matrx{X}^{\tau(k,i_k)}}^2 - \norm{\matrx{X}^{k+1} - \matrx{X}^k}^2 \big) }
\end{aligned}
\end{equation}

One can also substitute Eq. \eqref{lemma - lower bound for augmented lagrangian update - eq1} into Eq. \eqref{proof of lemma - sufficient descent in Lyapunov functions - eq3}, which leads to
\begin{equation}
    \label{proof of lemma - sufficient descent in Lyapunov functions - eq4}
\begin{aligned}
&M_\beta^k - M_\beta^{k+1} \\
&\geq\frac{\beta - L}{2n}\norm{\matrx{X}^k - \matrx{X}^{k+1}}^2 -\frac{L^2}{n\beta}\norm{\matrx{X}^{\tau(k,i_k)+1} -\matrx{X}^{\tau(k,i_k)}}^2 \\
&+\norm{\vect{y}^k -\vect{y}^{k+1}}^2+\frac{L^2}{n}\big( \norm{\matrx{X}^{\tau(k,i_k)+1}-\matrx{X}^{\tau(k,i_k)}}^2 - \norm{\matrx{X}^{k+1} - \matrx{X}^k}^2 \big) \\
&=\frac{\beta - L -2L^2}{2n}\norm{\matrx{X}^k - \matrx{X}^{k+1}}^2 +\frac{L^2(\beta-1)}{n\beta}\norm{\matrx{X}^{\tau(k,i_k)+1} -\matrx{X}^{\tau(k,i_k)}}^2 
+\norm{\vect{y}^k -\vect{y}^{k+1}}^2 
\end{aligned}
\end{equation}

Using $\frac{\beta}{2} - \frac{L}{2} - L^2 \geq 1$, $1 -\frac{1}{\beta}>\frac{1}{2}$, and $\beta < 2$, we complete the proof of Eq. \eqref{lemma - sufficient descent in Lyapunov functions - eq1}.
\end{proof}


Lemma \ref{lemma - lower bound for lyapunov function} states that the Lyapunov function $M_\beta^k$ is lower bounded.

\begin{lemma}
    \label{lemma - lower bound for lyapunov function}
For $\beta>2L^2+L+2$, RWSADMM ensures a lower bounded sequence $(M_\beta^k)_{k\geq0}$ in expectation.
\end{lemma}

\begin{proof}
     \label{proof of lemma - lower bound for lyapunov function}
For $k>T$, we have

\begin{equation}
    \label{proof of lemma - lower bound for lyapunov function - eq1}
\begin{aligned}
    &\mathbb{E}[M_\beta^k]\\
    &=\mathbb{E} [L_\beta^k] +\frac{L^2}{n} \sum^n_{i=1}\mathbb{E}\norm{\matrx{x}_i^{\tau(k,i_k)+1}-\matrx{x}_i^{\tau(k,i_k)}}^2\\
    &=\frac{1}{n}\sum^n_{j=1} \mathbb{E} \bigg(f_j(\vect{x}_j^k) +\inproduct{\vect{z}_j^k, \abs{\vect{y}^k -\vect{x}_j^k}-\vect \varepsilon_j^k}\bigg)
    +\frac{\beta}{2n}\sum^n_{j=1} \mathbb{E} \norm{ \vect{y}_j^k - \matrx{X}_j^k}^2  \\
    &+\frac{L^2}{n} \sum^n_{i=1}\mathbb{E} \norm{\matrx{x}_i^{\tau(k,i_k)+1}-\matrx{x}_i^{\tau(k,i_k)}}^2\\
    &\scalemath{0.9}{\overset{(a)}{=}\frac{1}{n}\sum^n_{j=1} \mathbb{E} \bigg(f_j(\vect{x}_j^k) +\inproduct{\mathbb{E}_\xi(g_j(\vect{x}_j^{\tau(k,j)},\xi)), \vect{y}^k -\vect{x}_j^k}\bigg)
    +\frac{\beta}{2n}\sum^n_{j=1} \mathbb{E} \norm{ \vect{y}_j^k - \matrx{X}_j^k}^2  
    +\frac{L^2}{n} \sum^n_{i=1}\mathbb{E} \norm{\matrx{x}_i^k-\matrx{x}_i^{\tau(k,i_k)}}^2 } \\
    &\overset{(b)}{\geq}\frac{1}{n}\sum^n_{j=1} \mathbb{E} \bigg(f_j(\vect{y^k}) +\inproduct{\mathbb{E}_\xi(g_j(\vect{x}_j^{\tau(k,j)},\xi)) - \mathbb{E}_\xi(g_j(\vect{x}_i^k,\xi)), \vect{y}^k -\vect{x}_j^k}\bigg)\\
    &+\frac{\beta -L}{2n}\sum^n_{j=1} \mathbb{E} \norm{ \vect{y}_j^k - \matrx{X}_j^k}^2
    +\frac{L^2}{n} \sum^n_{i=1}\mathbb{E} \norm{\matrx{x}_i^k-\matrx{x}_i^{\tau(k,i_k)}}^2 \\
    &\overset{(c)}{\geq}\frac{1}{n}\sum^n_{j=1} \mathbb{E} \bigg(f_j(\vect{y^k}) +\norm{\mathbb{E}_\xi(g_j(\vect{x}_i^{\tau(k,j)},\xi)) - \mathbb{E}_\xi(g_j(\vect{x}_i^k,\xi)}^2\bigg)\\
    &+\frac{\beta -L -2}{2n} \mathbb{E} \norm{\mathbf{1}\otimes \vect{y}^k - \matrx{X}^k}^2
    +\frac{L^2}{n} \sum^n_{i=1}\mathbb{E} \norm{\matrx{x}_i^k-\matrx{x}_i^{\tau(k,i_k)}}^2 \\
    &\overset{(d)}{\geq}\min_{\vect y}\bigg\{\frac{1}{n}\sum^n_{j=1}\mathbb{E}  f_j(\vect{y}^k)\bigg\}
    +\frac{2L^2}{2n} \mathbb{E} \norm{\mathbf{1}\otimes \vect{y}^k - \matrx{X}^k}^2  \\
    &\overset{(e)}{\geq} -\infty 
\end{aligned}
\end{equation}

where (a) holds due to Eq. \eqref{update dual bound update primal lemma-eq3}. (b) holds because $f_j$ is Lipschitz differentiable (Eq. \eqref{l-lipschitz for expectation of estimation - 2}). (c) holds due to Young's inequality. (d) follows from the assumption $\beta>2L^2+L+2$ and the Lipschitz smoothness of each $f_j$, and (e) holds due to Assumption \ref{f coercive and bounded below}. Therefore, $M_\beta^k$ is bounded from below in expectation. \end{proof}


Now we can prove Lemma \ref{convergence lemma} with the above lemmas.

\begin{proof}

Recall that the maximum hitting time $T$ is almost surely finite. For Statement 1, the monotonicity of Lyapunov function $(M_\beta^k)_{k>T}$ in Lemma \ref{lemma - sufficient descent in Lyapunov functions} and their lower boundedness in Lemma \ref{lemma - lower bound for lyapunov function} ensure convergence of $(M_\beta^k)_{k\geq0}$.
For Statement 2, consider Statement 1 and the lower boundedness of $(M_\beta^k)_{k>T}$ in Lemma \ref{lemma - lower bound for lyapunov function} (Eq. \eqref{proof of lemma - lower bound for lyapunov function - eq1}). $\frac{1}{n}F(\vect{y}^k)$ is upper bounded by $\scalemath{0.8}{\max\{ \max_{t\in\{0,\ldots,T\}}\{\frac{1}{n}F(\vect{y}^t)\},M_\beta^{T+1} \}}$;
and $\scalemath{0.8}{\norm{\mathbf{1}\otimes \vect{y}_j^k - \matrx{X}_j^k}^2}$ is upper bounded by $\scalemath{0.8}{\max\{ \max_{\{t\in{0,\ldots,T}\}}\{\norm{\mathbf{1}\otimes \vect{y}_j^t - \matrx{X}_j^t}^2\},L_\beta^{T+1} \}}$. By Assumption \ref{f coercive and bounded below}, the sequence $\{\vect{y}|k=\{0,1,\ldots,\}\}$ is bounded. The boundedness of $\scalemath{0.8}{\norm{\mathbf{1}\otimes \vect{y}_j^k - \matrx{X}_j^k}^2}$ further leads to that of $\scalemath{0.8}{\{\matrx{X}_j^k|k=\{0,1,\ldots,\}\}}$. Finally, Eq. \eqref{update dual bound update primal lemma-eq3} and Assumption \ref{f is lsmooth and l-lipschitz} ensure $(\matrx{Z}^k)$ is bounded as well. Altogether, $(\vect{y}^k, \matrx{X}^k, \matrx{Z}^k)$ is bounded.
\end{proof}


Based on Lemma \ref{convergence lemma}, the convergence of the subgradients of $L_\beta^k$ can be established as follows.
\vspace{+2pt}

\begin{lemma}
    \label{convergence lemma2}
With Assumption \ref{noise variance} and $\beta$ given in Lemma \ref{convergence lemma}, for any given subsequence (including the whole sequence) with its index $(k_s)_{s\geq0}$, there exists a sequence $(g^k)_{k\geq0}$ with $(g^k)\in \partial L_\beta^{k+1}$ containing an almost surely convergent subsequence $(g^{k_{s_j}})_{j\geq0}$, that is,

\begin{equation*}
   Pr \left( \underset{j\rightarrow\infty}{\lim} \norm{g^{k_{s_j}} = 0} \right) = 1 
\end{equation*}
\end{lemma}

\begin{proof}
    \label{proof of convergence lemma2-sketch}
The proof sketch is summarized as follows.
    \begin{enumerate}
        \item We construct the sequence $g^k\in \partial L_\beta^{k+1}$ and show that its subvector
        $q_i^k := \big(g^k_{\vect{y}^i}, g^k_{\vect{x}^i}, g^k_{\vect{z}^i}\big)$ satisfies
        $\underset{k\rightarrow\infty}{\lim} \mathbb{E}\norm{q_{i_k}^{k-\tau(\delta)-1}}^2 = 0$, where the mixing time $\tau(\delta)$ is defined in Eq. \eqref{mixing time ineq3}.

        \item For $k\geq0$, define the filtration of sigma algebras:

        $\Scale[0.9]{\chi^k =\sigma\left( \vect{y}^0, \ldots, \vect{y}^k, \matrx{X}^0, \ldots, \matrx{X}^k, \matrx{Z}^0, \ldots, \matrx{Z}^k, i_0, \ldots, i_k \right)}$.
        We show that

        $ \Scale[0.95]{\mathbb{E} \bigg( \norm{q_{i_k}^{k-\tau(\delta)-1}}^2 \bigg| \chi^{k-\tau(\delta)} \bigg) \geq (1-\delta)\pi_* \norm{g^{k-\tau(\delta)-1}}^2 }$,

        where $\pi_*$ is the minimal value in the Markov chain's stationary distribution. From this bound and the result in Step 1, we can get $\underset{k\rightarrow\infty}{\lim}\norm{g^k} = 0$.

        \item From the result in Step 2, we use the Borel-Cantelli lemma \cite{chung1952application} to obtain an almost unquestionably convergent subsubsequence of $g^k$.

    \end{enumerate}

\end{proof}


\vspace{-12pt}
The details of these steps are given as follows. 

\begin{proof}\label{proof of convergence lemma2-detailed}

First, recall Lemma \ref{lemma - sufficient descent in Lyapunov functions} and $T<\infty$, we have

\begin{equation}\label{proof of convergence lemma2 - eq1}
\sum^\infty_{k=0} \bigg( \mathbb{E}\norm{\vect{y}^k - \vect{y}^{k+1}}^2 +\mathbb{E}\norm{\matrx{X}^k - \matrx{X}^{k+1}}^2
+\mathbb{E}\norm{\matrx{X}^{\tau(k,i_k)} - \matrx{X}^{\tau(k,i_k)+1}}^2_2\bigg) < +\infty 
\end{equation}

Hence, by Lemma \ref{lemma - sufficient descent in Lyapunov functions}, one can infer
\begin{equation}\label{proof of convergence lemma2 - eq2}
\sum^\infty_{k=0} \bigg( \mathbb{E}\norm{\vect{y}^k - \vect{y}^{k+1}}^2 +\mathbb{E}\norm{\matrx{X}^k - \matrx{X}^{k+1}}^2
+\mathbb{E}\norm{\matrx{Z}^k - \matrx{Z}^{k+1}}^2_2\bigg) < +\infty 
\end{equation}

\textbf{Step 1:} The proof starts with computing the subgradients of the augmented Lagrangian Eq. \eqref{proof of lemma - descent of augmented lagrangian - eq1} with the updates in RWSADMM,

\begin{equation}
    \label{proof of convergence lemma2 - eq3}
\begin{aligned}
    \pdv{L^{k+1}_\beta}{\vect{y}} \ni -\frac{\beta}{n} (\vect{x}_{i_k}^{k+1} - \vect{x}_{i_k}^k) +\frac{1}{n} (\vect{z}_{i_k}^{k+1} - \vect{z}_{i_k}^k) =: w^k 
\end{aligned}
\end{equation}

\begin{equation}
    \label{proof of convergence lemma2 - eq4}
\begin{aligned}
    \grad_{\vect{x}_j} L^{k+1}_\beta = \frac{1}{n}\bigg( \grad f_j(\vect{x}_{j}^{k+1})
    -\vect{z}_j^{k+1} +\beta(\vect{x}_j^{k+1} - \vect{y}^{k+1})\bigg) 
\end{aligned}
\end{equation}

\begin{equation}
    \label{proof of convergence lemma2 - eq5}
\begin{aligned}
    \scalemath{0.8}{\grad_{\vect{z}_j} L^{k+1}_\beta = \frac{1}{n}(\vect{y}^{k+1} - \vect{x}_j^{k+1})}
\end{aligned}
\end{equation}

We define $g^k$ and $q_i^k$ as

\begin{equation}
    \label{proof of convergence lemma2 - eq6}
    g^k :=
    \begin{bmatrix}
        w^k\\
        \grad_{\matrx{X}}L^{k+1}_\beta \\
        \grad_{\matrx{Z}}L^{k+1}_\beta\\
    \end{bmatrix}
    , q^k_i :=
    \begin{bmatrix}
        w^k \\
        \grad_{\vect{x}_i}L^{k+1}_\beta \\
        \grad_{\vect{z}_i}L^{k+1}_\beta\\
    \end{bmatrix} 
\end{equation}

where $i\in\matrx{V}$ is the index of the agent and $g^k$ is the gradient of $L^\beta_k$. For $\delta\in(0,1)$ and $k\geq \tau(\delta) + 1$,

\begin{equation}
    \label{proof of convergence lemma2 - eq7}
\begin{aligned}
   \norm{q_{i_k}^{k-\tau(\delta)-1}}^2 = \norm{q_{i_k}^{k-\tau(\delta)-1} -q_{i_k}^k + q_{i_k}^k}^2 
   \overset{(a)}{\leq} 2\underbrace{\norm{q_{i_k}^{k-\tau(\delta)-1} -q_{i_k}^k}^2}_\text{A} + 2\underbrace{\norm{q_{i_k}^k}^2}_\text{B} 
\end{aligned}
\end{equation}

We upper bound A and B separately. A has three parts corresponding to the three parts of $g$. Its first part is

\begin{equation}
    \label{proof of convergence lemma2 - eq8}
\begin{aligned}
   &\norm{w^{k-\tau(\delta)-1} - w^k}^2 \\
   & \leq 2\norm{w^{k-\tau(\delta)-1}}^2 + 2\norm{w^k}^2 \\
   & \scalemath{0.8}{\overset{\eqref{proof of convergence lemma2 - eq3}}{\leq} \frac{4}{n^2} \bigg( \beta^2\norm{\matrx{X}^{k+1} - \matrx{X}^k}^2 +\beta^2\norm{\matrx{X}^{k-\tau(\delta)} - \matrx{X}^{k-\tau(\delta)-1}}^2} \\
   & \scalemath{0.8}{+\norm{\matrx{Z}^{k+1} - \matrx{Z}^k}^2 +\norm{\matrx{Z}^{k-\tau(\delta)} - \matrx{Z}^{k-\tau(\delta)-1}}^2 \bigg)}
\end{aligned}
\end{equation}

Then by Eq. \eqref{proof of convergence lemma2 - eq4}, we bound the 2nd part of $A$

\begin{equation}
    \label{proof of convergence lemma2 - eq9}
\begin{aligned}
   & \norm{\grad_{\vect{x}_{i_k}}L_\beta^{k-\tau(\delta)-1} - \grad_{\vect{x}_{i_k}}L_\beta^{k+1}}^2 \\
   & \overset{(a)}{\leq}\frac{4L^2+4\beta^2}{n^2}\norm{\vect{x}_{i_k}^{k-\tau(\delta)-1} - \vect{x}_{i_k}^{k+1}}^2+\frac{4}{n^2}\norm{\vect{z}_{i_k}^{k-\tau(\delta)-1} - \vect{z}_{i_k}^{k+1}}^2 \\
   & + \frac{4\beta^2}{n^2}\norm{\vect{y}^{k-\tau(\delta)-1} - \vect{y}^{k+1}}^2 \\
   & \leq\frac{D}{n^2}\sum^k_{t=k-\tau(\delta)-1}\big(
   \norm{\vect{y}^t - \vect{y}^{t+1}}^2 + \norm{\matrx{X}^t - \matrx{X}^{t+1}}^2
   +\norm{\matrx{Z}^t - \matrx{Z}^{t+1}}^2 \big)
\end{aligned}
\end{equation}

where $\scalemath{0.9}{D = (\tau(\delta) + 2)(4 + 4\beta^2 + 4L^2)}$, and (a) uses the inequality of arithmetic and geometric means and Lipschitz differentiability of $f_j$ in Assumption \ref{f is lsmooth and l-lipschitz}. From Eq. \eqref{proof of convergence lemma2 - eq5}. The third part of $A$ can be bounded as

\begin{equation}
    \label{proof of convergence lemma2 - eq10}
\begin{aligned}
   & \norm{\grad_{\vect{z}_{i_k}}L_\beta^{k-\tau(\delta)-1} - \grad_{\vect{z}_{i_k}}L_\beta^{k+1}}^2 \\
   & \leq  \frac{2}{n^2}\big( \norm{\vect{y}^{k-\tau(\delta)-1} - \vect{y}^{k+1}}^2 + \norm{\vect{x}_{i_k}^{k-\tau(\delta)-1} - \vect{x}_{i_k}^{k+1}}^2   \big) \\
   & \leq \frac{2(\tau(\delta)+2)}{n^2}\sum^k_{t=k-\tau(\delta)-1} \big( \norm{\vect{y}^t - \vect{y}^{t+1}}^2 +\norm{\matrx{X}^t - \matrx{X}^{t+1}}^2  \big)
\end{aligned}
\end{equation}

Plugging Eq. \eqref{proof of convergence lemma2 - eq8}, Eq. \eqref{proof of convergence lemma2 - eq9}, and Eq. \eqref{proof of convergence lemma2 - eq10} into term $A$, we get a constant $C_1 \sim O(\frac{\tau(\delta)+1}{n^2}) $, depending on $\tau(\delta)$, $\beta$, $L$, and $n$, such that

\begin{equation}
    \label{proof of convergence lemma2 - eq11}
\begin{aligned}
    A \leq C_1\sum^k_{t=k-\tau(\delta)-1} \big( \norm{\vect{y}^t - \matrx{y}^{t+1}}^2 +\norm{\matrx{x}^t - \matrx{x}^{t+1}}^2 +\norm{\matrx{z}^t - \matrx{z}^{t+1}}^2    \big)
\end{aligned}
\end{equation}

To bound the term $B$, using Eq. \eqref{proof of convergence lemma2 - eq5} and $\matrx{Z}$-update (Eq. \eqref{z update eq}), we have
\begin{equation}
    \label{proof of convergence lemma2 - eq12}
\begin{aligned}
    \grad_{\vect{z}_{i_k}} L^{k+1}_\beta = \frac{1}{n\beta}(\vect{z}^{k+1}_{i_k} - \vect{z}_{i_k}^k)
\end{aligned}
\end{equation}

Applying Eq. \eqref{proof of convergence lemma2 - eq4} and Eq. \eqref{update dual bound update primal lemma-eq3}, we drive $\grad_{\vect{x}_{i_k}}L_\beta^{k+1}$:
\begin{equation}
    \label{proof of convergence lemma2 - eq13}
\begin{aligned}
    \grad_{\vect{x}_{i_k}} L^{k+1}_\beta = \frac{1}{n\beta} \big( \grad f_{i_k}(\vect{x}_{i_k}^{k+1}) -\grad f_{i_k}(\vect{x}_{i_k}^k) +\vect{z}_{i_k}^k -\vect{z}^{k+1}_{i_k} \big)
\end{aligned}
\end{equation}
So we have
\begin{equation}
    \label{proof of convergence lemma2 - eq14}
\begin{aligned}
   B \leq C_2\big( \norm{\matrx{x}^t - \matrx{x}^{t+1}}^2 +\norm{\matrx{z}^t - \matrx{z}^{t+1}}^2    \big)
\end{aligned}
\end{equation}

for a constant $C_2$ depending on $L$, $\beta$, and $n$, in the order of $O(\frac{1}{n^2})$. Then substituting Eq. \eqref{proof of convergence lemma2 - eq11} and Eq. \eqref{proof of convergence lemma2 - eq14} into Eq. \eqref{proof of convergence lemma2 - eq7} and taking expectations, it yields

\begin{equation}
    \label{eq:bound_of_q}
\begin{aligned}
   &\mathbb{E}\norm{q_{i_k}^{k-\tau(\delta)-1}}^2 \\
   &\leq C\sum^k_{t=k-\tau(\delta)-1} \big( \mathbb{E}\norm{\vect{y}^t - \matrx{y}^{t+1}}^2 +\mathbb{E}\norm{\matrx{x}^t - \matrx{x}^{t+1}}^2 +\mathbb{E}\norm{\matrx{z}^t - \matrx{z}^{t+1}}^2 \big) 
\end{aligned}
\end{equation}

where $C = C_1 + C_2$, and $C \sim O(\frac{\tau(\delta)+1}{n^2})$. Recalling Eq. \eqref{proof of convergence lemma2 - eq2}, we get the convergence

\begin{equation}
\begin{aligned}
   \underset{k\rightarrow \infty}{\lim} \mathbb{E}\norm{q_{i_k}^{k-\tau(\delta)-1}}^2 = 0
\end{aligned}
\end{equation}

which completes the proof of \textbf{Step 1}.

\textbf{Step 2:} We compute the expectation:
\begin{equation}
    \label{eq:bound_of_q_no_E}
\begin{aligned}
    &\mathbb{E} \big( \norm{q_{i_k}^{k-\tau(\delta)-1}}^2 \bigg| \chi^{k-\tau(\delta)} \big)\\
    &= \sum^n_{j=1}[\matrx{P}(k)^{\tau(\delta)}]_{i_{k-\tau(\delta)},j}\big( \norm{\grad_{\vect{y}}L^{k-\tau(\delta)}_\beta}^2   +\norm{\grad_{\vect{x_j}}L^{k-\tau(\delta)}_\beta}^2\\
    &+\norm{\grad_{\vect{z_j}}L^{k-\tau(\delta)}_\beta}^2  \big) \overset{(a)}{\geq} (1-\delta)\pi_* \norm{g^{k-\tau(\delta)-1}}^2
\end{aligned}
\end{equation}

where (a) follows from Eq. \eqref{mixing time ineq2} and the definition of $g_k$ in Eq. \eqref{proof of convergence lemma2 - eq6}. Then, with Eq. \eqref{eq:bound_of_q}, one can derive

\begin{equation}
    \label{eq:g_limit0}
\begin{aligned}
   \underset{k\rightarrow \infty}{\lim} \mathbb{E}\norm{g^k}^2 = \underset{k\rightarrow \infty}{\lim} \mathbb{E}\norm{g^{k-\tau(\delta)-1}}^2 = 0,
\end{aligned}
\end{equation}

By the Schwarz inequality $ (\mathbb{E}\norm{g^k})^2 \leq \mathbb{E}\norm{g^k}^2$, we have

\begin{equation}
    \label{eq:g_limit0_v2}
\begin{aligned}
   \underset{k\rightarrow \infty}{\lim} \mathbb{E}\norm{g^k} = 0,
\end{aligned}
\end{equation}

\textbf{Step 3:} By Markov's inequality, for each $\omega>0$, it holds that

\begin{equation}
\begin{aligned}
   &Pr(\norm{g_k}>\omega) \leq \frac{\mathbb{E}\norm{g^k}}{\omega} \overset{Eq. \eqref{eq:g_limit0}}{\Longrightarrow} \underset{k\rightarrow \infty}{\lim} Pr(\norm{g_k}>\omega) = 0
\end{aligned}
\end{equation}

when a subsequence $(k_s)_{s\geq0}$ is provided, Eq. \eqref{eq:g_limit0_v2} implies,

\begin{equation}
\begin{aligned}
   Pr(\norm{g_{k_s}}>\omega) = 0
\end{aligned}
\end{equation}

Then, for $j\in\mathbb{N}$, select $\omega=2^{-j}$ and we can find a nondecreasing subsubsequence $(k_{s_j})$, such that

\begin{equation}
\begin{aligned}
   Pr(\norm{g_{k_s}}>2^{-j}) \leq 2^{-j}, \;\; \forall k_s\geq k_{s_j}
\end{aligned}
\end{equation}
We have
\begin{equation}
\begin{aligned}
   \sum^\infty_{j=1}Pr(\norm{g_{k_s}}>2^{-j}) \leq \sum^\infty_{j=1}2^{-j} = 1
\end{aligned}
\end{equation}

The Borel-Cantelli lemma yields

\begin{equation}
\begin{aligned}
   Pr\bigg(\underset{j}{\lim\sup}\{\norm{g_{k_s}}>2^{-j}\} \bigg) = 0
\end{aligned}
\end{equation}

and thus

\begin{equation}
\begin{aligned}
   Pr\bigg(\underset{j}{\lim}\norm{g_{k_{s_j}}}=0 \bigg) = 1
\end{aligned}
\end{equation}

This completes \textbf{Step 3} and thus the entire \textbf{Lemma \ref{convergence lemma2}}.

\end{proof}

\paragraph{Proof of Convergence}
Finally, we present the proof of the convergence theorem (Theorem \ref{convergence theorem}) as follows. 

\begin{proof}
    \label{proof of convergence theorem}
By statement 2 of Lemma \ref{convergence lemma}, the sequence $(\vect{y}^k, \matrx{X}^k, \matrx{Z}^k)$ is bounded, so there exists a convergent subsequence $(\vect{y}^{k_s}, \matrx{X}^{k_s}, \matrx{Z}^{k_s})$ converging to a limit point $(\vect{y}^*, \matrx{X}^*, \matrx{Z}^*)$ as $s\rightarrow\infty$. By continuity, we have
\begin{equation}
    \label{proof of convergence theorem eq1}
\begin{aligned}
\scalemath{0.9}{ L_\beta(\vect{y}^*, \matrx{X}^*, \matrx{Z}^*) = \underset{s\rightarrow\infty}{\lim} L_\beta(\vect{y}^{k_s}, \matrx{X}^{k_s}, \matrx{Z}^{k_s})}
\end{aligned}
\end{equation}
Lemma \ref{convergence lemma2} finds a subsubsequence $g^{k_{s_j}}\in\partial L_\beta^{k+1}$ such that $Pr\big(\lim_{j\rightarrow\infty}\norm{g^{k_{s_j}}}=0\big) = 1$. By the definition of general subgradient (\cite{rockafellar2009variational}, def 8.3), we have $0\in\partial L_\beta(\vect{y}^*, \matrx{X}^*, \matrx{Z}^*)$. Hence, Theorem \ref{convergence theorem} is proved. 
\end{proof}


\section{RWSADMM: Convergence Rate Analysis}
\label{appendix: convergence rate}

The detailed proof of the convergence rate theorem is as follows.

\begin{proof}
    \label{proof of convergence rate theorem}
    It can be verified that under specific initialization, Eq. \eqref{update dual bound update primal lemma-eq3} holds for all $k\geq 0$. Consequently, Lemmas \ref{update dual bound update primal}-\ref{lemma - sufficient descent in Lyapunov functions} hold for all $k\geq 0$. For $g^k$ defined in Eq. \eqref{proof of convergence lemma2 - eq6}, Eq. \eqref{eq:bound_of_q} and Eq. \eqref{eq:bound_of_q_no_E} hold. Jointly applying Eq. \eqref{eq:bound_of_q} and Eq. \eqref{eq:bound_of_q_no_E}, for any $k > \tau(\delta)+1$, one has 
    \begin{equation} \label{convergence rate proof eq1}
        \begin{aligned}
            &\mathbb{E}\norm{g^{k-\tau(\delta)-1}}^2 \leq \frac{C}{(1-\delta)\pi_{*}}  \\
            &\sum_{t=k-\tau(\delta)-1}^k \big( \mathbb{E}\norm{\vect y^t-\vect y^{t+1}}^2
            + \mathbb{E}\norm{\vect X^t-\vect X^{t+1}}^2 + \mathbb{E}\norm{\vect Z^t-\vect Z^{t+1}}^2 \big)
        \end{aligned}
    \end{equation}  

According to Lemmas \ref{update dual bound update primal} and \ref{lemma - sufficient descent in Lyapunov functions}, for $k \geq \tau(\delta) + 1$, it holds, 

\begin{equation} \label{convergence rate proof eq2}
    \begin{aligned}
        &\sum_{t=k-\tau(\delta)-1}^k \big( \mathbb{E}\norm{\vect y^t-\vect y^{t+1}}^2 + \mathbb{E}\norm{\vect X^t-\vect X^{t+1}}^2 + \mathbb{E}\norm{\vect Z^t-\vect Z^{t+1}}^2 \big)  \\
        &\leq \text{max} \big\{ \frac{2}{\beta-\gamma}, (1 + L^2)n \big\}
        \big(  \mathbb{E}L_\beta^{k-\tau(\delta)-1} - \mathbb{E}L_\beta^{k+1} \big)
    \end{aligned}
\end{equation}  

It implies that for any $k\geq0$, it holds

\begin{equation} \label{convergence rate proof eq3}
    \begin{aligned}
         \mathbb{E}\norm{g^k}^2 \leq C' \big( \mathbb{E}L^k_\beta - \mathbb{E}L_\beta^{k+\tau(\delta)+2} \big)
    \end{aligned}
\end{equation}  

where $C' := \text{max} \{ \frac{2}{\beta-\gamma}, (1 + L^2)n \}\frac{C}{(1-\delta)\pi_*}$. It can be verified that $C' = O ( \frac{\tau(\delta)+1}{(1-\delta)n\pi_*} )$. Let $\tau' := \tau(\delta) + 2$; for any $K > \tau'$, summing Eq. \eqref{convergence rate proof eq3} over $k \in \{ K-\tau', \ldots, \text{mod}_{\tau'}K \}$ gives

\begin{equation} \label{convergence rate proof eq4}
    \begin{aligned}
        \sum_{l=1}^{\lfloor\frac{K}{\tau'}\rfloor} \mathbb{E}\norm{g^{K-l\tau'}}^2 \leq C' \bigg(
        \mathbb{E}L_\beta^{\text{mod}_{\tau'}K} - 
        \mathbb{E}L_\beta^{K} \bigg)
        \leq C' ( L_\beta^0 - \underset{-}{f} ) 
    \end{aligned}
\end{equation}  

where the last inequality follows from the non-decreasing property of the sequence $(L_\beta^k)_{k\geq0}$ and the fact that $(L_\beta^k)_{k\geq0}$ is lower bounded by $\underset{-}{f}$. According to Eq. \eqref{convergence rate proof eq4},

\begin{equation} \label{convergence rate proof eq5}
    \begin{aligned}
        &\underset{k\leq K}{min}\mathbb{E}\norm{g^k}^2 \leq
        \underset{1\leq l\leq \lfloor\frac{K}{\tau'}\rfloor}{min} \mathbb{E}\norm{g^{K-l\tau'}}^2  \\
        &\leq \frac{1}{\lfloor \frac{K}{\tau'} \rfloor }\sum_{l=1}^{\lfloor\frac{K}{\tau'}\rfloor} \mathbb{E}\norm{g^{K-l\tau'}}^2 \leq \frac{\tau'C'}{K - \tau'} ( L_\beta^0 - \underset{-}{f} )  \\
        &\leq \frac{C'(\tau'+1)}{K} ( L_\beta^0 - \underset{-}{f} )
    \end{aligned}
\end{equation}  

where the constant $C'(\tau'+1) = O\big( \frac{\tau(\delta)^2 + 1}{(1-\delta)n\pi_*} \big)$.
\end{proof}

We consider a reversible Markov chain on an undirected graph. Using the definition of $\tau(\delta)$ in Eq. \eqref{mixing time ineq3} and setting $\delta = 1/2$, one has $\tau(\delta)\sim\frac{\text{ln}n}{1-\lambda_2(P(k))}$. To guarantee $\underset{k\leq K}{min} \mathbb{E}\norm{g^k}^2 \leq \omega$, certain number of iterations is sufficient,

\begin{equation} \label{convergence rate proof eq6}
    \begin{aligned}
      O\bigg(\frac{1}{\omega}.\bigg(\frac{\text{ln}^2 n}{(1-\lambda_2(P(k)))^2}\bigg)+1 \bigg)
    \end{aligned}
\end{equation} 

%


\section{Experiments}\label{appendix: experiments-rest}

\subsection{Models}\label{appendix: models}

One strongly convex model and two non-convex models are utilized in our implementations. An MLR model with a logistic regression classifier is implemented for the strongly convex setting.
For the first non-convex setting, an MLP model with two hidden dense layers of vector image size (resizing the 2D image as 1D vector) and 100 hidden nodes in the hidden layer are implemented. The cross-entropy loss is employed for this network. 
For the second non-convex setting, a CNN model with two convolutional layers with convolution operations of size $5\times 5$ and one fully-connected layer of 512 followed by a Softmax layer is implemented. The cross-entropy loss is utilized in the network, and dropout rates of 25\% and 50\% are applied after convolutional layers.

\subsection{Graph Construction}\label{appendix:graph_construction}
To meet this requirement on Assumption \ref{irreducible aperiodic markov chain}, we propose using a Markov transition matrix $\matrx P$ with a maximum eigenvalue of less than $1 - 1/m^{2/3}$, where $m$ is the number of edges. To fulfill this inequality, we can increase the number of edges in the network. In our experiments, we have addressed this issue by requiring each client to be a neighbor of at least $M$ other clients, which ensures a sufficient number of edges to satisfy the assumption on $\matrx P$. By incorporating this requirement into our implementation, we can guarantee the validity of our results and ensure that our algorithm performs optimally under realistic conditions.

\subsection{Hyperparameter tuning }\label{appendix: paramTuning}

The two hyperparameters of RWSADMM ($\beta$ and $\kappa$) must be fine-tuned to optimize the performance.  
In the first stage of the experiments, we set $\kappa=0.001$ and search for the optimal values of $\beta$. The fine-tuning process is performed for each dataset and each model separately. In Fig. \ref{fig:betatune}, the effect of $\beta$ values on the performance of RWSADMM for the MNIST dataset is presented. 

\begin{figure}[ht!]
    \centering
    \begin{subfigure}{0.3\textwidth}
        \includegraphics[trim=5pt 5pt 5pt 5pt, clip,width=\textwidth]{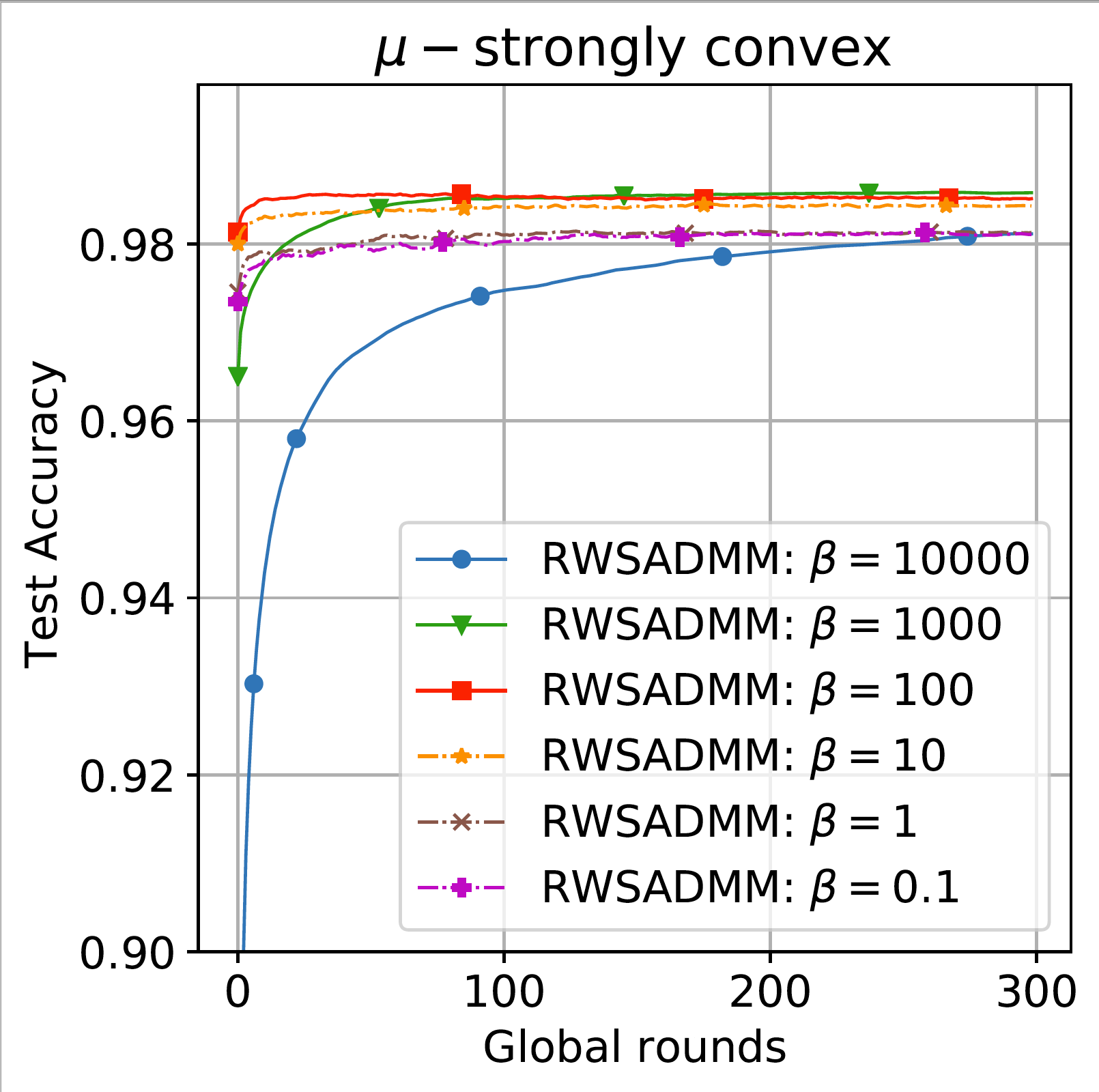}
        \caption{MLR acc}
        \label{fig:mlr-acc-betatune-mnist}
    \end{subfigure}
    \begin{subfigure}{0.3\textwidth}
        \includegraphics[trim=5pt 5pt 5pt 5pt, clip,width=\textwidth]{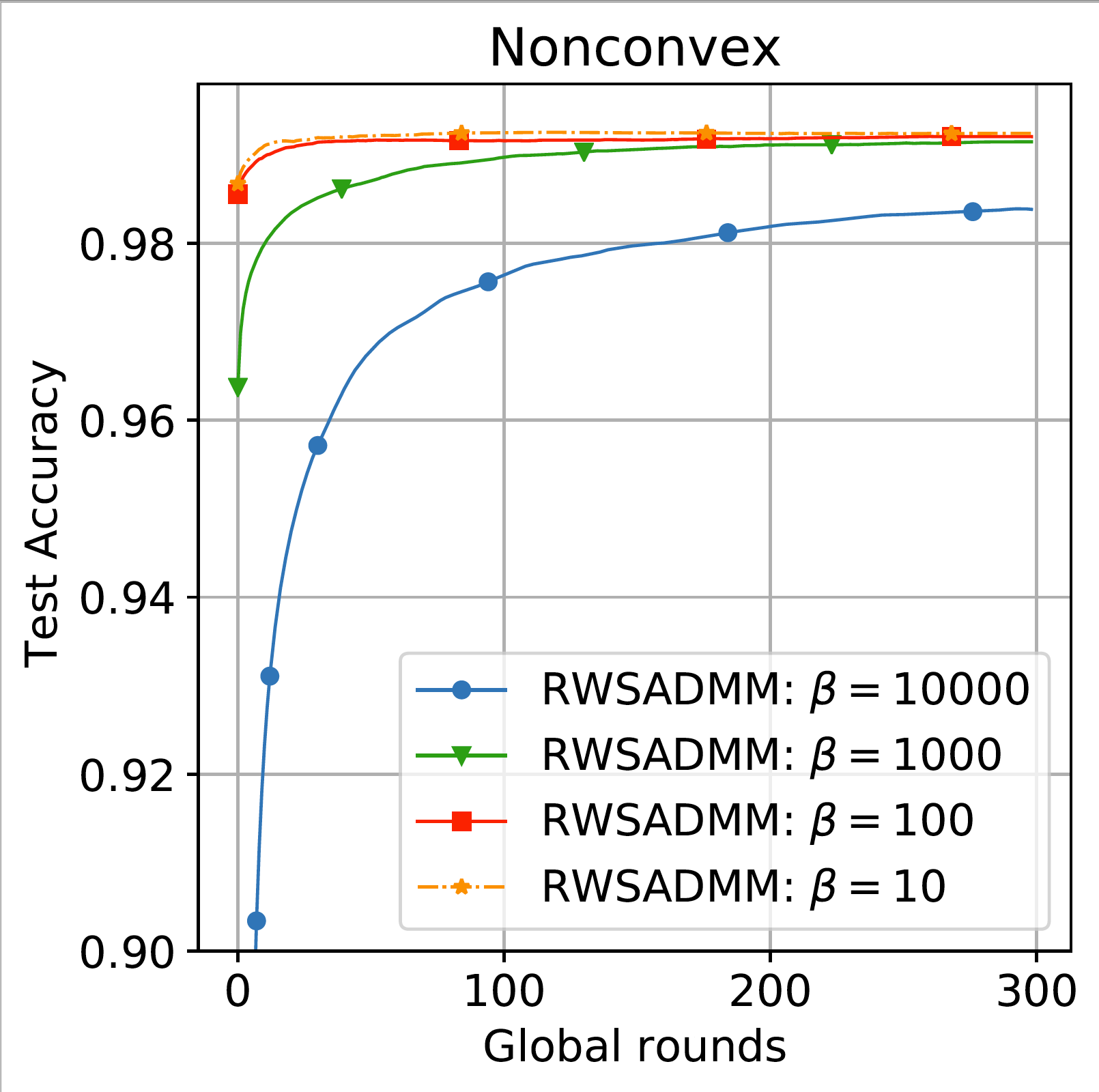}
        \caption{MLP acc}
        \label{fig:mlp-acc-betatune-mnist}
    \end{subfigure}
    \vskip\baselineskip
    \begin{subfigure}{0.3\textwidth}
        \includegraphics[trim=5pt 5pt 5pt 5pt, clip, width=\textwidth]{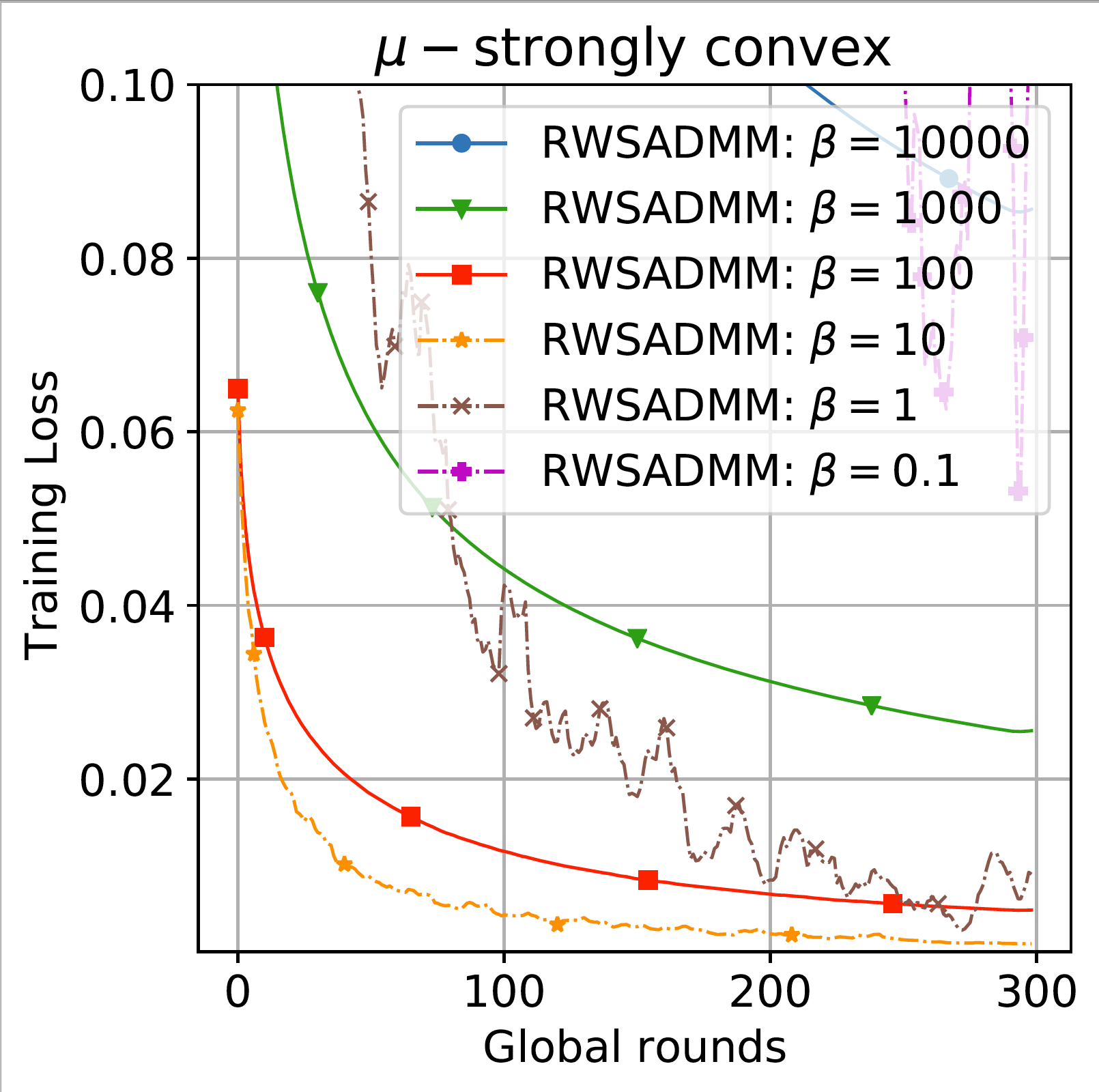}
        \caption{MLR loss}
        \label{fig:mlr-loss-betatune-mnist}
    \end{subfigure}
    \begin{subfigure}{0.3\textwidth}
        \includegraphics[trim=5pt 5pt 5pt 5pt, clip, width=\textwidth]{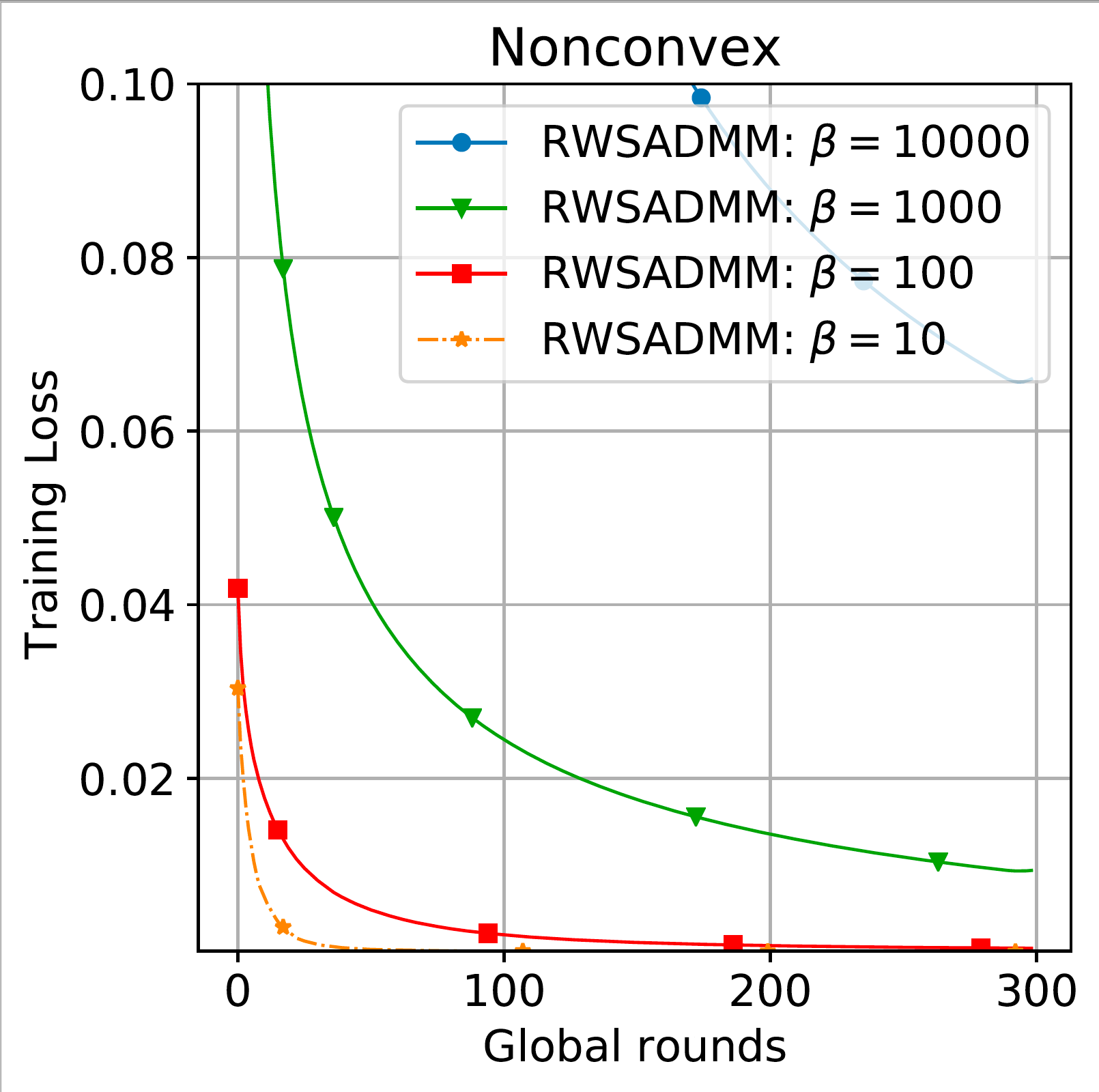}
        \caption{MLP loss}
        \label{fig:mlp-loss-betatune-mnist}
    \end{subfigure}
  \caption{Effect of $\beta$ on the convergence of RWSADMM in the MLR (\ref{fig:mlr-acc-betatune-mnist}, \ref{fig:mlr-loss-betatune-mnist}) and MLP (\ref{fig:mlp-acc-betatune-mnist}, \ref{fig:mlp-loss-betatune-mnist}) models for MNIST dataset. }
    \label{fig:betatune}
\end{figure}

The $\kappa$ parameter, which affects the dual variable $\matrx{Z}$, is also fine-tuned in the second stage. Same as the first stage, $\kappa$ is fine-tuned for the MNIST dataset for the MLR, MLP, and CNN models, and the results are shown in Fig. \ref{fig:kappatune}. 
The optimal values of $\{\kappa = 0.001, 0.01\}$ have the best performances for MLR and MLP, respectively.

\begin{figure}[ht!]
    \centering
    \begin{subfigure}{0.3\textwidth}
        \includegraphics[trim=5pt 5pt 5pt 5pt, clip, width=\textwidth]{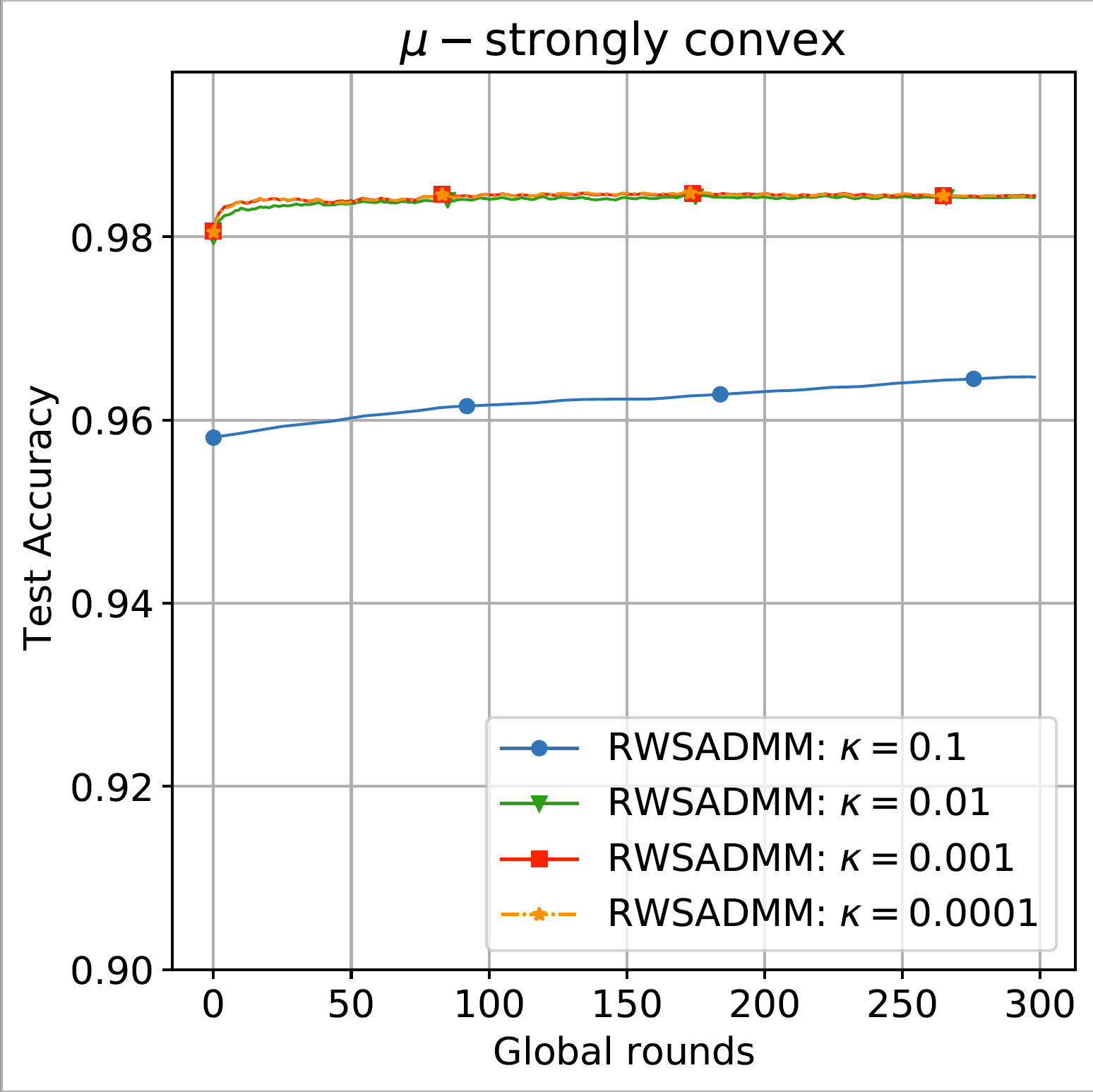}
        \caption{MLR acc}
        \label{fig:mlr-acc-kappatune-mnist}
    \end{subfigure}
    \begin{subfigure}{0.3\textwidth}
        \includegraphics[trim=5pt 5pt 5pt 5pt, clip, width=\textwidth]{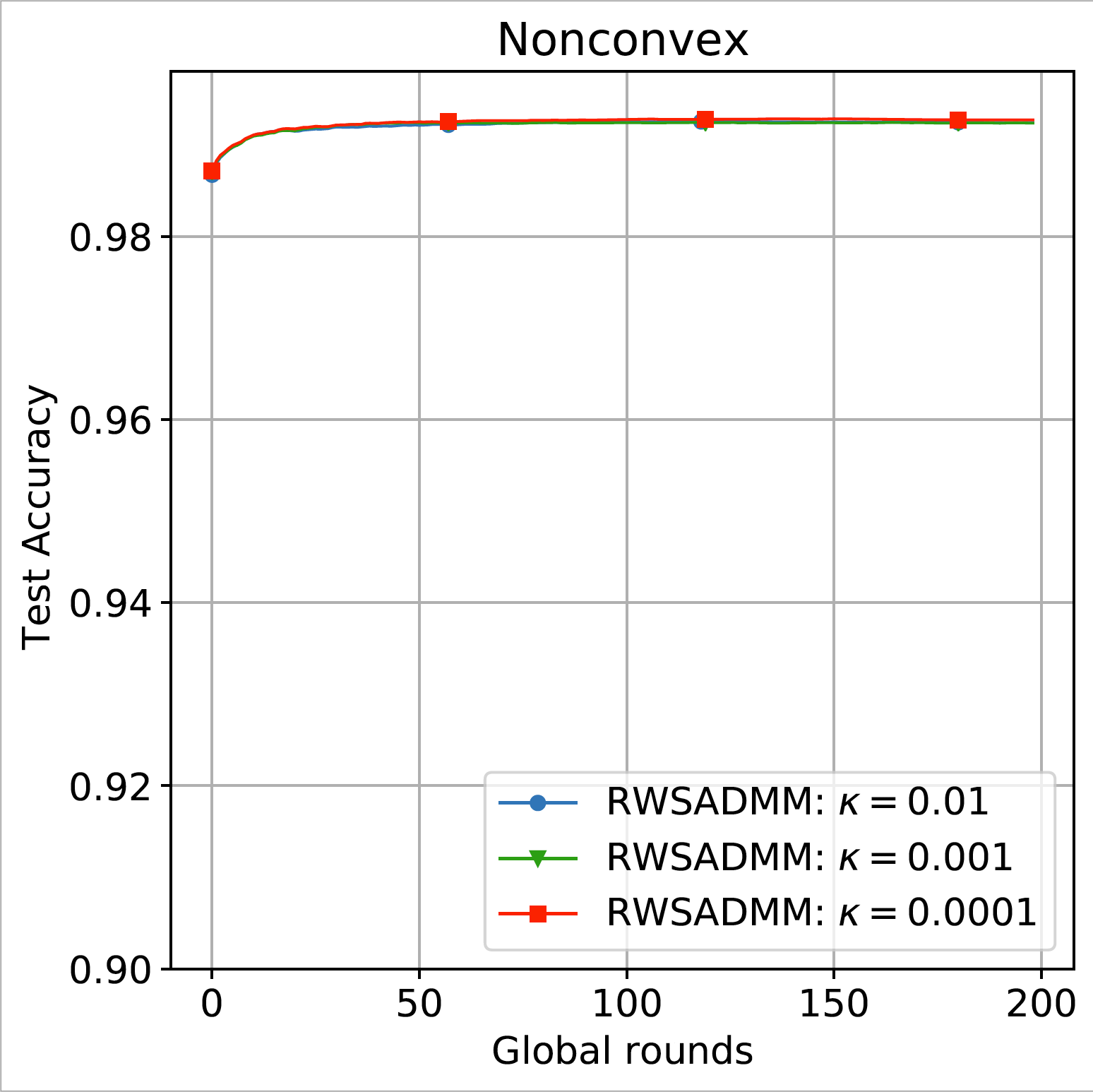}
        \caption{MLP acc}
        \label{fig:mlp-acc-kappatune-mnist}
    \end{subfigure}
    \vskip\baselineskip
    \begin{subfigure}{0.32\textwidth}
        \includegraphics[trim=5pt 5pt 25pt 5pt, clip, width=\textwidth]{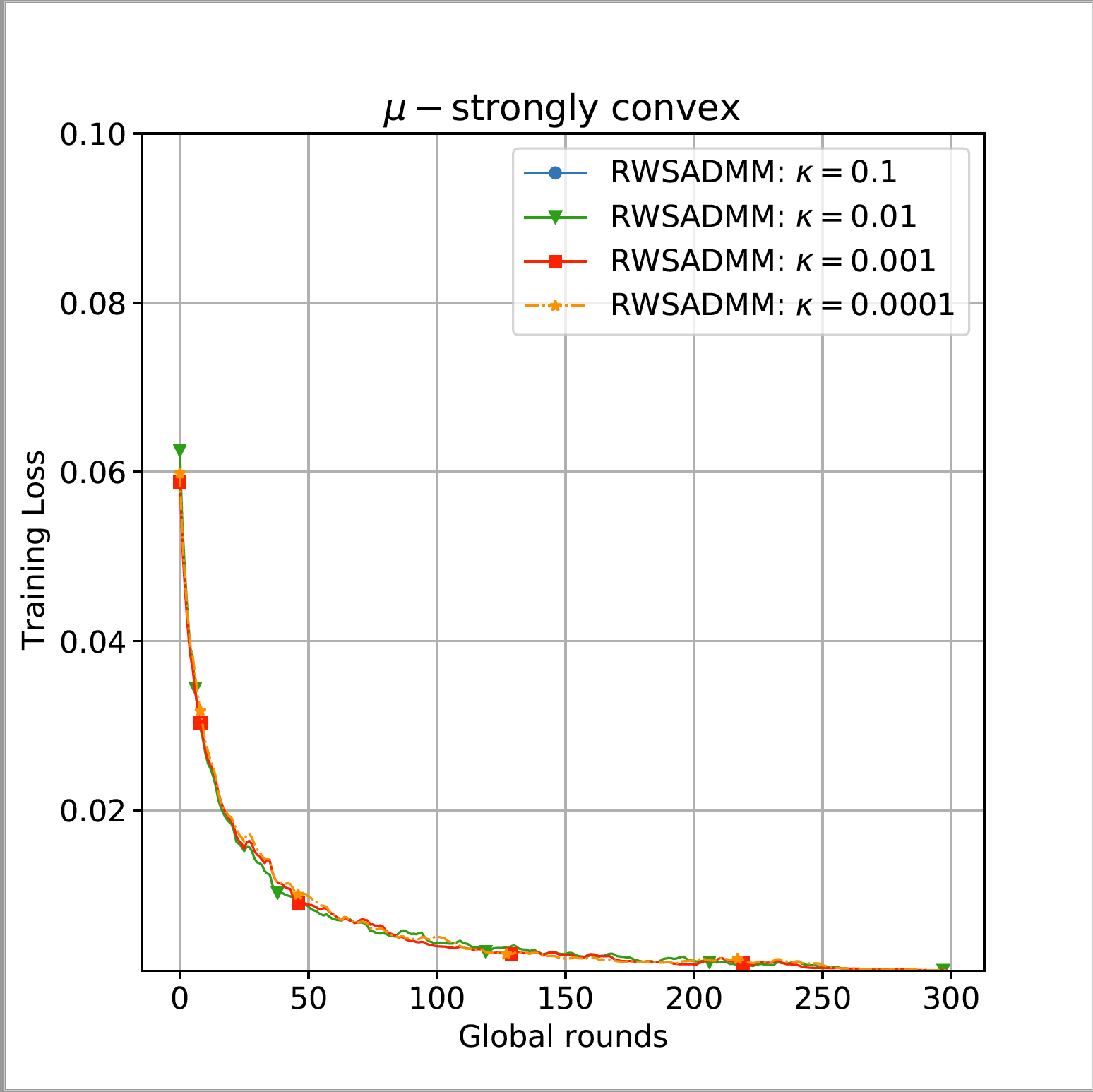}
        \caption{MLR loss}
        \label{fig:mlr-loss-kappatune-mnist}
    \end{subfigure}
    \begin{subfigure}{0.3\textwidth}
        \includegraphics[trim=5pt 5pt 5pt 5pt, clip, width=\textwidth]{images/kappatune_mnist_acc_mlp.png}
        \caption{MLP loss}
        \label{fig:mlp-loss-kappatune-mnist}
    \end{subfigure}
  \caption{Effect of $\kappa$ on the convergence of RWSADMM in the MLR (\ref{fig:mlr-acc-kappatune-mnist}, \ref{fig:mlr-loss-kappatune-mnist}) and MLP (\ref{fig:mlp-acc-kappatune-mnist}, \ref{fig:mlp-loss-kappatune-mnist}) models for MNIST dataset. }
    \label{fig:kappatune}
\end{figure}

The proximal parameter, $\epsilon$, is set to a fixed value of $\{\epsilon = 1e-5\}$ for all the experiments. For the MNIST dataset, the fine-tuned parameter values of $\beta = 10$ and $\kappa = 0.001$ for MLR, $\beta = 10$ and $\kappa = 0.01$ for MLP are used. 
For the Synthetic dataset, $\beta = 10$ and $\kappa = 0.01$ for MLR, $\beta = 100$ and $\kappa = 0.01$ for MLP models are utilized. 
Finally, for CIFAR10 dataset, $\beta = 100$ and $\kappa = 0.001$ for MLR, $\beta = 100$ and $\kappa = 1$ for MLP are used.


\subsection{CIFAR10 figures}\label{appendix: cifar10}

The performance comparison of RWSADMM, FedAvg, Per-FedAvg, pFedMe, APFL, and Ditto for the Cifar10 dataset are depicted in Fig. \ref{fig:cifar10-performance}. The fine-tuned values of $\beta=0.001$ and $\kappa=0.001$ are used for RWSADMM. The RWSADMM has a steep curve nearly reaching the optimal values from the first few rounds in strongly convex and non-convex models, indicating a faster convergence process than the benchmark algorithms. Also, RWSADMM shows a clear advantage for MLR or DNN models for accuracy.

\begin{figure}[ht!]
    \centering
    \begin{subfigure}{0.3\textwidth}
        \includegraphics[trim=5pt 5pt 5pt 5pt, clip, width=\textwidth]{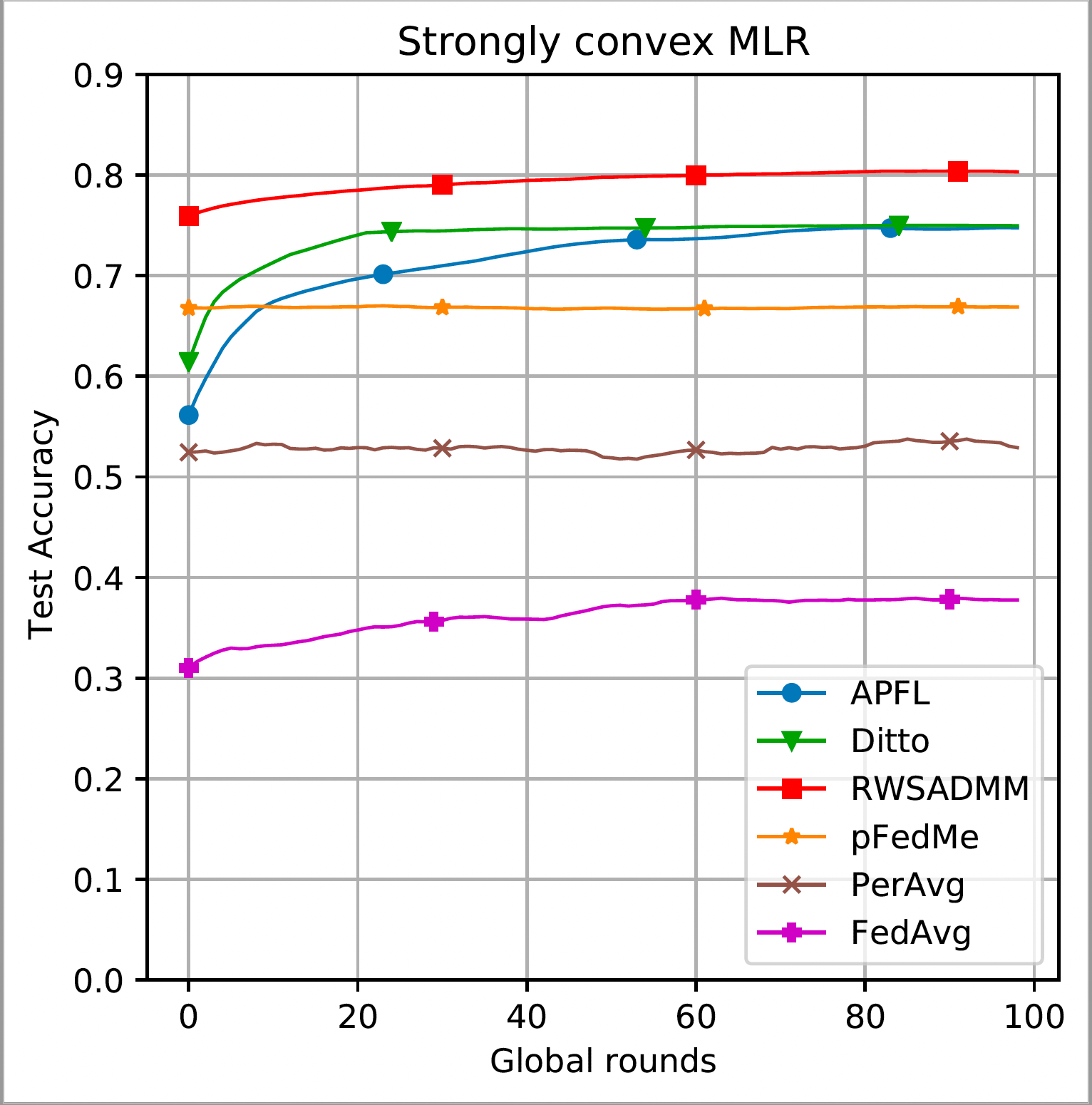}
        \caption{MLR acc}
        \label{fig:mlr-acc-cifar10-all}
    \end{subfigure}
    \begin{subfigure}{0.3\textwidth}
        \includegraphics[trim=5pt 5pt 5pt 5pt, clip, width=\textwidth]{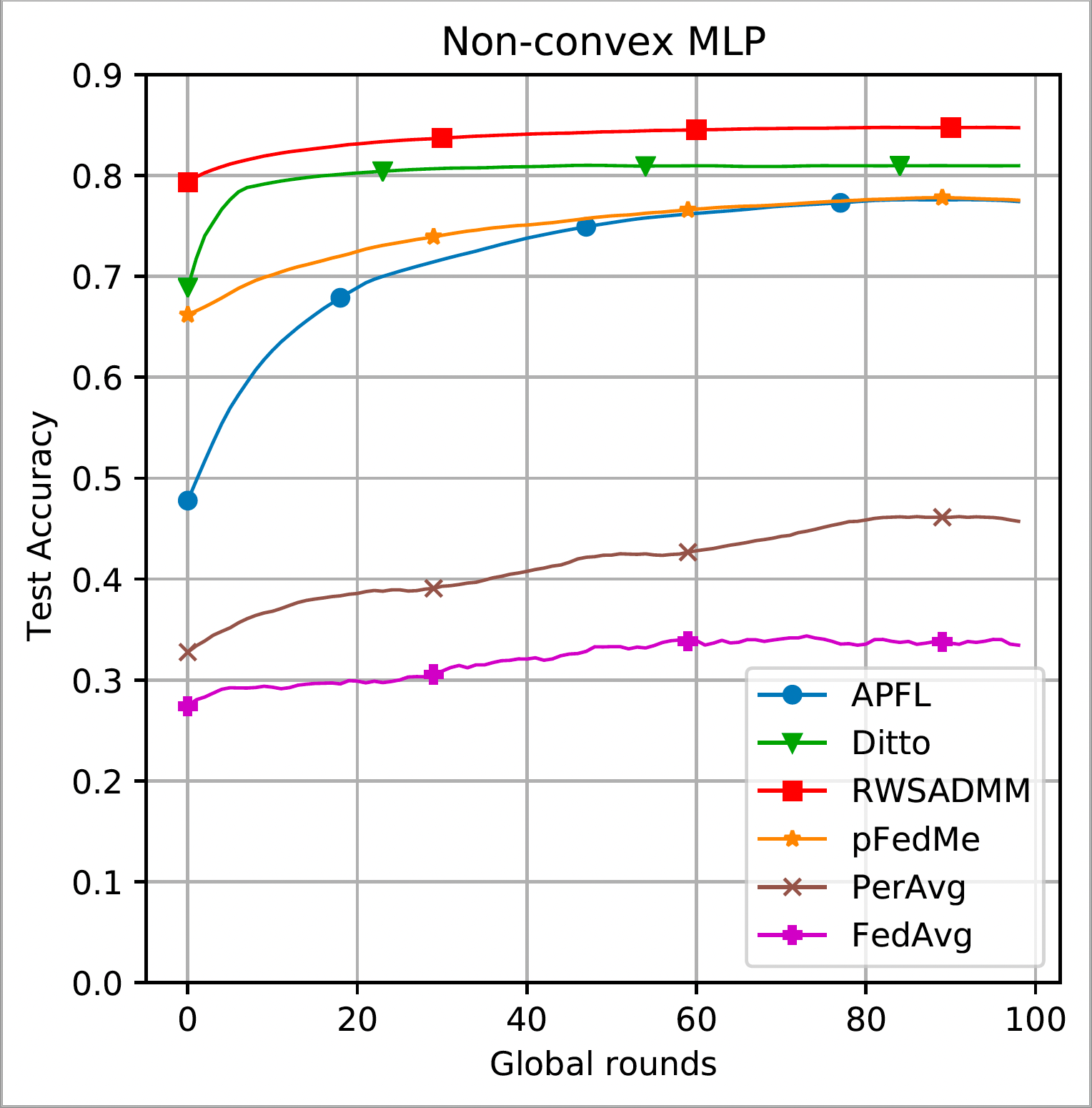}
        \caption{MLP acc}
        \label{fig:mlp-acc-cifar10-all}
    \end{subfigure}
    \begin{subfigure}{0.3\textwidth}
        \includegraphics[trim=5pt 5pt 5pt 5pt, clip, width=\textwidth]{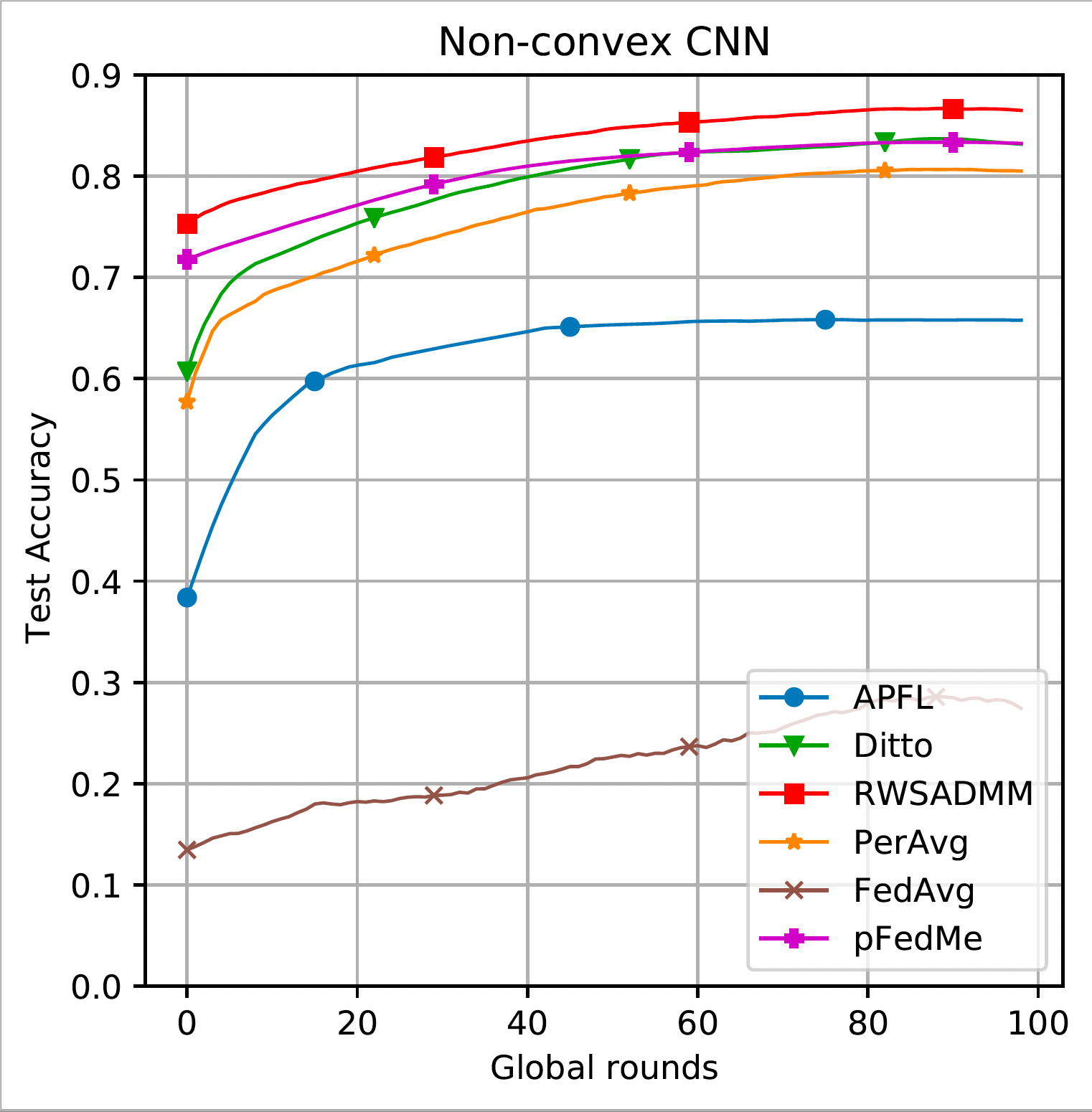}
        \caption{CNN acc}
        \label{fig:cnn-acc-cifar10-all}
    \end{subfigure}
    \vskip\baselineskip
    
    \begin{subfigure}{0.3\textwidth}
        \includegraphics[trim=5pt 5pt 5pt 5pt, clip, width=\textwidth]{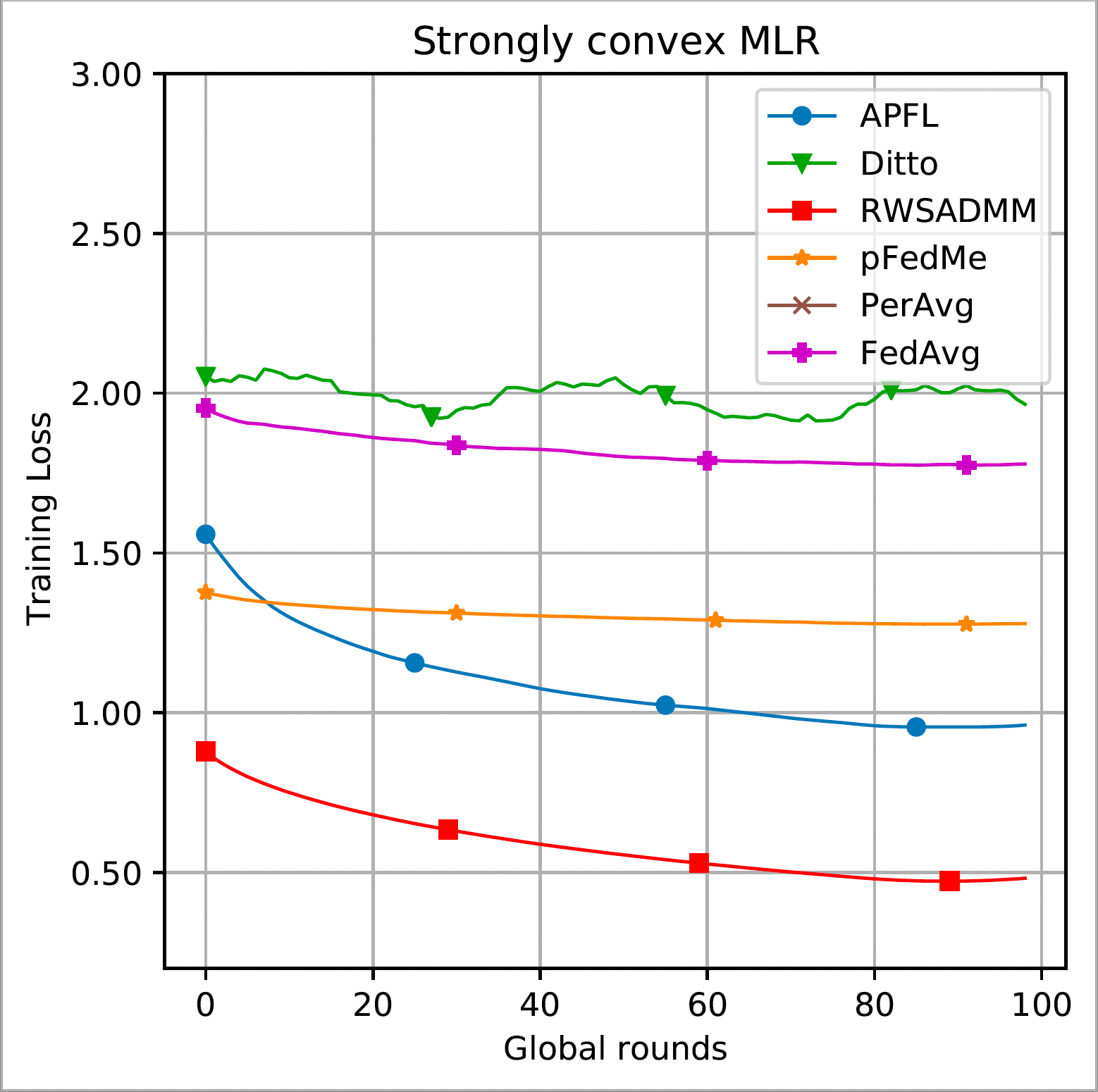}
        \caption{MLR loss}
        \label{fig:mlr-loss-cifar10-all}
    \end{subfigure}
    \begin{subfigure}{0.3\textwidth}
        \includegraphics[trim=5pt 5pt 5pt 5pt, clip, width=\textwidth]{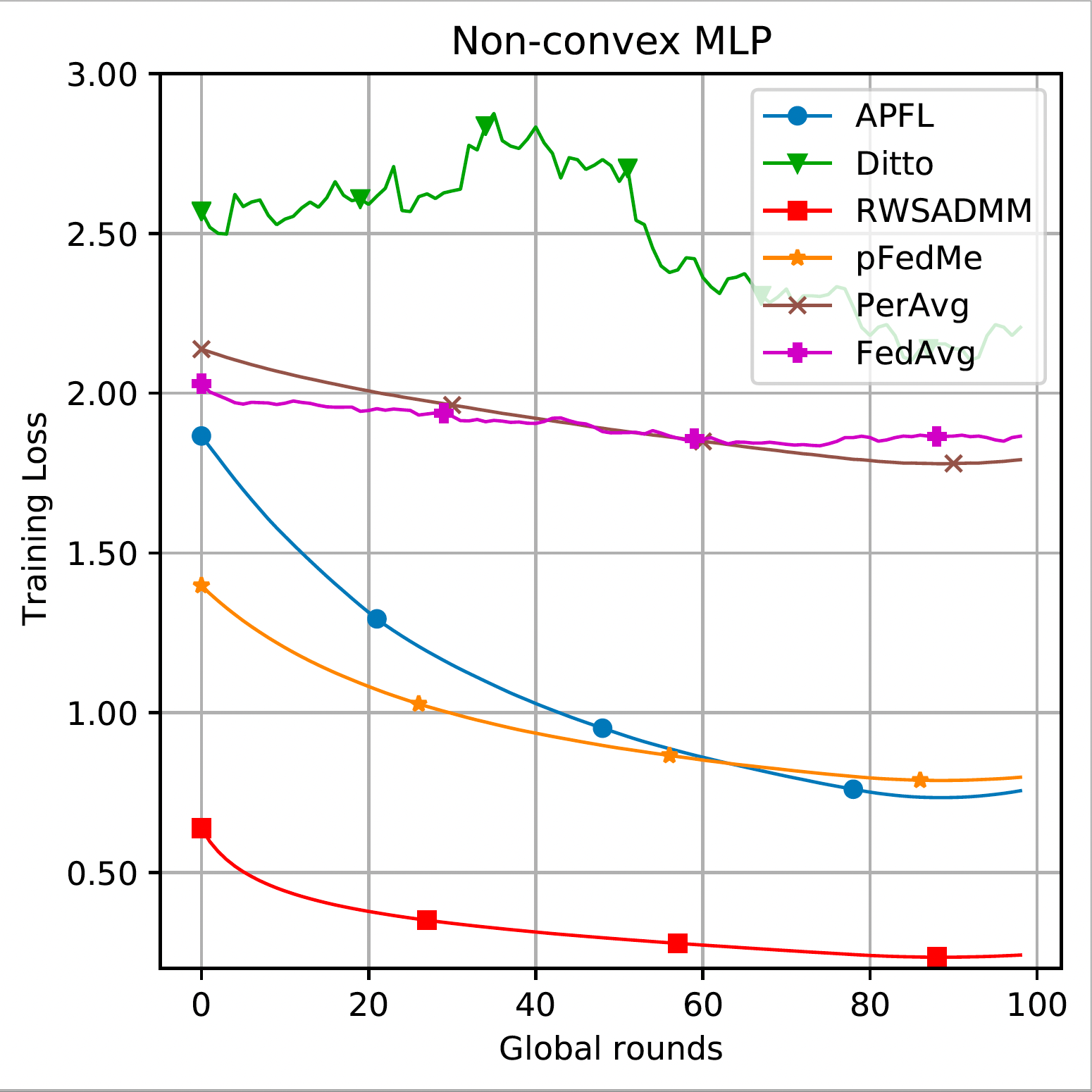}
        \caption{MLP loss}
        \label{fig:mlp-loss-cifar10-all}
    \end{subfigure}
    \begin{subfigure}{0.3\textwidth}
        \includegraphics[trim=5pt 5pt 5pt 5pt, clip, width=\textwidth]{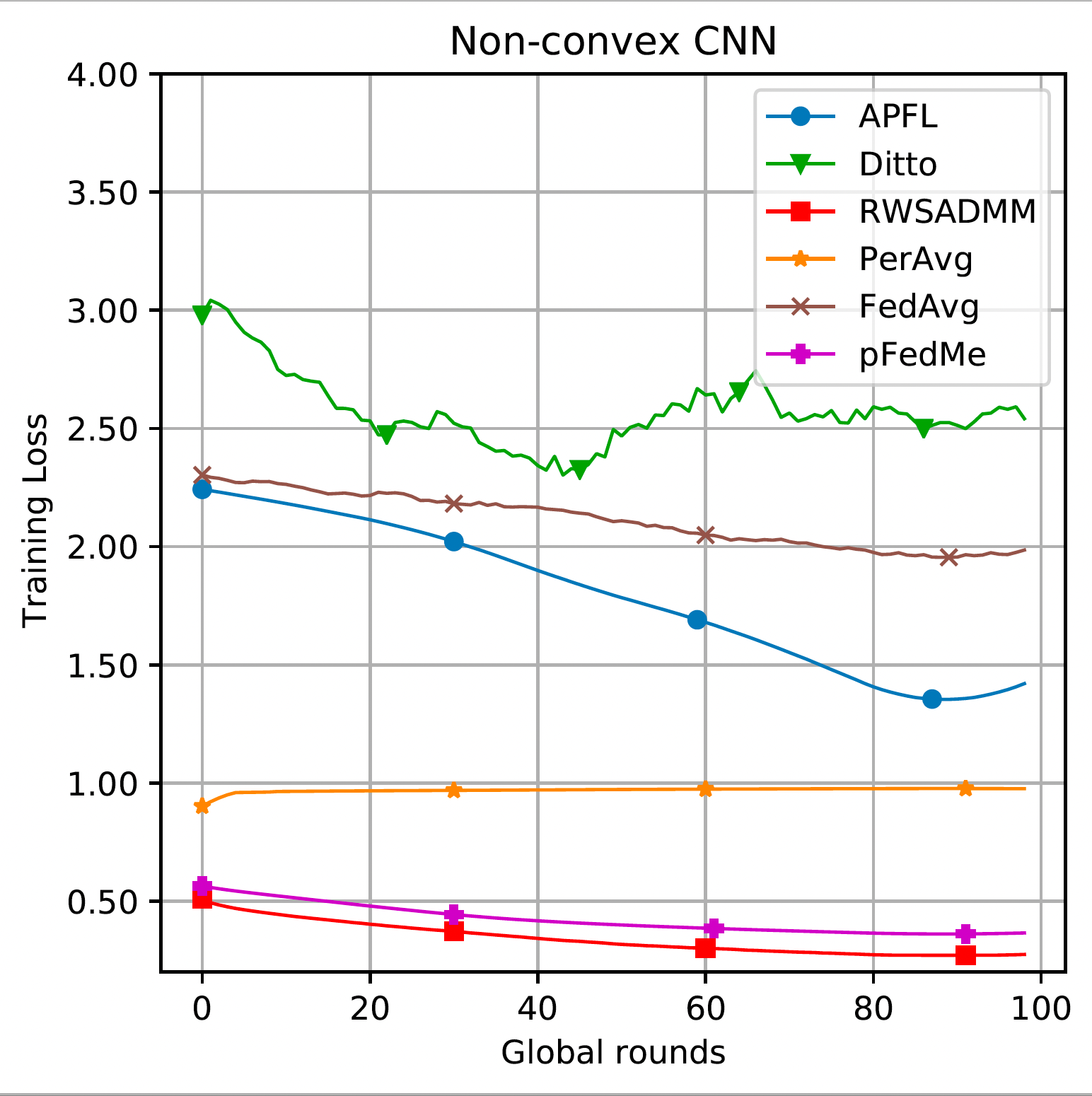}
        \caption{CNN loss}
        \label{fig:cnn-loss-cifar10-all}
    \end{subfigure}
    
    \caption{Performance comparison (test accuracy and training loss) of RWSADMM, pFedMe, Per-Avg, FedAvg, APFL, and Ditto for CIFAR10 dataset for strongly convex MLR, non-convex MLP, and non-convex CNN models. }
    \label{fig:cifar10-performance}
\end{figure}

\subsection{Synthetic figures}\label{appendix: synthetic}

Due to the 1D nature of the synthetic dataset, only MLR and MLP models are utilized for it. The performance comparison of RWSADMM, FedAvg, Per-FedAvg, pFedMe, APFL, and Ditto for the Synthetic dataset are depicted in Fig. \ref{fig:synthetic-performance}. The fine-tuned values of $\beta=100$ and $\kappa=0.001$ are used for RWSADMM for all the settings. By comparing the accuracy and loss diagrams, RWSADMM performs visibly better than the rest of the algorithms in both strongly convex and non-convex settings. The accuracy rate of RWSADMM shows accuracy improvement compared with the benchmark algorithms by the margin of $14\%$ for both MLR and MLP models.

\begin{figure}[ht!]
    \centering
    \begin{subfigure}{0.3\textwidth}
        \includegraphics[trim=5pt 5pt 5pt 5pt, clip, width=\textwidth]{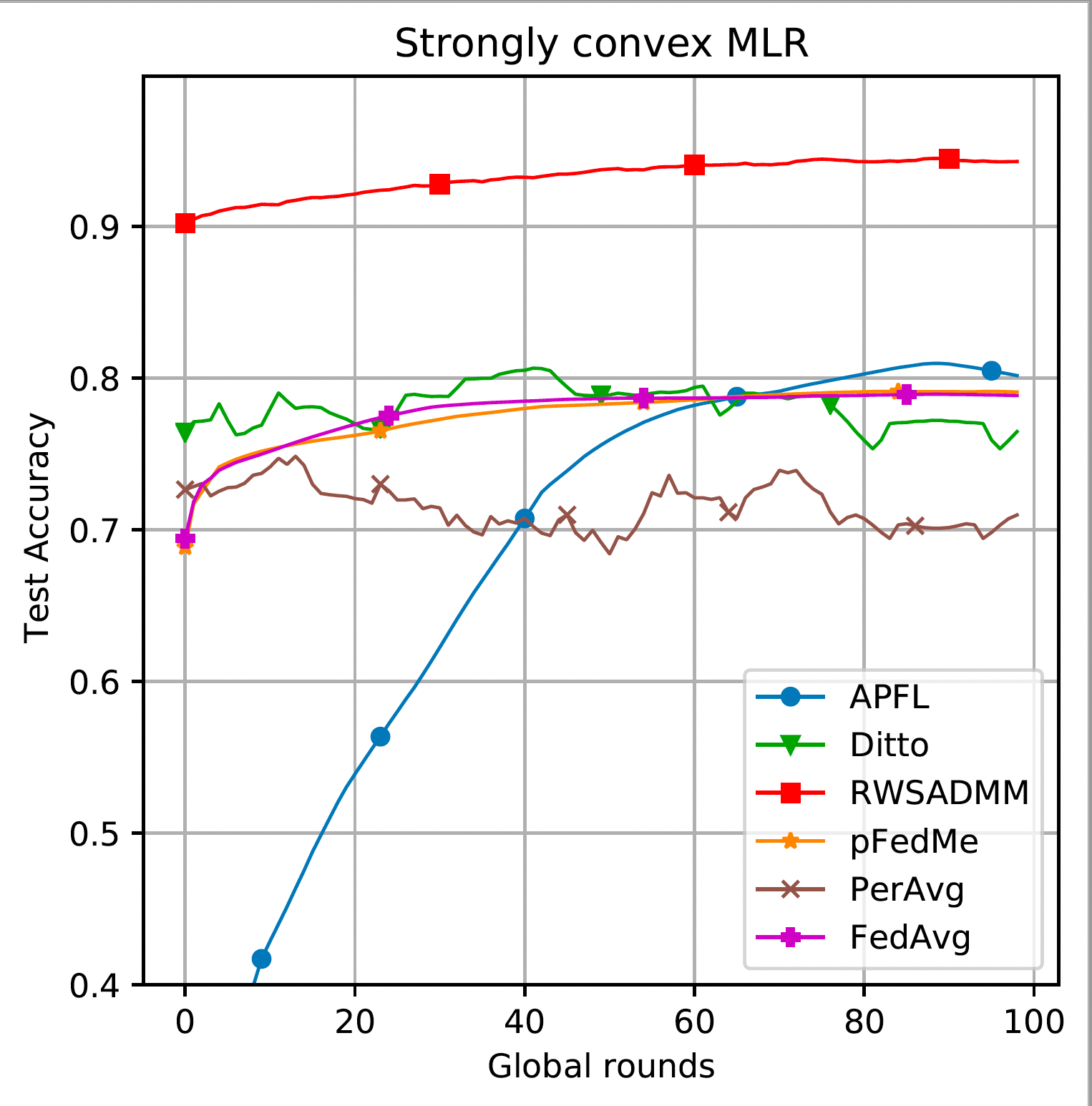}
        \caption{MLR acc}
        \label{fig:mlr-acc-synthetic-all}
    \end{subfigure}
    \begin{subfigure}{0.3\textwidth}
        \includegraphics[trim=5pt 5pt 5pt 5pt, clip, width=\textwidth]{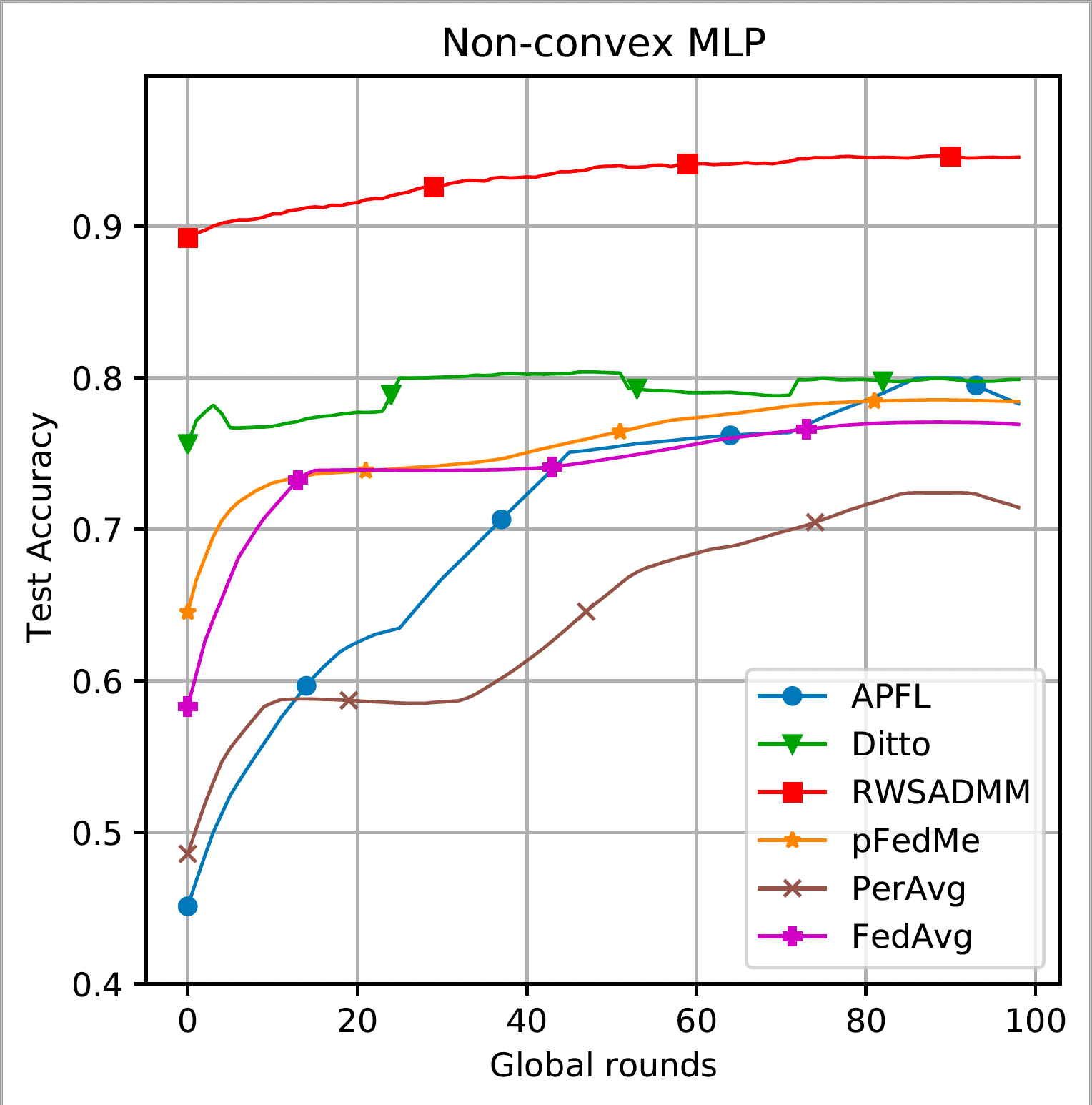}
        \caption{MLP acc}
        \label{fig:mlp-acc-synthetic-all}
    \end{subfigure}
    \vskip\baselineskip
    \begin{subfigure}{0.3\textwidth}
        \includegraphics[trim=5pt 5pt 5pt 5pt, clip, width=\textwidth]{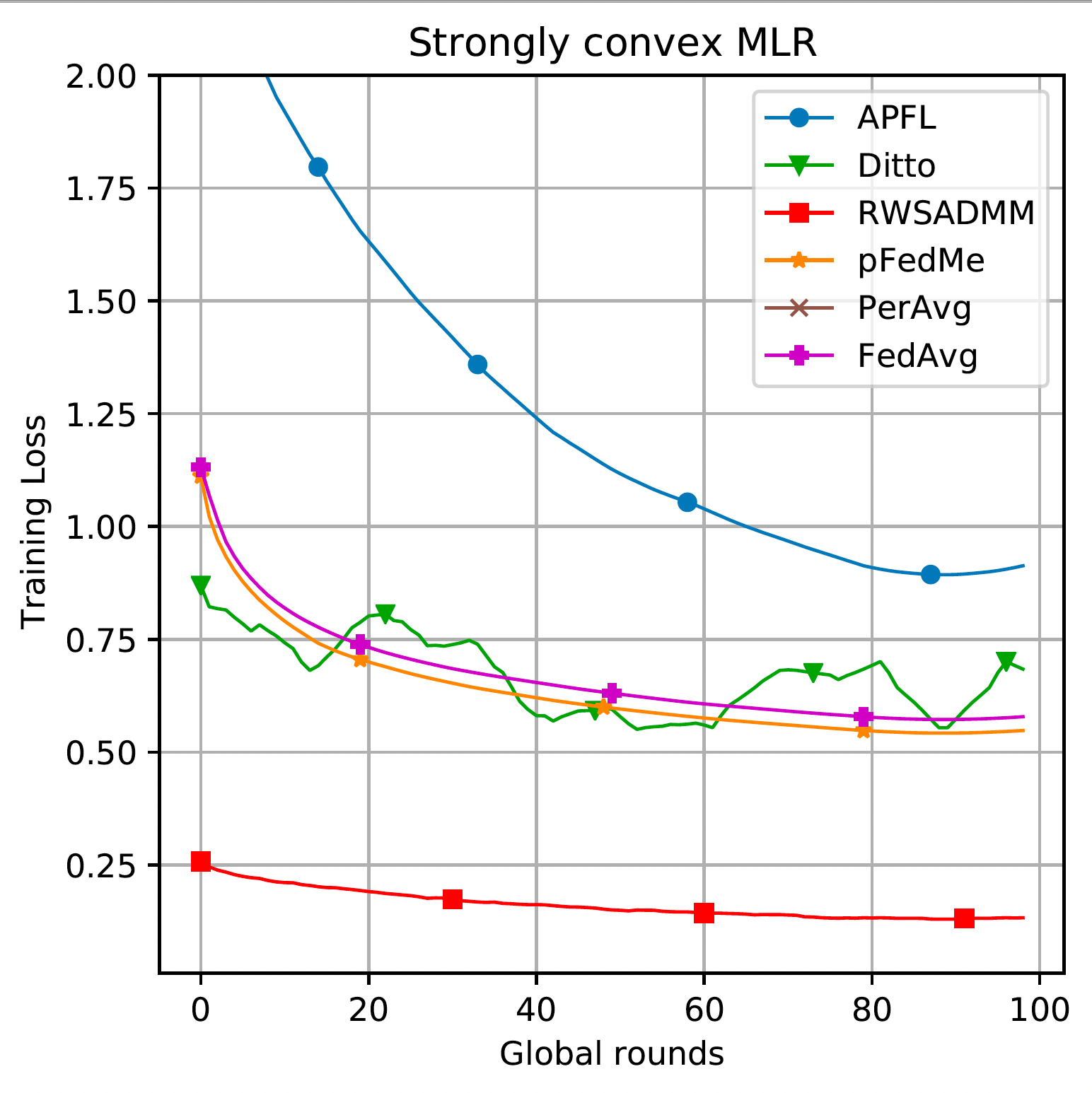}
        \caption{MLR loss}
        \label{fig:mlr-loss-synthetic-all}
    \end{subfigure}
    \begin{subfigure}{0.3\textwidth}
        \includegraphics[trim=5pt 5pt 5pt 5pt, clip, width=\textwidth]{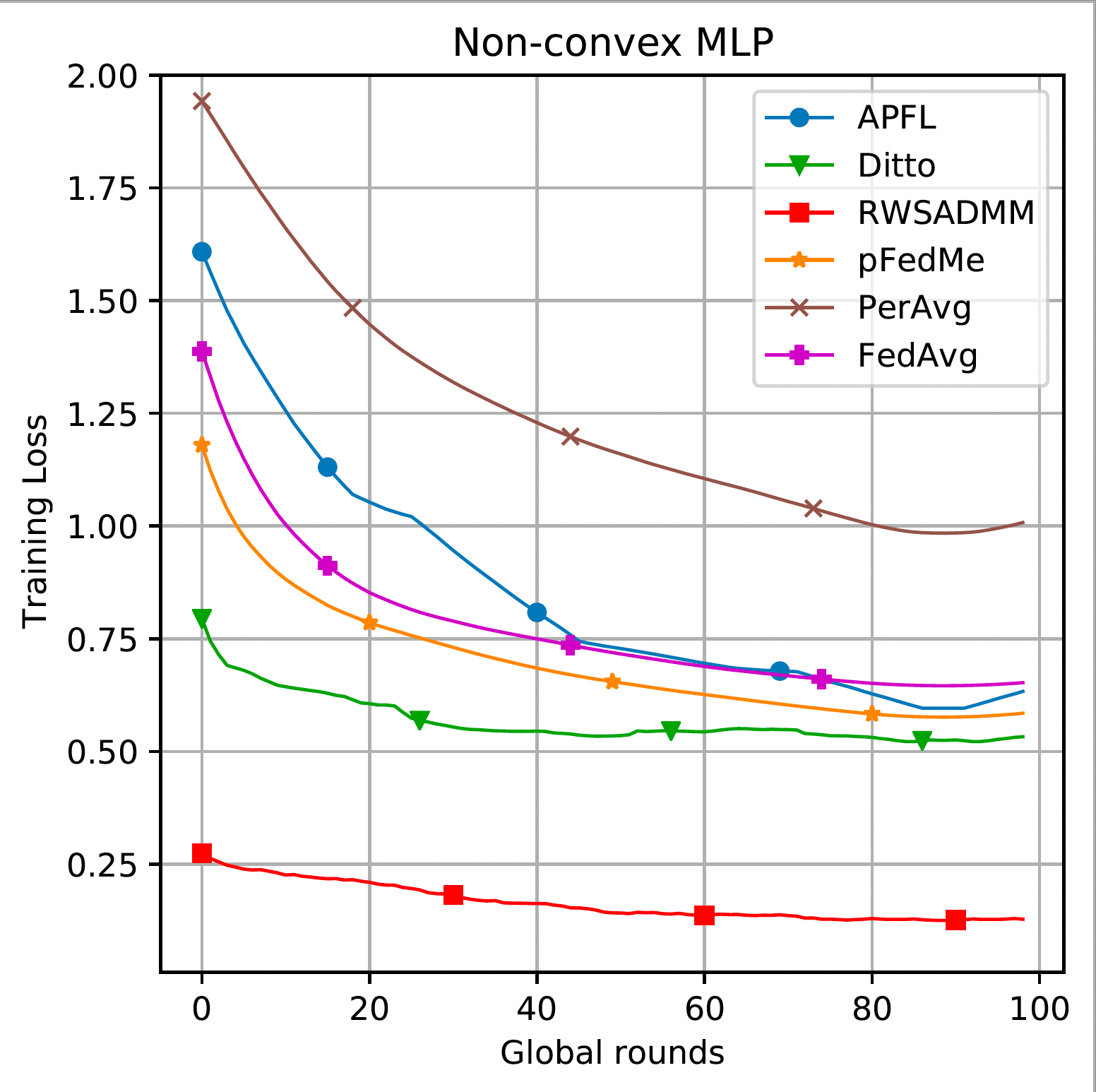}
        \caption{MLP loss}
        \label{fig:mlp-loss-synthetic-all}
    \end{subfigure}
    
    \caption{Performance comparison (test accuracy and training loss) of RWSADMM, pFedMe, Per-Avg, FedAvg, APFL, and Ditto for Synthetic dataset and different settings: strongly convex MLR and non-convex MLP. }
    \label{fig:synthetic-performance}
\end{figure}

\subsection{Different Number of Users} \label{appn:diff_users}

This appendix examines the effect of modifying the number of users/clients in the graph.  The configuration values are all the same as the optional configuration for the RWSADMM algorithm with 20 agents. We keep the size of the neighborhood and the overall graph configurations the same for all the experiments. The batch size is decreased to 5 due to memory limitations, and the number of iterations is increased to 500. The total of users tested is 20, 50, and 100 users. The test accuracy and train loss progress curves for different numbers of clients are shown in Fig. \ref{fig:mnist-performance-diff-users}.

\begin{figure}[ht!]
    \centering
    \begin{subfigure}{0.3\textwidth}
        \includegraphics[trim=5pt 5pt 5pt 5pt, clip, width=\textwidth]{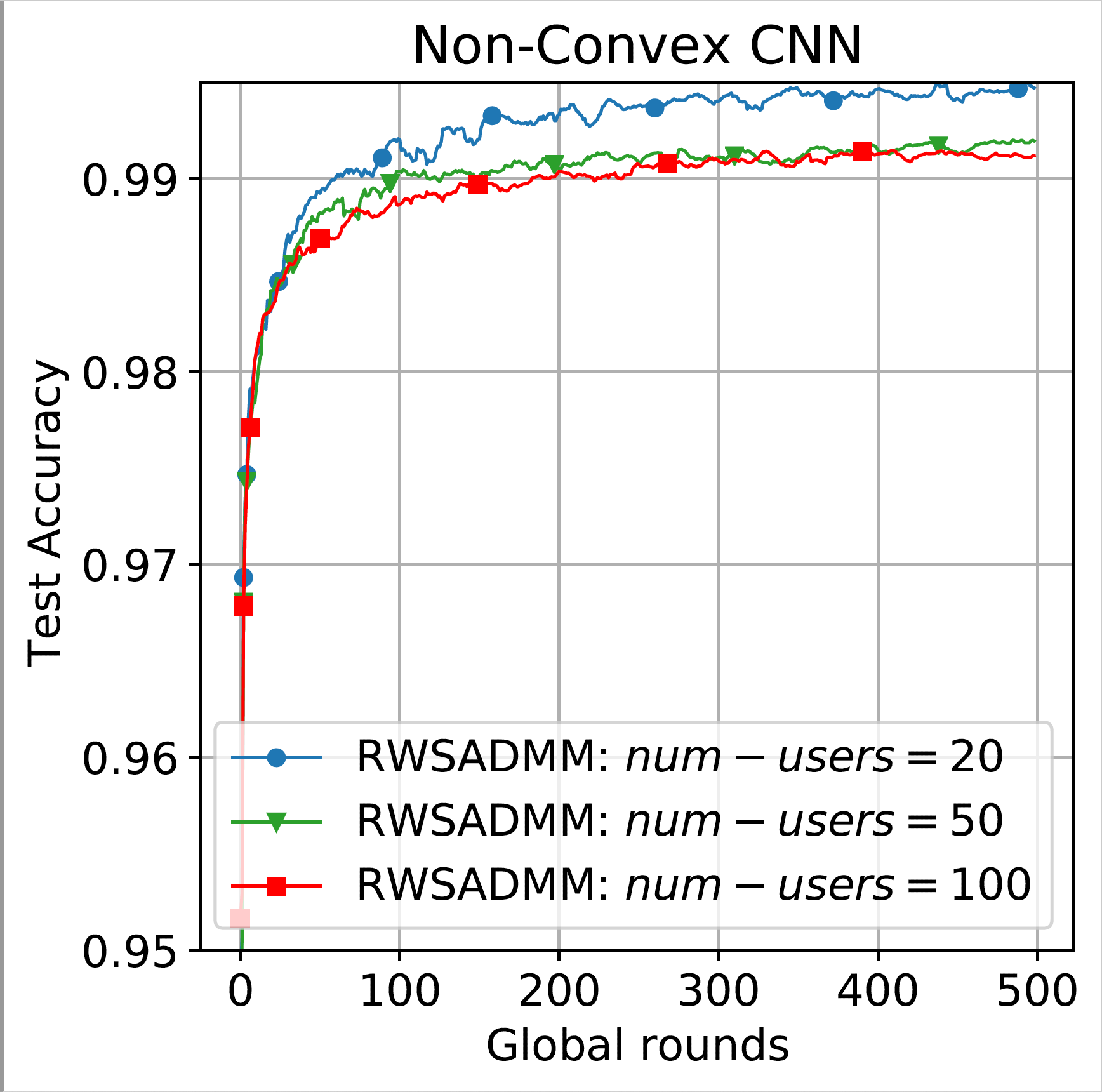}
        \caption{CNN acc}
        \label{fig:cnn-acc-mnist-diff-users}
    \end{subfigure}
    \begin{subfigure}{0.3\textwidth}
        \includegraphics[trim=5pt 5pt 5pt 5pt, clip, width=\textwidth]{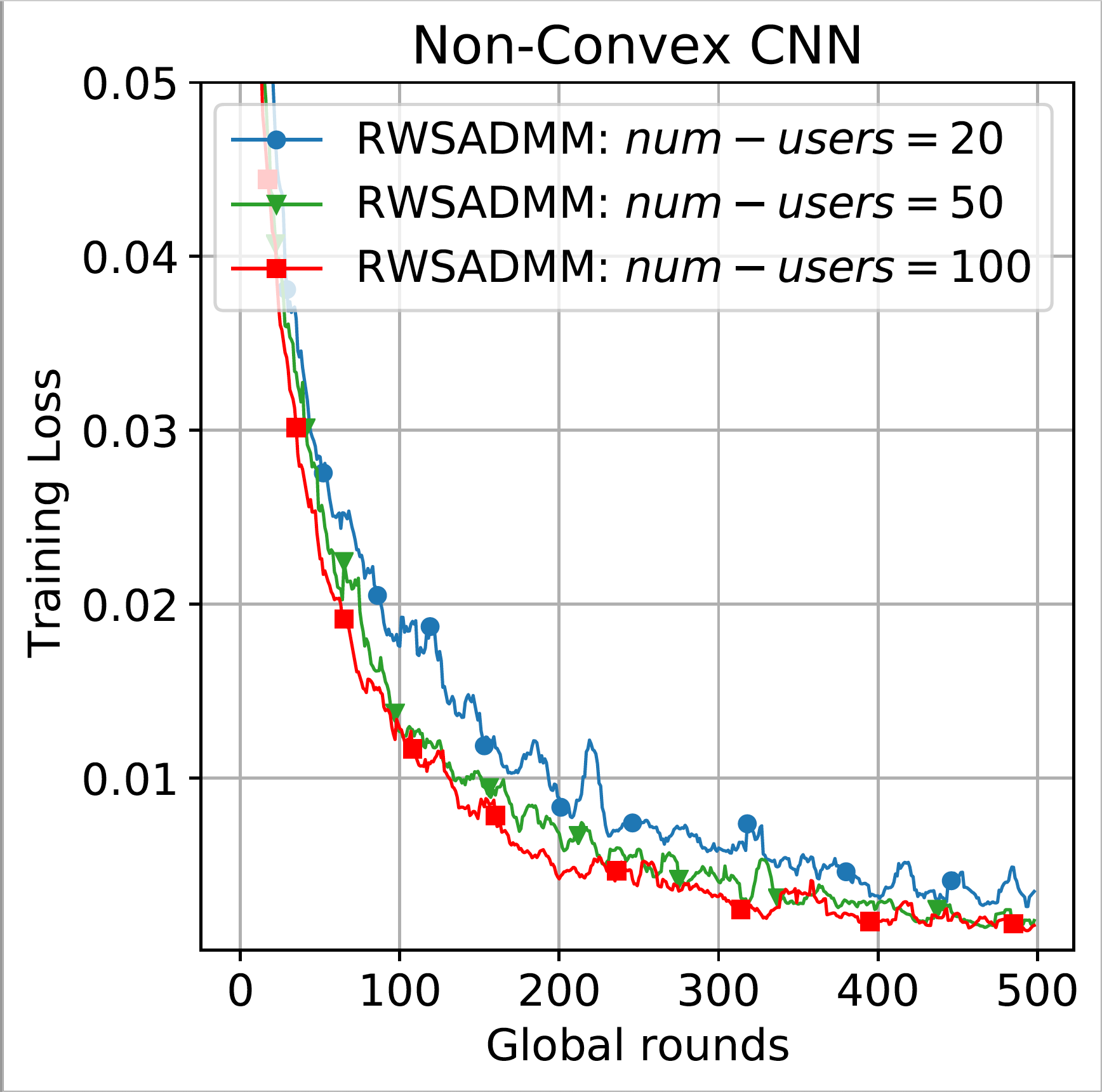}
        \caption{CNN loss}
        \label{fig:cnn-loss-mnist-diff-users}
    \end{subfigure}
    \caption{Performance comparison (test accuracy and training loss) of RWSADMM for different graphs with 20, 50, and 100 users/nodes in the graph. }
    \label{fig:mnist-performance-diff-users}
\end{figure}

As the number of users increases and the graph expands, the convergence gets more challenging, the test accuracy rates slightly decrease, and the time duration of the algorithm increases. The final test accuracy rates and time consumption of different graphs are presented in Table \ref{tab: accuracies-times-different-users}.

\begin{table}[ht!] 
    \centering 
    \scalebox{0.85}{%
    \begin{tabular}{lccccc}
    \toprule
    \multirow{2}{*}{RWSADMM} & \multicolumn{2}{c}{\textbf{MNIST}} \\ \cline{2-3}
    & \multicolumn{2}{c}{CNN}  \\ \cline{2-3}
    \# of users & acc(\%) & t(s)   \\
     \toprule
    20  & $99.57$  & 2929 \\
    50  & $99.25$ & 6994 \\
    100  & $99.19$  & 13878 \\
    \bottomrule
    \end{tabular}
    }
    \vspace{5pt}
    \caption{Test accuracy rate and time duration comparison of RWSADMM for different graphs with 20, 50, and 100 users/nodes in the graph.  }
    \label{tab: accuracies-times-different-users}
\end{table}






\end{document}